\documentclass[twoside]{article}

\usepackage[accepted]{aistats2021}
\usepackage[utf8]{inputenc} % allow utf-8 input
\usepackage[T1]{fontenc}    % use 8-bit T1 fonts
\usepackage{hyperref}    % hyperlinks
\usepackage{url}            % simple URL typesetting
\usepackage{booktabs}       % professional-quality tables
\usepackage{amsfonts}       % blackboard math symbols
\usepackage{nicefrac}       % compact symbols for 1/2, etc.
\usepackage{microtype}      % microtypography

\usepackage{hyperref}

\usepackage[utf8]{inputenc}
\usepackage[T1]{fontenc}
\usepackage{stmaryrd}
\usepackage{wrapfig}
\usepackage{amsfonts}
\usepackage{amsmath}
\usepackage{mathrsfs}
\usepackage{amsthm}
\usepackage{xcolor}
\usepackage{multicol}
\usepackage[ruled, vlined]{algorithm2e}
\usepackage{graphicx}
\usepackage{subcaption}

\usepackage[round]{natbib}
\theoremstyle{definition}
\newtheorem{definition}{Definition}[section]

\theoremstyle{plain}
\newtheorem{assumption}{Assumption}[section]

\theoremstyle{plain}
\newtheorem{theorem}{Theorem}[section]

\theoremstyle{plain}
\newtheorem{proposition}{Proposition}[section]

\theoremstyle{remark}
\newtheorem*{remark}{Remark}

\theoremstyle{plain}
\newtheorem{lemma}{Lemma}[section]

\newcommand{\cali}[1]{\mathcal{#1}}
\newcommand{\rmc}[1]{\mathsf{#1}}
\newcommand{\bb}[1]{\mathbb{#1}}

\DeclareSymbolFont{matha}{OML}{txmi}{m}{it}% txfonts
\DeclareMathSymbol{\varv}{\mathord}{matha}{118}

\newcommand{\linop}[1]{\rmc{#1}}
\newcommand{\matop}[1]{\rmc{#1}}
\newcommand{\matrb}[1]{\pmb{\rmc{#1}}}

\newcommand{\mattovec}{\rmc{vec}}

\newcommand{\trans}{\text{T}}
\newcommand{\adj}{\#}

\newcommand{\boundedops}{\cali L}
\newcommand{\hsops}{\cali L_2}

\newcommand{\spanspace}{\text{Span}}
\newcommand{\nullspace}{\text{Ker}}
\newcommand{\rangespace}{\text{Im}}
\newcommand{\rank}{\text{Rank}}
\newcommand{\spectrum}{\text{Sp}}

\newcommand{\diffop}{D}

\newcommand{\geneloss}{\ell}
\newcommand{\genelossev}[2]{\geneloss(#1, #2)}
\newcommand{\sqrlossev}[2]{\left \| #1 - #2 \right \|_{\Lsqr}^2}
\newcommand{\groundloss}{l}

\newcommand{\regfunc}{f}
\newcommand{\generalhyp}{\cali G}
\newcommand{\generalregu}{\Omega}

\newcommand{\Kx}{\rmc K}
\newcommand{\HKx}{{\cali H}_{\Kx}}

\newcommand{\inspace}{\cali X}

\newcommand{\prodspace}{\cali Z}
\newcommand{\insamp}{\mathbf x}
\newcommand{\outsamp}{\mathbf y}
\newcommand{\prodsamp}{\mathbf z}

\newcommand{\inva}{\rmc X}
\newcommand{\outva}{\rmc Y}
\newcommand{\funcprob}{\rho}

\newcommand{\exprisk}{\cali R}
\newcommand{\emprisk}{\widehat{\exprisk}}

\newcommand{\nsamp}{n}
\newcommand{\sampiter}{i}
\newcommand{\sampiterbis}{j}
\newcommand{\regridge}{h_{\prodsamp}^{\lambda}}
\newcommand{\regridgepartial}{\obsfun{h}_{\obsfun{\prodsamp}}^{\lambda}}

\newcommand{\dict}{\phi}
\newcommand{\projope}{\Phi}
\newcommand{\outfuncdom}{\Theta}
\newcommand{\Lsqr}{\mathsf{L}^2(\outfuncdom)}

\newcommand{\ndict}{d}
\newcommand{\dictiter}{l}
\newcommand{\dictiterbis}{s}

\newcommand{\obsfun}[1]{\widetilde{#1}}

\newcommand{\nobsf}{m}
\newcommand{\obsfiter}{p}
\newcommand{\obslocs}{\theta}
\newcommand{\obslocsvec}{\theta}

\newcommand{\locsprob}{\mu}

\newcommand{\Kxkadri}{\Kx^{\text{fun}}}
\newcommand{\Kxinmatkadri}{\Kx_{\inspace}^{\text{fun}}}
\newcommand{\Kxmatkadri}{\pmb \Kx^{\text{fun}}}
\newcommand{\Kxbarath}{k^{\text{add}}}

\begin{document}

\runningauthor{Dimitri Bouche, Marianne Clausel, Fran\c cois Roueff, Florence d'Alch\'e-Buc}

\twocolumn[

\aistatstitle{Nonlinear Functional Output Regression: A Dictionary Approach}

\aistatsauthor{Dimitri Bouche$^1$ \And Marianne Clausel$^2$  \And Fran\c cois Roueff$^1$  \And Florence d'Alch\'e-Buc$^1$}

\aistatsaddress{$^1$LTCI, T\'el\'ecom Paris, Institut Polytechnique de Paris \And $^2$ Universit\'e de Lorraine, CNRS, IECL} ]

\begin{abstract}
To address functional-output regression, we introduce projection learning (PL), a novel dictionary-based approach that learns to predict a function that is expanded on a dictionary while minimizing an empirical risk based on a functional loss. PL makes it possible to use non orthogonal dictionaries and can then be combined with dictionary learning; it is thus much more flexible than expansion-based approaches relying on vectorial losses. This general method is instantiated with reproducing kernel Hilbert spaces of vector-valued functions as kernel-based projection learning (KPL). For the functional square loss, two closed-form estimators are proposed, one for fully observed output functions and the other for partially observed ones. Both are backed theoretically by an excess risk analysis. Then, in the more general setting of integral losses based on differentiable ground losses, KPL is implemented using first-order optimization for both fully and partially observed output functions. Eventually, several robustness aspects of the proposed algorithms are highlighted on a toy dataset; and a study on two real datasets shows that they are competitive compared to other nonlinear approaches. Notably, using the square loss and a learnt dictionary, KPL enjoys a particularily attractive trade-off between computational cost and performances. 
\end{abstract}

\section{INTRODUCTION} \label{sec :: introduction}
In a large number of fields such as Biomedical Signal Processing, Epidemiology Monitoring, Speech and Acoustics, Climate Science, each data instance consists in a high number of measurements of a common underlying phenomenon. Such high-dimensional data generally enjoys strong smoothness across features. To exploit that structure, it can be interesting to model the underlying functions rather than the vectors of discrete measurements we observe, opening the door to functional data analysis \citep[FDA;][]{RamsaySilverman05}. In practice, FDA relies on the assumption that the sampling rate of the observations is high enough to consider them as functions. Of special interest is the general problem of functional output regression (FOR) in which the output variable is a function and the input variable can be of any type, including a function.

While functional linear models have received a great deal of attention---see the additive linear model and its variations \citep[and references therein]{RamsaySilverman05, Morris15}---, nonlinear ones have been less studied. \citet{BarathAl17} extend the function-to-function additive linear model by considering a tri-variate regression function in a reproducing kernel Hilbert space (RKHS). In non-parametric statistics, \citet{FerratyVieu06} introduce variations of the Nadaraya-Watson kernel estimator for outputs in a Banach space. \citet{PoczosAl15} rather project both input and output functions on orthogonal bases and regress the obtained output coefficients separately on the input ones using approximate kernel ridge regressions (KRR). Finally, extending kernel methods to functional data, \citet{Lian07} introduces a function-valued KRR. In that context \citet{KadriAl10, KadriAl16} propose a solution based on the approximate inversion of an infinite-dimensional linear operator and studies richer kernels. We give more details on those methods and compare them with our approach in Section \ref{subsec :: related main}.

In this paper we introduce a novel dictionary-based approach to FOR. We learn to predict a function that is expanded on a dictionary while minimizing an empirical risk based on a functional loss. We call this approach {\it projection learning} (PL). It can be instantiated with any machine learning algorithm outputting vectors using a wide range of functional losses. PL also makes it possible to use non-orthonormal dictionaries. It represents a crucial advantage as complex functions generally cannot be well represented using few vectors in conventional bases. They can however be compressed very efficiently using learnt redundant dictionaries \citep{Mallat09}. Then, to solve FOR problems with complex output functions, PL combined with dictionary learning (DL) algorithms \citep{DumitrescuIrofti18} can be both fast and accurate. In practice functions are not fully observed; discrete observations are rather available. PL can accommodate such realistic case without making any assumptions on the sampling grids, either by learning with an estimated gradient or by plugging in an estimator in a closed-form functional solution. 

Then, considering vector-valued RKHSs \citep[vv-RKHS,][]{MicchelliPontil05}, we introduce \textit{kernel-based projection learning} (KPL). Vv-RKHSs extend the scope of kernel methods to vector-valued functions by means of operator-valued kernels (OVK)---see Section \ref{sec :: vvRKHSs} of the Supplement for an introduction. They constitute a principled way of performing vector-valued nonlinear regression considering any type of input data for which a kernel can be defined \citep{Shawe-TaylorCristianini04}. Learning typically relies on a representer theorem which remains valid for the KPL problem.

\noindent{\bf Contributions.} We introduce PL, a novel dictionary-based approach to FOR. It can handle non orthonormal dictionaries and can thus be combined with dictionary learning. Then, we focus on KPL, an instantiation based on vv-RKHSs. For the functional square loss, we propose two estimators, one for fully observed output functions and another for partially observed ones. Both are backed with an excess risk bound. For an integral loss based on a differentiable ground loss, we solve KPL using first-order optimization and show that the gradient can easily be estimated from partially observed functions. Eventually, we study different robustness aspects of the proposed algorithms on a toy dataset; and demonstrate on two real datasets that they can be competitive with other nonlinear FOR methods while keeping the computational cost significantly lower.

\noindent{\bf Notations and context.} We assimilate the spaces $(\bb R^{\ndict})^\nsamp$ and $\bb R^{\ndict \times \nsamp}$. The concatenation of vectors $(u_\sampiter)_{\sampiter=1}^\nsamp \in \bb R^{\ndict \times \nsamp}$ is denoted $\mattovec((u_\sampiter)_{\sampiter=1}^\nsamp) \in \bb R^{\ndict \nsamp}$. For $\nsamp \in \bb N^*$, we use the shorthand $[\nsamp]$ for the set $\{ 1,\ldots,\nsamp \}$. We denote by $\cali F (\cali X, \cali Y)$ the space of functions from $\cali X$ to $\cali Y$. For two Hilbert spaces $\cali U$ and $\cali Y$, $\cali L(\cali U, \cali Y)$ is the set of bounded linear operators from $\cali U$ to $\cali Y$ and $\cali L (\cali U):= \cali L(\cali U, \cali U)$. The adjoint of a linear operator $\matop A$ is denoted $\matop A^\adj$. For $\cali U = \bb R^\ndict$, we introduce $\linop A_{(\nsamp)} \in \cali L(\bb R^{\ndict \nsamp}, \cali Y^\nsamp)$ as $\linop A_{(\nsamp)}: \mattovec((u_\sampiter)_{\sampiter=1}^\nsamp) \longmapsto (\linop A u_1,..., \linop A u_\nsamp)$ and $\linop A_{\rmc{mat}, (\nsamp)} \in \cali L(\bb R^{\ndict \times \nsamp}, \cali Y^\nsamp)$ as $\linop A_{\rmc{mat}, (\nsamp)}: (u_\sampiter)_{\sampiter=1}^\nsamp \longmapsto (\linop A u_1,..., \linop A u_\nsamp).$ For $\matop B \in \bb R^{p \times q}, \matop C \in \bb R^{\ndict \times \nsamp}$, $\matop B \otimes \matop C \in \bb R^{p \ndict \times q \nsamp}$ denotes the Kronecker product. Finally $\Lsqr$  stands for the Hilbert space of real-valued square integrable functions on a given compact subset $\outfuncdom \subset \mathbb{R}^q$; without loss of generality we suppose that $|\outfuncdom |:= \int_{\outfuncdom} 1 \mathrm{d} \theta = 1$.

\section{PROJECTION LEARNING} \label{sec :: projection}
\subsection{Functional output regression}
Let $\inspace$ be a measurable space and $(\inva, \outva)$ be a couple of random variables on $\prodspace:= \inspace \times \Lsqr $ with joint probability distribution $\funcprob $. To introduce the FOR problem, we define a functional loss $\geneloss$ as a real-valued function over $\Lsqr \times \Lsqr$. Examples of functional losses include the functional square loss and more generally, any integral of a ground loss $\groundloss: \mathbb{R} \times \mathbb{R} \to \mathbb{R}$. 
Particularly, given such ground loss $\groundloss$, for $(y_0,y_1) \in \Lsqr \times \Lsqr$, a functional loss $\geneloss$ can be defined as:
\begin{equation} \label{eq :: int loss}
\geneloss(y_0, y_1) = \int_{\outfuncdom}\groundloss(y_0(\theta),y_1(\theta)) \mathrm{d}\theta.
\end{equation}
Specifically, taking the square loss as ground loss $\groundloss(y_0(\theta), y_1(\theta))=(y_0(\theta)- y_1(\theta))^2$ we obtain the functional square loss 
$\geneloss_2(y_0,y_1) := \| y_0 - y_1 \|_{\Lsqr}^2$, widely used in the literature \citep{KadriAl10}.

Given such functional loss $\geneloss $ and a hypothesis class $\generalhyp \subset \cali F(\inspace, \Lsqr)$, we now define the FOR problem as
\begin{equation} \label{prob :: true risk functional}
\min_{\regfunc \in \generalhyp} \exprisk (\regfunc):= \bb E_{(\inva, \outva) \sim \funcprob} \left[ \genelossev{\outva}{\regfunc(\inva)} \right].
\end{equation}
However, we have access to the joint probability distribution $\funcprob$ only through an observed sample. The aim is then to approximately solve the above problem using the available data. We study two possible settings. 

In the first one, the output functions are \textit{fully observed}. Our sample $\prodsamp := (x_\sampiter, y_\sampiter)_{\sampiter=1}^\nsamp$ then consists of $\nsamp \in \bb N$ i.i.d. realizations drawn from $\funcprob $, this setting coincides with the so-called {\it dense} one described in FDA \citep{ReimherrKokoska17}. By contrast, in the  \textit{partially observed} setting (also referred to as the {\it sparse} one, described and studied in \citet{ReimherrKokoska17, LiHsing10, CaiYuan11}), the output functions are observed on grids which may be irregular, subject to randomness and potentially different for each function. Even though the former scenario is relatively frequent in theoretical works, the latter can be more realistic. 

In the partially observed setting, we suppose that we only observe each $y_i$ on a random sample of locations, $\obslocsvec_\sampiter := (\obslocs_{\sampiter \obsfiter})_{\obsfiter=1}^{\nobsf_\sampiter} \in \outfuncdom^{\nobsf_\sampiter}$, drawn from a probability distribution $\locsprob$. For the sake of simplicity, $\locsprob$ is chosen as the uniform distribution on $\outfuncdom$ and the draws of locations are supposed to be independent. The learning problem depicted in Equation \eqref{prob :: true risk functional} has now to be solved using a partially observed functional output sample: %defined as a set of observations of the form:
\begin{equation}\label{eq :: partial sample}
 \obsfun{\prodsamp} :=(x_i, (\obslocsvec_\sampiter, \obsfun y_\sampiter) )_{\sampiter=1}^\nsamp,
 \end{equation} 
 where for all $\sampiter \in [\nsamp]$, $\obslocsvec_\sampiter \in \outfuncdom^{\nobsf_\sampiter}$, $\obsfun{y}_\sampiter \in \bb R^{\nobsf_\sampiter}$ with $\nobsf_\sampiter \in \bb N^*$  the number of observations available for the $\sampiter$-th function, and for all $\obsfiter \in [\nobsf_\sampiter]$, $\obslocs_{\sampiter \obsfiter} \in \outfuncdom$ and $\obsfun {y}_{\sampiter \obsfiter} \in \bb R$.
 
In this paper, we propose a novel angle to address the FOR problem using both types of samples.

\subsection{Approximated FOR}
To tackle Problem \eqref{prob :: true risk functional}, we propose to learn to predict expansion coefficients on a dictionary of functions $\dict:=(\dict_\dictiter)_{\dictiter=1}^\ndict \in \Lsqr^\ndict$ with $\ndict \in \bb N^*$ (considerations on the choice of this dictionary are postponed to Section \ref{sec :: dictionaries}). We then introduce the following linear operator:
\begin{definition} \label{def :: projection operator}
\textbf{(Projection operator)}
For a dictionary $\dict$, the associated projection operator $\projope$ is defined by $ \projope :~~u \in \bb R^\ndict \longmapsto \sum_{\dictiter=1}^\ndict u_\dictiter \dict_\dictiter \in \Lsqr$.
\end{definition}
We can give an explicit expression of $\projope ^{\adj} $ as well as a matrix representation of $\projope ^{\adj} \projope $.
\begin{lemma}\label{lemma :: adjoint phi}
The adjoint of $\projope $ is given by 
$\projope^\adj:~g \in \Lsqr \longmapsto (\langle \dict_\dictiter, g \rangle_{\Lsqr})_{\dictiter=1}^\ndict \in \bb R^\ndict $. Thus we have $\projope^ \# \projope = (\langle \dict_\dictiter, \dict_\dictiterbis \rangle_{\Lsqr})_{\dictiter, \dictiterbis=1}^\ndict$.
\end{lemma}
The core idea of PL is to define a simpler model $f(x) = \projope h(x)$ in Problem \eqref{prob :: true risk functional}, where $h : \inspace \longmapsto \bb R^\ndict$ is a vector-valued function. This yields the problem
\begin{equation} \label{prob :: true risk approx}
\min_{h \in \cali H}  \exprisk (\projope \circ h), 
\end{equation}
that we can solve using a sample from one or the other of the two observation settings previously defined.

In the fully observed setting, we can minimize over $\cali H \subset \cali F(\inspace, \bb R^{\ndict})$ the empirical counterpart of the true risk based on $\prodsamp$, $\emprisk(\projope \circ h, \prodsamp) := \frac 1 \nsamp \sum_{\sampiter=1}^\nsamp \genelossev{y_\sampiter}{\projope h(x_\sampiter)}$, with some additional penalty $\generalregu_{\cali H}: \cali H \longrightarrow \bb R$ to control the model complexity:
\begin{equation} \label{prob :: empirical risk approx}
\min_{h \in \cali H} \emprisk(\projope \circ h, \prodsamp)  + \lambda \generalregu_{\cali H}(h), 
\end{equation}
with $\lambda > 0$. In other words, we search a solution in the hypothesis space $\left \{\regfunc: x \longmapsto \projope h(x),~ h \in \cali H \right \}$ and solve a function-valued problem at the price of solving a vector-valued one in $\cali H$. Even though a vector-valued function is learned, the loss remains a functional one. Moreover, any predictive model devoted to vectorial output regression (e. g. neural networks, random forests, kernel methods etc.) is eligible. We regularize our model through the vector-valued function $h$.

To tackle the partially observed setting, rather than formulating an empirical counterpart of the true risk based on $\obsfun{\prodsamp}$, we exploit specific properties of the learning algorithms proposed in Section \ref{sec :: KPL}. Namely in our closed form ridge estimator (Proposition \ref{prop :: closed form ridge}) or in the gradient (Equation \eqref{eq :: gradient separable kernel}), the output functions only appear through scalar products with elements of the dictionary. We can then estimate those from $((\obslocsvec_\sampiter, \obsfun{y}_\sampiter))_{\sampiter=1}^\nsamp$ and use a plug-in strategy. Interestingly, computing the gradient for the data attach term in Problem \eqref{prob :: empirical risk approx} shows that this is a feature of projection learning which is not specific to the vv-RKHS instantiation (see Section \ref{subsec :: gradient projection general} of the Supplement for details).

\section{DICTIONARIES} \label{sec :: dictionaries}
In solving Problem \eqref{prob :: true risk approx} instead of Problem \eqref{prob :: true risk functional}, we restrict the predictions of our model to $\spanspace(\dict)$, the space of linear combinations of functions of $\dict$. As a result $\dict$ must be chosen so that the functions $(y_\sampiter)_{\sampiter=1}^\nsamp$ can be approximated accurately by elements from $\spanspace(\dict)$. To achieve this, several strategies are possible.
 
\subsection{General dictionaries}
\noindent{\bf Orthonormal and Riesz bases}. We can consider families of functions known to provide sharp approximations of functions belonging to $\Lsqr$. Orthogonal bases such as Fourier bases or wavelets bases \citep{devore1992compression}, as well as Riesz bases (see Definition \ref{def :: riesz family}) such as splines \citep{oswald1990degree}, have proved their efficiency in signal compression. In practice, a choice among those families can be made from observed properties of the output functions or prior information on the generating process. Then within a family, dictionaries with different parameters (number of functions and/or other parameters) can be considered. A cross-validation can be performed to select one.

\noindent{\bf Families of random functions}, such as random Fourier features \citep[RFFs,][]{RahimiRecht08} can enjoy good approximation properties as well. Through the choice of such family, we approximate the output functions in a space that is dense in a RKHS \citep{RahimiRecht08Bis}. The link with this RKHS can moreover be made explicit as a family is associated to a given kernel. The kernel can then be chosen by cross-validation and number of functions to include results from a precision/computation time trade-off.

\subsection{Dictionary learning} \label{subsec :: dictionary learning}

When the output functions are too complex, selecting a dictionary can however be difficult. The choice of a family may not be evident and it may take too many atoms (functions) to reach a satisfying approximation precision. While functional principal component analysis \citep[FPCA;][]{RamsaySilverman05} addresses the first issue by ensuring that $\spanspace(\dict)$ is close to $\spanspace((y_\sampiter)_{\sampiter=1}^\nsamp)$, it does not address the second one. If the functions at hand are too complex, a very large number of eigenfunctions will be necessary to reach an acceptable approximation quality. By opposition, dictionary learning (DL) solves both problems; it can generally synthesize faithfully the properties of a complicated set of functions while using very few atoms \citep{MairalAl09}.
%If the choice of a family is not obvious from visualizing the output functions, or if we wish to minimize human intervention, the dictionary can be also be learnt. A functional principal component analysis \citep[FPCA;][]{RamsaySilverman05} can ensure that $\spanspace(\dict)$ is close to $\spanspace((y_\sampiter)_{\sampiter=1}^\nsamp)$. However, for complex functional datasets, a very large number of eigenfunctions must be used to achieve a good representation power. Dictionary learning is then especially attractive as it can generally synthesize more faithfully the properties of the functions at hand while using much fewer atoms \citep{MairalAl09}.
%However because it imposes orthogonality of the basis vectors and requires that the functions are centered beforehand, it can miss key structural patterns in a functional dataset. By opposition, dictionary learning (DL) is free of such constraints. The obtained dictionaries thus generally synthesize more faithfully the properties of the functions at hand using fewer atoms \citep{MairalAl09}.
The DL problem is of the form
\begin{equation} \label{prob :: dictionary learning}
\min_{\dict \in \cali C, \beta \in \bb R^{\ndict \times \nsamp}} \frac 1 \nsamp \sum_{\sampiter=1}^\nsamp \left ( \| y_\sampiter - \projope \beta_i \|_{\Lsqr}^2 + \tau \Omega_{\bb R^\ndict}(\beta_i) \right),
\end{equation}
where $\cali C$ is a set of constraint for the dictionary, $\Omega_{\bb R^\ndict}: \bb R^\ndict \longrightarrow \bb R$ is a penalty on the learned representation coefficients and $\tau > 0$ is a trade-off parameter. $\cali C: = \{\dict \in \Lsqr^\ndict,~ \| \dict_\dictiter \|_{\Lsqr}^2 \leq 1,~~ \dictiter \in [\ndict] \}$ and $\Omega_{\bb R^\ndict} := \|. \|_{1}$ are the most common choices \citep{LeeAl07, MairalAl09}, and most existing algorithms are based on alternating optimization schemes \citep[][and references therein]{DumitrescuIrofti18}.

As opposed to other dictionary based methods \citep{PoczosAl15}, KPL can handle the resulting non orthonormal dictionary and can thus benefit from the compression power of DL. Then combining the two, we obtain a FOR method that can deal directly with complex functional-output datasets at a low computational cost. Admittedly, solving Problem \eqref{prob :: dictionary learning} has a cost, which must however be mitigated. Many efficient algorithms exist \citep{DumitrescuIrofti18} and the dictionary moreover needs to be learnt only once (when selecting other parameters through cross-validation, it needs only be learnt once per fold).
%
%As opposed to other dictionary based methods \citep{PoczosAl15}, KPL can handle the resulting non orthonormal dictionary which brings several benefits.
%First, combined with DL, it can handle datasets with complex output functions. Second, the burden of selecting and tuning the dictionary is alleviated. Finally, as few atoms are needed for a good representation, the computational cost of solving Problem \eqref{prob :: empirical risk approx} is greatly reduced. Admittedly, solving Problem \eqref{prob :: dictionary learning} has a cost, which must be mitigated. Indeed, many efficient algorithms exist \citep{DumitrescuIrofti18} and the dictionary is learnt once. For instance, when selecting other parameters through cross-validation, it need only be learnt once per fold.

\section{VV-RKHS INSTANTIATION} \label{sec :: KPL}
We now focus on projection learning using vv-RKHSs.

\subsection{Vv-RKHSs and representer theorem}

Let $\Kx: \inspace \times \inspace \longmapsto \cali L(\bb R^\ndict)$ be an OVK and $\HKx \subset \cali F(\inspace, \bb R^\ndict)$ its associated vv-RKHS. For $x \in \inspace$, we define $\Kx_x \in \boundedops(\bb R^\ndict, \HKx)$ as $\Kx_x: u \longmapsto \Kx_x u$, with $\Kx_x u: x' \longmapsto \Kx(x', x)u$. We consider Problem \eqref{prob :: empirical risk approx} taking $\cali H=\HKx$ as vector-valued hypothesis class. Setting the regularization as $\generalregu_{\HKx}(h):= \| h \|_{\HKx}^2$ yields the following instantiation of PL with vv-RKHS:
\begin{equation} \label{prob :: empirical risk dict vvrkhs}
\min_{h \in \HKx} \frac 1 \nsamp \sum_{\sampiter=1}^\nsamp \genelossev{y_\sampiter}{\projope h(x_\sampiter)} +  \lambda  \|h \|_{\HKx}^2.
\end{equation}

To solve Problem \eqref{prob :: empirical risk dict vvrkhs}, we show in Proposition \ref{prop :: representer primal} that it benefits from a representer theorem, which proof is given in Section \ref{subsec :: proof of representer primal} of the Supplement. It can then be restated as a problem with $\ndict \nsamp$ variables.

\begin{proposition} \label{prop :: representer primal}
\textbf{(Representer theorem)} For $\geneloss$ continuous and convex with respect to its second argument, Problem \eqref{prob :: empirical risk dict vvrkhs} admits a unique minimizer $h_{\prodsamp}^\lambda$. Moreover there exists $\alpha \in \bb R^{\ndict \times \nsamp} $ such that $$h_{\prodsamp}^\lambda = \sum_{\sampiterbis=1}^\nsamp \Kx_{x_\sampiterbis} \alpha_\sampiterbis. $$
\end{proposition}

\noindent{\bf Choice of kernels.} In vv-RKHSs, the choice of the kernel determines the regularization conveyed by the RKHS norm. In practice, the separable kernel is often used: $\Kx = k \matop B : (x_0, x_1) \longmapsto k(x_0, x_1) \matop B$ \citep{Alvarez12}, with $k$ a scalar kernel on $\inspace$ and $\matop B \in \bb R^{\ndict \times \ndict}$ a positive definite symmetric matrix encoding relations between the output variables. In KPL, $\matop B$ can encode prior information on the dictionary. A diagonal matrix can for instance penalize higher frequencies/scales more. We exploit this with wavelets in the experiments related to biomedical imaging in Section \ref{subsec :: expe dti}.

\subsection{Ridge solution}

In this section, we focus on the functional square loss.

\noindent{\bf Fully observed setting}. By Proposition \ref{prop :: representer primal}, Problem \eqref{prob :: empirical risk dict vvrkhs} can be rewritten as
\begin{align}
\min_{\alpha \in \bb R^{\ndict \times \nsamp}} & \frac 1 \nsamp \left \| \outsamp- \projope_{(\nsamp)} \pmb{\Kx} \mattovec(\alpha) \right \|_{\Lsqr^\nsamp}^2 \nonumber \\
& + \lambda \langle \mattovec(\alpha), \pmb {\Kx} \mattovec(\alpha)\rangle_{\bb R^{\ndict \nsamp}}, \label{prob :: problem in ridge form}
\end{align}
%\begin{equation*}
%\min_{\alpha \in \bb R^{\ndict \times \nsamp}} \frac 1 \nsamp \left \| \outsamp- \projope_{(\nsamp)} \pmb{\Kx} \mattovec(\alpha) \right \|_{\Lsqr^\nsamp}^2 
%+ \lambda \langle \mattovec(\alpha), \pmb {\Kx} \mattovec(\alpha)\rangle_{\bb R^{\ndict \nsamp}}, \label{prob :: problem in ridge form}
%\end{equation*}
where $\outsamp := (y_\sampiter)_{\sampiter=1}^\nsamp \in \Lsqr^\nsamp$, the kernel matrix is defined block-wise as $\pmb {\Kx} :=[\Kx(x_\sampiter, x_\sampiterbis)]_{\sampiter, \sampiterbis=1}^\nsamp \in \bb R^{\ndict \nsamp \times \ndict \nsamp}$; and $\mattovec$ and $\projope_{(\nsamp)}$ are introduced in Section \ref{sec :: introduction}. We then derive a closed-form for fully observed output functions. 
\begin{proposition} \label{prop :: closed form ridge}
\textbf{(Ridge solution)}
The minimum in Problem \eqref{prob :: problem in ridge form} is achieved by any $\alpha^* \in \bb R^{\ndict \times \nsamp} $ verifying
\begin{equation} \label{eq :: closed form ridge}
\left( \pmb {\Kx} (\projope^\adj \projope)_{(\nsamp)} \pmb {\Kx} + \nsamp \lambda \pmb {\Kx}
\right ) \mattovec (\alpha^*):=  \pmb {\Kx} \projope^\adj_{(\nsamp)} \outsamp.
\end{equation}
Such $\alpha^*$ exists. Moreover if $\pmb {\Kx}$ is full rank then $\left( (\projope^\adj \projope)_{(\nsamp)} \pmb {\Kx} + \nsamp \lambda \pmb {\matop I}
\right )$ is invertible and $\alpha^*$ is such that
\begin{equation}\label{eq :: closed form ridge inverse}
\mattovec (\alpha^*) =  \left( (\projope^\adj \projope)_{(\nsamp)} \pmb {\Kx} + \nsamp \lambda \pmb {\matop I}
\right )^{-1} \projope^\adj_{(\nsamp)} \outsamp.
\end{equation}
We define the ridge estimator as
$\regridge := \sum_{\sampiterbis=1}^\nsamp \Kx_{x_\sampiterbis} \alpha^*_\sampiterbis.$
\end{proposition}
The proof is detailed in Section \ref{subsec :: proof closed ridge supplementary} of the Supplement. $(\projope^\adj \projope)_{(\nsamp)}$ is a block diagonal matrix with the Gram matrix $\projope^\adj \projope$ of the dictionary repeated on its diagonal. Then if $\dict$ is orthonormal, Equation \eqref{eq :: closed form ridge inverse} simplifies to $\mattovec (\alpha^*) =  \left(\pmb {\Kx} + \nsamp \lambda \pmb {\matop I} \right )^{-1} \projope^\adj_{(\nsamp)} \outsamp.$
% defined in Lemma \ref{lemma :: adjoint phi} repeated on its diagonal. Then if $\dict$ is an orthonormal system (ONS), Equation \eqref{eq :: closed form ridge inverse} simplifies to: 
%$\mattovec (\alpha^*) =  \left(\pmb {\Kx} + \nsamp \lambda \pmb {\matop I} \right )^{-1} \projope^\adj_{(\nsamp)} \outsamp.$

\noindent{\bf Partially observed setting}
We can derive a solution for partially observed functions from Proposition \ref{prop :: closed form ridge}. To that end, we remark that in Equation \eqref{eq :: closed form ridge inverse}, the output functions only appear through the quantity $(\projope_{(\nsamp)})^\adj \outsamp = \mattovec((\projope^\adj y_\sampiter)_{\sampiter=1}^\nsamp) \in \bb R^{\ndict \nsamp}$ with for $\sampiter \in [\nsamp]$, $\projope^\adj y_\sampiter = \left (\langle y_\sampiter, \dict_\dictiter \rangle_{\Lsqr} \right )_{\dictiter=1}^\ndict$. As a consequence, we propose to estimate those scalar products from the available observations and then to plug the obtained estimators into Equation \eqref{eq :: closed form ridge inverse}.

\begin{definition}(\textbf{Plug-in ridge estimator.}) \label{def :: plug-in ridge estimator}
For all $\dictiter \in [\ndict]$ and $\sampiter \in [\nsamp]$, let $\obsfun{\nu}_{\sampiter \dictiter}: = \frac 1 {\nobsf_\sampiter} \sum_{\obsfiter=1}^{\nobsf_\sampiter} \obsfun{y}_{\sampiter \obsfiter} \dict_\dictiter (\obslocs_{\sampiter \obsfiter})$ be the entries of $\obsfun {\nu} \in \bb R ^{\ndict \times \nsamp}$. Let $\obsfun{\alpha}^* \in \bb R^{\ndict \times \nsamp}$ be such that $\mattovec (\obsfun{\alpha}^*) = \left( (\projope^\adj \projope)_{(\nsamp)} \pmb {\Kx} + \nsamp \lambda \pmb {\matop I}
\right )^{-1} \mattovec (\obsfun{\nu})$. We then define the plug-in ridge estimator as $\regridgepartial := \sum_{\sampiterbis=1}^\nsamp \Kx_{x_\sampiterbis} \obsfun{\alpha}^*_\sampiterbis$.
\end{definition}
We propose the following strategy to compute this estimator for a separable kernel $\Kx = k \matop B$.

\noindent{\bf Fast algorithm for plug-in ridge estimator.} The matrix $\matrb K $ can be rewritten as $\matrb K = \Kx_{\inspace} \otimes \matop B$ with $\Kx_{\inspace}:=(k(x_\sampiter, x_\sampiterbis))_{\sampiter, \sampiterbis=1}^\nsamp \in \bb R^{\nsamp \times \nsamp}$. Solving the linear system in Equation \eqref{eq :: closed form ridge inverse} has time complexity $\cali O(\nsamp^3 \ndict^3)$. However, $(\projope_{(\nsamp)})^\adj \projope_{(\nsamp)} = \matop I \otimes (\projope ^\adj \projope)$, thus $(\projope_{(\nsamp)})^\adj \projope_{(\nsamp)} \matrb K = (\matop I \otimes (\projope ^\adj \projope))(\Kx_{\inspace} \otimes \matop B)$. Using the mixed product property \citep[Lemma~4.2.10]{HornJohnson91}, we must solve $(\Kx_{\inspace} \otimes ((\Phi ^\# \Phi) \rmc B) + \nsamp \lambda \matrb I) \mattovec(\alpha) = \mattovec(\obsfun{\nu})$. Two classic resolution strategies can separate the contribution of $\nsamp$ and $\ndict$ in the cubic term of the complexity. We can notice that the above linear system is equivalent to a discrete time Sylvester equation \citep{Sima96, DinuzzoAl11}, which can be solved in $\cali O(\nsamp^3 + \ndict^3 + \nsamp^2 \ndict + \nsamp \ndict^2)$ time. Or if we wish to test many values of $\lambda$, using the Kronecker structure, we can deduce an eigendecomposition of  $\Kx_{\inspace}\otimes ((\Phi ^\# \Phi) \rmc B)$ from one of $\Kx_{\inspace}$ and one of $(\Phi ^\# \Phi) \rmc B$ \citep[Theorem~ 4.2.12]{HornJohnson91} in $\cali O(\nsamp^3 + \ndict^3)$ time. 
\begin{algorithm}
 \SetAlgoLined
\textbf{Input:} Sample $\obsfun{\prodsamp}$, matrices $\matop B$, $\projope^\adj \projope$ \\
\textbf{Compute:} kernel matrix $\Kx_{\inspace} = (k(x_\sampiter, x_\sampiterbis))_{\sampiter, \sampiterbis=1}^\nsamp$ \\
\textbf{Compute:} estimates $\obsfun {\nu}$ of $(\langle y_i, \phi_d \rangle_{\Lsqr})_{\sampiter=1, \dictiter=1}^{\nsamp, \ndict}$ \\
\textbf{Solve:} $(\Kx_{\inspace} \otimes ((\Phi ^\# \Phi) \rmc B) + \nsamp \lambda \rmc I) \mattovec(\alpha) = \mattovec(\obsfun {\nu})$ \\
\textbf{Output:} Representer coefficients $\alpha \in \bb R^{ \ndict \times \nsamp}$.
\caption{Plug-in ridge estimator}
\label{alg :: solve with separable kernels}
\end{algorithm}
For a given $\alpha \in \bb R^{\ndict \times \nsamp}$, the predicted function at a new input point $x \in \inspace$ is then given by $\projope \matop B \alpha \rmc k_{\insamp}(x)$ with $\rmc k_{\insamp}(x):=(k(x, x_{\sampiter}))_{\sampiter=1}^\nsamp$.

\subsection{Iterative optimization}
For other losses, since it is no longer possible to find a closed-form, we resort to iterative optimization.

\noindent{\bf Fully observed setting} For $\Kx$ separable, using Proposition \ref{prop :: representer primal} and defining $\geneloss_{y_\sampiter}(y):=\geneloss(y_\sampiter, y)$; Problem \eqref{prob :: empirical risk dict vvrkhs} is rewritten as 
\begin{equation*}
\min_{\alpha \in \bb R^{\ndict \times \nsamp}} \frac 1 \nsamp \sum_{\sampiter=1}^\nsamp \geneloss_{y_\sampiter} \left (\projope \matop B \alpha \rmc k_{\insamp}(x_\sampiter) \right ) + \lambda \langle \Kx_{\inspace}, \alpha^{\trans} \matop B \alpha \rangle_{\bb R^{\nsamp \times \nsamp}}.
\end{equation*}

The gradient of the objective is given by 
\begin{equation} \label{eq :: gradient separable kernel}
\frac 1 \nsamp \matop B \projope_{\rmc{mat}, (\nsamp)}^\adj \matop G(\alpha) \Kx_{\inspace} + \lambda \matop B \alpha \Kx_{\inspace}, 
\end{equation}
with $\matop G(\alpha) := \left (\nabla \geneloss_{y_\sampiter}( \projope \matop B \alpha \rmc k_{\insamp}(x_\sampiter))\right )_{\sampiter=1}^\nsamp \in \Lsqr^\nsamp$ and $\nabla \geneloss_{y_\sampiter} : \Lsqr \longmapsto \Lsqr$ the gradient of $\geneloss_{y_\sampiter}$. 

\noindent{\bf Partially observed setting}. We notice that the entries of $\projope_{\rmc{mat}, (\nsamp)}^\adj \matop G(\alpha) \in \bb R^{\ndict \times \nsamp}$ are the scalar products $\left (\langle \nabla \geneloss_{y_\sampiter}( \projope \matop B \alpha \rmc k_{\insamp}(x_\sampiter)), \dict_\dictiter \rangle_{\Lsqr} \right )_{\dictiter, \sampiter=1}^{\ndict, \nsamp}$. For $\geneloss$ an integral loss (Equation \eqref{eq :: int loss}) based on a differentiable ground loss $\groundloss$, $\nabla \geneloss_{y_\sampiter}: y \longmapsto (\theta \longmapsto \groundloss(y_\sampiter(\theta), y(\theta)))$.
We can thus estimate the columns $\projope^\adj \nabla \geneloss_{y_\sampiter}(\projope \matop B \alpha \rmc k_{\insamp}(x_\sampiter))$ as 
%$((\obslocsvec_{\sampiter}, \obsfun{y}_\sampiter))_{\sampiter=1}^\nsamp$:
\begin{equation} \label{eq :: estimated gradient KPL}
\frac 1 {\nobsf_\sampiter} \sum_{\obsfiter=1}^{\nobsf_\sampiter} \groundloss \left (y_\sampiter(\obslocs_{\sampiter \obsfiter}), \dict(\obslocs_{\sampiter \obsfiter})^\trans \matop B \rmc k_{\insamp}(x_i) \right ) \dict (\obslocs_{\sampiter \obsfiter}), 
\end{equation}
where we have used the convention that for $\theta \in \outfuncdom$, $\dict(\theta) := (\dict_\dictiter(\theta))_{\dictiter=1}^\ndict \in \bb R^\ndict$. The corresponding estimation of $\projope_{\rmc{mat}, (\nsamp)}^\adj \matop G(\alpha)$ can be plugged into Equation \eqref{eq :: gradient separable kernel} to yield an estimated gradient.

\noindent{\bf Link with ridge estimator}. In the partially observed setting, for the square loss, iterative optimization and the plug-in ridge estimator do not yield the same result. In fact, they correspond to two different ridge closed-forms (see Section \ref{subsec :: comparison iterative plug-in} of the Supplement). While the former is slower to compute than the latter it can be more robust (see Section \ref{subsec :: expe toy}).
%While iterative optimization is slower to compute than the ridge plug-in estimator, it is more robust as demonstrated in Section \ref{sec :: experiments}. In the fully observed setting, the two are equivalent, nevertheless computing the latter with Algorithm \ref{alg :: solve with separable kernels} is much faster.

\section{THEORETICAL ANALYSIS} \label{sec :: theory}
In this section we give two finite sample excess risk bounds. One for the ridge estimator in the fully observed setting and one for the plug-in ridge estimator in the partially observed setting. In the first case, we study the effect of the number of samples $\nsamp$, and in the second case that of both $\nsamp$ and the number of observations per function $\nobsf$. We suppose that for all $\sampiter \in [\nsamp]$, $\nobsf_\sampiter = \nobsf$. We leave however a detailed analysis with respect to the size of the dictionary $\ndict$ (including approximation aspects) for future work. Our analysis is based on the framework of integral operators \citep{CaponettoDevito07, SmaleZhou07} to which we give an introduction in the context of our problem in Section \ref{sec ::  integral operators} of the Supplement.

\subsection{Fully observed setting}
In this section, we suppose that $\inspace$ is a separable metric space. We also need to relate the $\Lsqr$ norm of any $g \in \spanspace(\dict)$ to the square norm of its coefficients in the dictionary $\dict$. To that end, a usual assumption is that it is a {\it Riesz family}~\citep{Casazza00}.
\begin{definition} \label{def :: riesz family}
\textbf{(Riesz family)}
$\dict \in \Lsqr^\ndict$ is a Riesz family of $\Lsqr$ with constants $(c_\dict, C_\dict)$ if it is linearly independent and for any $u \in \mathbb{R}^\ndict$, 
$$c_\dict \left \| u \right \|_{\mathbb{R}^\ndict} \leq \left \|\sum_{\dictiter=1}^\ndict u_\dictiter \dict_\dictiter \right \|_{\Lsqr} \leq C_\dict \left \| u \right \|_{\mathbb{R}^\ndict}.$$ If in addition for all $\dictiter \in [\ndict]$, $\left \| \dict_\dictiter \right \|_{\Lsqr}=1$, it is a normed Riesz family.
\end{definition}
\begin{remark}
Riesz families provide a natural generalization of orthonormal families as a normed Riesz family with $c_\dict = C_\dict = 1$ is orthonormal.
\end{remark}
We make the following assumptions. 
\begin{assumption} \label{ass :: conditions on kernel}
$\Kx$ is a vector-valued continuous kernel and there exists $\kappa > 0$ such that for $x \in \inspace$, $\| \Kx(x, x) \|_{\cali L(\bb R^\ndict)} \leq \kappa$.
\end{assumption}
\begin{remark}
We suppose that $\kappa$ is independant from $\ndict$. This is for instance the case if for $x \in \inspace$, $\Kx(x, x)$ is diagonal or block diagonal with bounded coefficients. More generally, we can rely on the fact that $\kappa$ is bounded by the maximal $\|\cdot\|_1$-norm of the columns of $\Kx(x, x)$, which can easily be imposed to be be independent of $\ndict$.
\end{remark}
\begin{assumption} \label{ass :: dictionary orthonormality}
The dictionary $\phi$ is a normed Riesz family in $\Lsqr$ with upper constant $C_{\phi}.$
\end{assumption}
\begin{remark}
We do not use the lower constant $c_\dict$.
\end{remark}
\begin{assumption} \label{ass :: existence of risk minimizer}
There exist $h_{\HKx} \in \HKx$ such that $h_{\HKx} = \inf_{h \in \HKx} \cali R(\projope \circ h).$
\end{assumption}
\begin{remark}
This is a standard assumption \citep{CaponettoDevito07, BaldassarreRosascoBarla12, LiAl19}, it implies the existence of a ball of radius $R>0$ in $\HKx$ containing $h_{\HKx}$, as a consequence
$\|h_{\HKx} \|_{\HKx} \leq R$.
\end{remark}
%\begin{assumption} \label{ass :: as boundedness of outputs}
%There exist $L > 0$ such that almost surely, $\| \rmc Y \|_{\Lsqr}\leq L$. 
%\end{assumption}
\begin{assumption} \label{ass :: value bounded Y}
There exists $L \geq 0$ such that for all $\theta \in \outfuncdom$, almost surely $|\outva(\theta)| \leq L$.
\end{assumption}
We then have the following excess risk bound for the ridge estimator defined in Proposition \ref{prop :: closed form ridge}. We prove it in Section \ref{subsec :: proof excess risk bound} of the Supplement. 
\begin{proposition} \label{prop :: excess risk bound}
Let $0 < \eta < 1$, taking $\lambda = \lambda_\nsamp^*( \nicefrac \eta 2):= 6 \kappa C_{\phi}^2 \frac {\log \left (\nicefrac 4 \eta \right )\sqrt \ndict}{\sqrt \nsamp}$, with probability at least $1 - \eta$
\begin{equation*}
\cali R(\projope \circ \regridge) - \cali R(\projope \circ h_{\HKx}) \leq 27 \left ( \frac {B_0}{\sqrt \ndict} + B_1 \sqrt \ndict \right ) \frac {\log \left (\nicefrac 4 \eta \right )} {\sqrt \nsamp},
\end{equation*}
with $B_0: = (L + \sqrt \kappa C_{\phi} R)^2 $ and $B_1:= \kappa C_\dict^2 R^2$.
\end{proposition}

This bound implies the consistency of the ridge estimator in the number of samples $\nsamp$.

\subsection{Partially observed setting}
To treat the partially observed setting, we need to make the following additional assumption.

\begin{assumption} \label{ass :: value bounded dict}
There exists $M(\ndict) \geq 0$ such that for all $\theta \in \outfuncdom$ and for all $\dictiter \in [\ndict]$, $|\dict_\dictiter(\theta) | \leq M(\ndict)$.
\end{assumption}
\begin{remark}
The dependence in $\ndict$ is specific to the family to which $\dict$ belongs; for wavelets we have $M(\ndict) = 2^{\nicefrac {r(\outfuncdom, \ndict)} 2} \max_{\theta \in \outfuncdom} |\psi(\theta)|$ with $\psi$ the mother wavelet and $r(\outfuncdom, \ndict) \in \bb N$ the number of dilatations included in $\dict$, whereas for a Fourier dictionary we have $M(\ndict)=1$. 
\end{remark}

We then have the following excess risk bound for the plug-in ridge estimator from Definition \ref{def :: plug-in ridge estimator} which we prove in Section \ref{subsec :: proof excess risk bound partial} of the Supplement. 

\begin{proposition} \label{prop :: excess risk bound partial}
Let $0 < \eta < 1$, taking $\lambda = \lambda_\nsamp^*(\nicefrac \eta 3) := 6 \kappa C_{\phi}^2 \frac {\log \left (\nicefrac 6 \eta \right )\sqrt \ndict}{\sqrt \nsamp}$, with probability at least $1 - \eta$, 
\begin{align*}
& \cali R(\projope \circ \regridgepartial) - \cali R(\projope \circ h_{\HKx}) \\ 
& \leq \left (\frac{B_2(\ndict)\sqrt \nsamp}{\nobsf^2} + \frac {B_3(\ndict)} {\nobsf^{\nicefrac 32}} + \frac {9C(\ndict)^2} {2 \sqrt \nsamp \nobsf} + \frac {B_4(\ndict)} {\sqrt \nsamp} \right ) \log \left ( \nicefrac 6 \eta \right ),
\end{align*}
with  $ C(\ndict):=\frac {LM(\ndict)}{C_{\dict}}$, $B_2(\ndict):=18 \sqrt \ndict \left (C(\ndict) + \frac R {\sqrt \ndict} \right )^2$, \\ $B_3(\ndict): = B_2(d) - 18 \frac{R^2}{\sqrt \ndict}$, $B_4(\ndict):= \frac {81} 2 \left (\frac {B_0} {\sqrt \ndict} + B_1 \sqrt \ndict \right )$ and $B_0$ and $B_1$ are defined as in Proposition \ref{prop :: excess risk bound}.
\end{proposition}

We highlight that if $\nobsf \asymp \sqrt {\nsamp}$, then this bounds yields consistency for the plug-in ridge estimator.

\section{NUMERICAL EXPERIMENTS} \label{sec :: experiments}
Section \ref{subsec :: expe toy} is dedicated to the study of several aspects of robustness of KPL algorithms. Then we compare KPL with the nonlinear FOR methods presented in Section \ref{subsec :: related main} on two datasets. In Section \ref{subsec :: expe dti} we explore a biomedical imaging dataset with relatively small number of samples ($\nsamp=100$) and partially observed functions, whereas in Section \ref{subsec :: speech} we study a speech inversion dataset with relatively large number of samples ($\nsamp=413$) and fully observed output functions. 

We use the mean squared error (MSE) as metric. Given observed functions $((\obslocsvec_\sampiter, \obsfun{y}_{\sampiter}))_{\sampiter=1}^\nsamp$ and predicted ones $(\widehat y_\sampiter)_{\sampiter=1}^{\nsamp} \in \Lsqr$, we define it as $\text{MSE} :=  \frac 1 {\nsamp}  \sum_{\sampiter=1}^\nsamp \frac 1 {\nobsf_\sampiter} \sum_{\obsfiter=1}^{\nobsf_\sampiter} ( \widehat y_\sampiter (\obslocs_{\sampiter \obsfiter} ) - \obsfun{y}_{\sampiter \obsfiter}  ) ^2$. The presented results are averaged either over 10 or 20 runs with different train/test splits. Full details of the experimental procedures are postponed to Section \ref{sec :: experiments :: supp} of the Supplement.

\subsection{Related works} \label{subsec :: related main}
We compare KPL to four existing nonlinear FOR methods that we present in this section. More detailed descriptions are given in Section \ref{sec :: related} of the Supplement.

\noindent{\bf Functional kernel ridge regression (FKRR). }\citet{KadriAl10, KadriAl16} solve a functional KRR using function-valued-RKHSs. A representer theorem yields a closed-form solution computed by inverting an operator in $\cali L(\Lsqr)^{\nsamp \times \nsamp}$. For a separable kernel $\Kxkadri=k \matop L$ with $\matop L \in \boundedops(\Lsqr)$, if an eigendecomposition of $\matop L$ is known in closed-form, an approximate solution is computed in $\cali O(\nsamp^3 + \nsamp^2 J \nobsf)$ time, with $J$ the number of eigenfunctions considered and $\nobsf$ the size of the discretization grid. If not, a discretized problem is solved in $\cali O(\nsamp^3 + \nobsf^3 + \nsamp^2 \nobsf + \nsamp \nobsf^2)$ time. 

\noindent{\bf Triple basis estimator (3BE).} In \citep{PoczosAl15}, the input and the output functions are represented by decomposition coefficients on two orthonormal families. The output coefficients are then regressed on the input ones using KRRs approximated with $J$ RFFs in $\cali O(J^3 + J^2 \ndict)$ time, with $\ndict$ the size of the output family. As 3BE is specific to function-to-function regression with scalar-valued inputs, we deal with vector-valued input functions (as in Section \ref{subsec :: speech}), directly through a kernel. We call this extension \textbf{one basis estimator (1BE)}; it is solved in $\cali O(\nsamp^3 + \nsamp^2 \ndict)$ time. 1BE is in fact a particular case of the KPL plug-in ridge estimator with $\dict$ orthonormal and $\Kx = k \matop I$. However, our estimator offers the additional possibility to use non orthonormal dictionaries and to impose richer regularizations through kernels $\Kx = k \matop B$ with $\matop B \neq \matop I$. KPL can moreover can be used with a wide range of functional losses.

\noindent{\bf Kernel additive model (KAM).} \citet{BarathAl17} propose an additive function-to-function regression model using RKHSs. A representer theorem leads to a closed-from solution. Computations are performed in a truncated FPCA basis of size $J < \nsamp $. For a product of kernels, if the Kronecker structure is exploited (a possibility which is however not highlighted by the authors), the complexity is $\cali O(\nsamp^3 + J^3 + \nsamp^2 J + \nsamp J^2)$ time using a Sylvester solver. However, computing the matrix to form the linear system---matrix $A$ in page 6 of \citep{BarathAl17}---is generally much more expensive; exploiting the product of kernels, $\nsamp^2 + J^2$ double integrals must be computed which has time complexity $\cali O(\nsamp^2 t^2 + J^2 \nobsf^2)$, with $t$ the size of the input discretization grid. Those computations must moreover generally be repeated many times so as to tune the multiple kernel parameters. 

\noindent{\bf Kernel Estimator (KE).} Finally, an extension of the Nadaraya-Watson kernel estimator to Banach spaces is introduced and studied in \citep{Ferraty2011}.

\subsection{Preliminary elements} \label{subsec :: preliminary elements}
\noindent{\bf Note on optimization}. We compute the KPL plug-in ridge estimator as in Algorithm \ref{alg :: solve with separable kernels} with Sylvester solver. For iterative optimization, we use L-BFGS-B \citep{ZhuAl97LBFGS}; the estimates of partial second order informations improve convergence speed. For FKRR, of the two possible approaches from Section \ref{sec :: related} of the Supplement, we use the faster Sylvester approach. For KAM we exploit the separability as well using a Sylvester solver. 

\noindent{\bf Logcosh functional loss}.
As an example of a robust integral loss, for $\gamma > 0$, we introduce $\geneloss_{\rmc{lch}}^{(\gamma)}$. It is obtained by taking $\groundloss_{\rmc{lch}}^{(\gamma)}: (a, b) \longmapsto \nicefrac 1 \gamma \log(\text{cosh}(\gamma(a - b))$ as ground loss in Equation \eqref{eq :: int loss}. This ground loss behaves similarly to the Huber loss \citep{Huber64}---almost quadratically around $0$ and almost linearly elsewhere. The parameter $\gamma$ gives us control on its behaviour around $0$, as it grows bigger,  $\groundloss_{\rmc{lch}}^{(\gamma)}$ tends to the absolute loss (see Section \ref{sec :: experiments :: supp} for examples). As opposed to our proposed integral loss $\geneloss_{\rmc{lch}}^{(\gamma)}$, the extension of the Huber loss to $\Lsqr \times \Lsqr$ \citep[e. g. ][Example 13.7]{BauschkeCombettes17} is not differentiable everywhere. 

\subsection{Toy data} \label{subsec :: expe toy}
\begin{figure}
\begin{center}
\includegraphics[width=\linewidth]{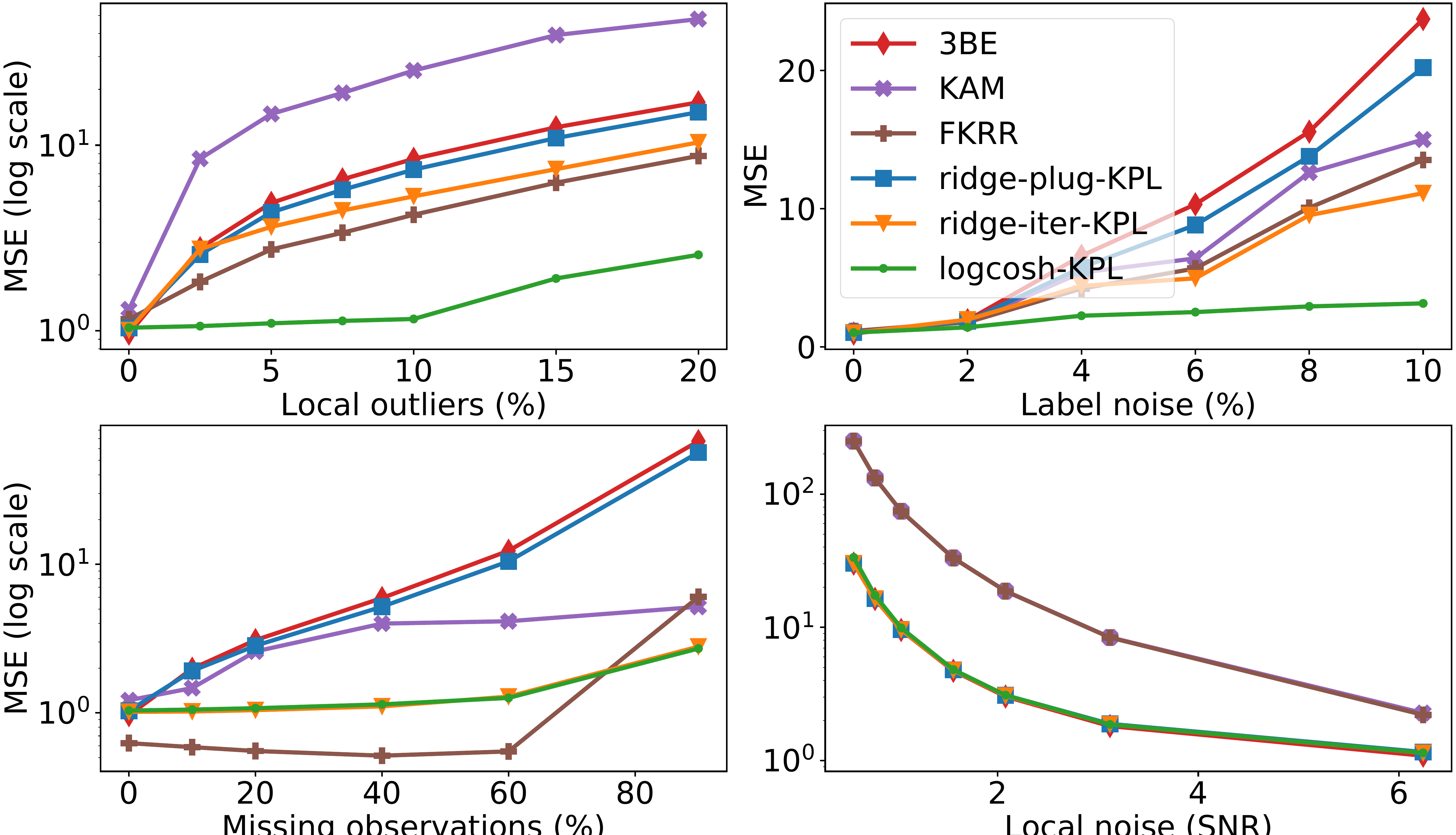}
\end{center}
\caption{Several aspects of robustness.}
\label{fig :: toy all}
\end{figure}

In this section, we take $\Kx=k \matop I$ with $k$ a scalar-valued Gaussian kernel. We use a generated toy dataset: inputs are random mixtures of cubic B-splines \citep{DeBoor01} centered at different
locations and outputs are associated mixtures of Gaussian processes (drawn once and then fixed). The full generation procedure is described in Section \ref{sec :: experiments :: supp} of the Supplement. We use $\nsamp_{\text{train}}=100$ samples for training and $\nsamp_{\text{test}}=100$ samples and use Fourier dictionaries for KPL and 3BE. 

\noindent{\bf Corruption modalities}. We study the effect of four types of corruptions of the training data: local outliers, label noise, missing observations and local noise. In the first case, observations from the output functions are replaced with random draws in their range. In the second case, some output functions are replaced with erroneous ones. In the third case we remove observations from the output functions uniformly at random. Finally, in the last one we add Gaussian noise to those observations. We then use the signal to noise ratio as x-axis; for a noise level $\sigma$ and a sample $\obsfun{\prodsamp}$, we define it as $\text{SNR}: = \frac 1 {\sigma \nsamp}  \sum_{\sampiter=1}^\nsamp \frac 1 {\nobsf_\sampiter} \sum_{\obsfiter=1}^{\nobsf_\sampiter} \left | \obsfun{y}_{\sampiter \obsfiter} \right |.$ 

\noindent{\bf Comments on the results}. The evolution of the MSEs for several levels of corruption are displayed in Figure \ref{fig :: toy all}. For each type, at least one KPL algorithm is particularly robust which demonstrates the versatility of our framework.
KPL can be combined with the functional logcosh loss to obtain a FOR algorithm that is robust to outliers (\textit{logcosh-KPL}). Dealing with partially observed functions, KPL solved iteratively using estimated gradients works especially well (\textit{ridge-iter-KPL}, \textit{logcosh-KPL}). Finally all proposed KPL algorithms are robust to local noise.
\subsection{Diffusion tensor imaging dataset (DTI)} \label{subsec :: expe dti}
\noindent{\bf Dataset.} We now consider the DTI dataset.\footnote[1]{This dataset was collected at Johns Hopkins University and the Kennedy-Krieger Institute and is freely available as a part of the \textit{Refund} R package} It consists of 382 Fractional anisotropy (FA) profiles inferred from DTI scans along two tracts---corpus callosum (CCA) and right corticospinal (RCS). The scans were performed on 142 subjects; 100 multiple sclerosis (MS) patients and 42 healthy controls. MS is an auto-immune disease which causes the immune system to gradually destroy myelin, however the structure of this process is not well understood. Using the proxy of FA profiles, we propose to predict one tract (RCS) from the other (CCA). We consider only the first $\nsamp=100$ scans of MS patients. Finally, we highlight that the functions are partially observed: significant parts of the FA profiles along the RCS tract are missing. 
\begin{table}[t]
\caption{MSEs on the DTI dataset.}
\label{tab :: results DTI}
\begin{sc}
\begin{small}
\begin{center}
\begin{tabular}{lccr}
\toprule
KE & 0.231 $\pm$ 0.025  \\
3BE & 0.227 $\pm$ 0.017 \\
KAM & 0.222 $\pm$ 0.021 \\
FKRR & 0.215 $\pm$ 0.020 \\
\midrule
Ridge-KPL & 0.211 $\pm$ 0.022 \\
Logcosh-KPL & \textbf{0.209 $\pm$ 0.020} \\
\bottomrule
\end{tabular}
\end{center}
\end{small}
\end{sc}
\end{table}

\noindent{\bf Experimental setting.} We perform linear smoothing if necessary---for FKRR and KAM. We split the data as $\nsamp_{\text{train}}=70$ and $\nsamp_{\text{test}}=30$ and use wavelets dictionaries for 3BE and KPL. For KPL, we take a kernel of the form $\Kx = k \matop D$ with $k$ a Gaussian kernel and $\matop D$ a diagonal matrix with diagonal decreasing with the corresponding wavelet scale. Finally, when using wavelets, we extend the signal symmetrically to avoid boundary effects. The MSEs are shown in Table \ref{tab :: results DTI}. 

\noindent{\bf Comments on the results.} The studied methods perform almost equally well, with a slight advantage for ours. The combination of an efficient use of wavelets (well suited to non-smooth data) with the scale-dependant regularization induced by the kernel $\Kx = k \matop D$ may explain this. 

\subsection{Synthetic speech inversion dataset} \label{subsec :: speech}
\begin{figure}[t]
\begin{center}
\includegraphics[width=\linewidth]{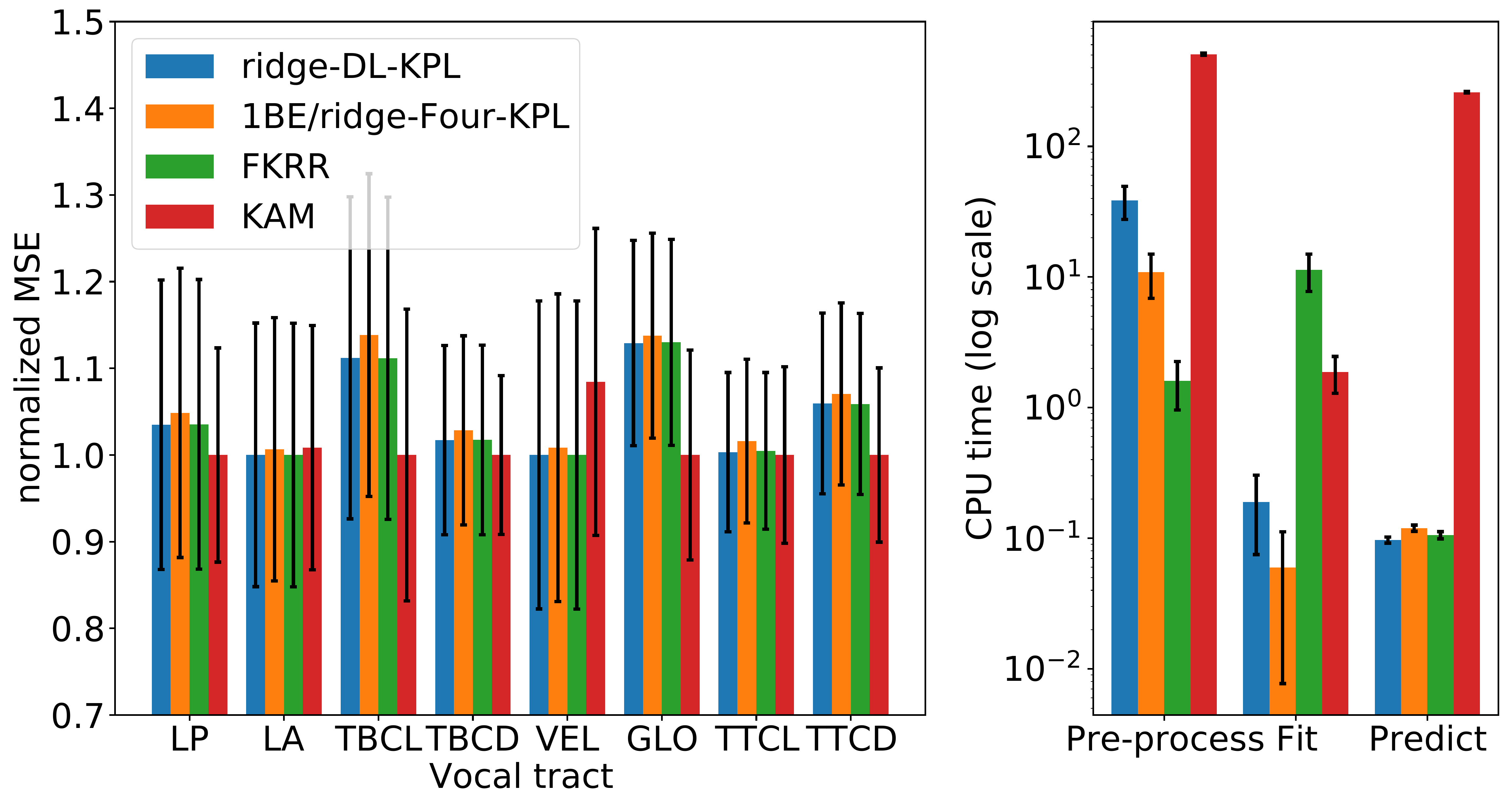}
\end{center}
\caption{MSEs and CPU times on the speech dataset.}
\label{fig :: speech}
\end{figure}

\textbf{Dataset.} We consider a speech inversion problem: from an acoustic speech signal, we estimate the underlying vocal tract (VT) configuration that produced it \citep{Richmond02}. Such information can improve performance in speech recognition systems or in speech synthesis. The dataset was introduced by \citet{MitraAl09}; it is generated by a software synthesizing words from an articulatory model. It consists of a corpus of $\nsamp=413$ pronounced words with 8 distinct VT functions: lip aperture (LA), lip protrusion (LP), tongue tip constriction degree (TTCD), tongue tip constriction location (TTCL), tongue body constriction degree (TBCD), tongue body constriction location (TBCL), Velum (VEL) and Glottis (GLO). 

\noindent{\bf Experimental setting.} To match words of varying lengths, we extend symmetrically both the input sounds and the VT functions matching the longest word. We represent the sounds using 13 mel-frequency cepstral coefficients (MFCC), the input data thus consist of vector-valued functions. We split the data as $\nsamp_{\text{train}} = 300$ and $\nsamp_{\text{test}}=113$. We normalize the output functions so that they take their values in $[-1, 1]$. To deal with the vector-valued functional inputs, we use an integral of Gaussian kernels on the standardized MFCCs (KPL, FKRR, 1BE/KPL). For KAM we take Laplace kernels for both input and output locations, and use a Gaussian kernel defined on  $\bb R^{13}$ to compare the evaluations of the standardized MFCCs (see Section \ref*{sec :: experiments :: supp} of the Supplement for details on the employed kernels). 

The MSEs for the 8 VTs (left panel) as well as an analysis of the computation times (right panel) are displayed in Figure \ref{fig :: speech}. \textit{Pre-process} entails all pre-processing operations (e. g. computing the the kernel matrices, learning the dictionary, computing the gram matrix of $\dict$), \textit{fit} measures the fitting time per se (solving the relevant linear system) and \textit{predict} measures the prediction time on the test set (for all methods, it entails computing new kernel matrices). \textit{ridge-DL-KPL} is the KPL ridge estimator with $\dict$ learnt by solving Problem \eqref{prob :: dictionary learning} with $\cali C$ and $\Omega_{\bb R^\ndict}$ as introduced in Section \ref{subsec :: dictionary learning}. \textit{1BE/ridge-Four-KPL} corresponds to 1BE (or equivalently KPL with $\Kx=k\matop I$) with $\dict$ a Fourier family. To give an order of idea, we use $30$ atoms for the learnt dictionaries while the numbers of atoms selected by cross-validation for the Fourier ones are around $100$. We do not include KE in the figure as it performed poorly on this dataset. 

\newpage
\textbf{Comments on the results.}
For 4 out of 8 VTs (LP, LA, TBCD, TTCL), the performances of the methods are comparable, with KAM being slightly more precise. On the remaining 4 VTs, ridge-DL-KPL, 1BE/ridge-Four-KPL and FKRR beat KAM on one (VEL) and are beaten by KAM on the 3 other (TBCL, GLO, TTCD). This could be explained by the fact that KAM predicts locally the functions while the other three methods have more of a global approach. Depending on the properties of the functions and the nature of the dependency between input and output functions, one or the other could be more favorable. However KAM's main weakness is its computational cost for pre-processing and prediction, which makes it unpractical to use on medium-sized datasets and impossible to use on larger ones. The particularily time-consuming operation in question is the computation of an analogous to the kernel matrix (see Section \ref{subsec :: related main}). The three other methods display very close MSEs, with 1BE/ridge-Four-KPL being a bit less precise than the two others. Ridge-DL-KPL and FKRR perform equally well. However for the former the main computational burden comes from a pre-processing operation (learning the dictionary), which is performed only once per dataset (or once per fold in a cross-validation); whereas for the latter it comes from fitting the method, which must be done many times so as to tune its parameters. Moreover for Ridge-DL-KPL, once a number of atoms yielding a good approximation has been found and the dictionary has been learnt, no further tuning must be performed for the outputs,  whereas for FKRR an output kernel must be chosen. 

\section{CONCLUSION} \label{sec:6-conclusion}
We introduced PL, a general dictionary-based framework to address FOR. It can be used with a wide class of functional losses and non orthonormal dictionaries. Through an extensive study in the context of vv-RKHSs, we illustrated some aspects of its versatility and demonstrated that the approach is efficient and can be backed theoretically in some cases. For future research, PL could be instantiated using other hypothesis classes than vv-RKHS and the possibilities offered by dictionary learning could be investigated further. 

\subsubsection*{Acknowledgements}
The authors thank Zolt\'an Szab\'o for his insightful feedbacks. This work was supported by the Télécom Paris research chair on Data Science and Artificial Intelligence for Digitalized Industry and Services (DSAIDIS).

\bibliography{biblio.bib}

\appendix

\onecolumn

{\huge \bf SUPPLEMENTARY MATERIAL.}
\vspace{0.5cm}
\\
This supplementary material is organized as follows.  Section \ref{sec :: vvRKHSs} provides a reminder about operator-valued kernels and vector-valued RKHSs. In Section \ref{sec :: proofs KPL}, we detail the proofs of the propositions from Section \ref{sec :: KPL} of the main paper. In Section \ref{sec :: integral operators}, we introduce key concepts from learning theory using integral operators. Section \ref{sec :: supporting results} is dedicated to supporting results for the the theoretical proofs. The proofs of the two propositions from Section \ref{sec :: theory} of the main paper are detailed in Section \ref{sec :: proofs theory}. In Section \ref{sec :: additional}, some additional results on projection learning and kernel-based projection learning are presented. Section \ref{sec :: related} is dedicated to a detailed description of related work. Eventually, in section \ref{sec :: experiments :: supp}, experimental details supplements are laid out. The Python code is provided in a separate zip file.

\section{OVKs AND VV-RKHSs}\label{sec :: vvRKHSs}
First, we give the definition of an operator-valued kernel (OVK) and of its associated reproducing kernel Hilbert space (RKHS).

\begin{definition}
Let $\inspace $ be a space on which a kernel can be defined and let $\cali U$ be a Hilbert space. An operator-valued kernel on $\inspace \times \inspace$ is a function $ \Kx: \inspace \times \inspace \rightarrow \cali L(\cali U) $ satisfying the two following conditions:

\begin{itemize}
\item Symmetry: for all $x,x' \in \inspace $, $ \Kx (x,x') = \Kx (x',x)^ \adj $.
\item Positivity: for all $\nsamp \in \bb N^*$, for all $(x_1,...,x_\nsamp) \in \inspace^\nsamp$, for all $(u_1,...,u_\nsamp) \in \cali U^\nsamp$,

$$\sum_{\sampiter=1}^\nsamp \sum_{\sampiterbis=1}^\nsamp \langle u_\sampiter, \Kx (x_\sampiter, x_\sampiterbis) u_\sampiterbis \rangle_{\cali U} \geq 0~.$$
\end{itemize}
\end{definition}

The following theorem shows that given an OVK, it is possible to build a unique RKHS associated to it.

\begin{theorem} \citep{Senkene73,CarmeliAl10}
\label{thm :: def vvrkhs reproducing}
Let $\Kx$ be a given operator-valued kernel $\Kx: \inspace \times \inspace \rightarrow \cali L(\cali U) $. For any $x \in \inspace$, we define $\Kx_x$ as

\begin{equation} \label{eq :: def Kx}
\Kx_x: u \longmapsto \Kx_x u, ~~ \text{with} ~~ \Kx_x u: x' \longmapsto \Kx(x', x)u.
\end{equation}

There exists a unique Hilbert space $\HKx$ of functions $h:\inspace \rightarrow \cali U $ satisfying the two conditions:

\begin{itemize}
\item For all $x \in \inspace$, $\Kx_x \in \cali L(\cali U, \HKx)$.
\item For all $h \in \HKx,~ h(x) = \Kx_x^\adj h$.
\end{itemize}

The second condition is called the reproducing property; it implies that for all $x \in \inspace$, for all $u \in \cali U$ and for all $h \in \HKx, \; $

\begin{equation}\label{eq :: reproducing property supp}
\langle \Kx_x u, h \rangle_{\HKx}  = \langle u, h(x)\rangle_{\cali U}.
\end{equation}
\end{theorem}

The Hilbert space $\HKx$ is the RKHS associated to the kernel $\Kx$.

The scalar product on $\HKx$ between two functions $h_0 = \sum_{\sampiter=1}^\nsamp \Kx_{x_\sampiter} u_\sampiter$ and $h_1 = \sum_{\sampiterbis=1}^{\nsamp'} \Kx_{x_\sampiterbis'} u_\sampiterbis'$ with $x_\sampiter, x_\sampiterbis' \in \inspace, \, u_\sampiter, u_\sampiterbis' \in \cali U,$ is defined as:

\begin{equation*}
\langle h_0, h_1 \rangle_{\HKx} = \sum_{\sampiter=1}^\nsamp \sum_{\sampiterbis=1}^{\nsamp'} \langle u_\sampiter, \Kx(x_\sampiter,x_\sampiterbis') u_\sampiterbis\rangle_{\cali U}.
\end{equation*}

The corresponding norm  $\| \cdot \|_{\HKx}$ is defined by $\| h \| ^2_{\HKx} = \langle h, h \rangle_{\HKx}$.

This RKHS $\HKx$ can be built by taking the closure of the set $\{\Kx_x u~ \vert x \in \inspace, \, u \in \cali U \}$ with respect to the topology induced by $\| \cdot \|_{\HKx}$.

Finally, we state the following Lemma which we use in the subsequent proofs. We now take $\cali U = \bb R^\ndict$ in accordance with the use we make of vector-valued RKHSs (vv-RKHS) in the main paper.

\begin{lemma} \citep{MicchelliPontil05} \label{lemma :: 2.1f michellipontil}
%\begin{lemma} \label{lemma :: 2.1f michellipontil}
Let $\cali \HKx \subset \cali F(\inspace, \bb R^\ndict)$ a vv-RKHS associated to a positive matrix-valued kernel $\Kx$. Then we have for all $x \in \inspace $:
\begin{equation*}
\|h(x)\|_{\bb R^\ndict} \leq \| h \|_{\HKx}\ \|\Kx(x, x) \|^{\nicefrac 1 2}_{\boundedops(\bb R^\ndict)}.
\end{equation*}
\end{lemma}

Additionally, since for all $x \in \inspace$, $h(x) = \Kx_x^\adj h$, this implies that 
\begin{equation} \label{eq :: op norm Kx}
\| \Kx_x\|_{\boundedops(\bb R^\ndict, \HKx)} = \| \Kx_x^\adj \|_{\boundedops(\HKx, \bb R^\ndict)} \leq \|\Kx(x, x) \|^{\nicefrac 1 2}_{\boundedops(\bb R^\ndict)}~.
\end{equation}

\section{PROOFS FOR SECTION \ref{sec :: KPL}}\label{sec :: proofs KPL}
\subsection{Proof of Proposition~\ref{prop :: representer primal} from the main paper} \label{subsec :: proof of representer primal}
We recall first the proposition which corresponds to Proposition \ref{prop :: representer primal} of the main paper. Given $\Kx: \inspace \times \inspace \longmapsto \cali L(\bb R^\ndict)$ an OVK with $\HKx \subset \cali F(\inspace, \bb R^\ndict)$ its associated vv-RKHS, we want to solve the following optimization problem 

\begin{equation} \label{prob :: empirical risk dict vvrkhs supp}
\min_{h \in \HKx} \frac 1 \nsamp \sum_{\sampiter=1}^\nsamp \genelossev {y_\sampiter}{\projope h(x_\sampiter)} + \lambda  \|h \|_{\HKx}^2.
\end{equation}

\begin{proposition} \label{prop :: representer primal supp}
\textbf{(Representer theorem.)} For $\geneloss$ continuous and convex with respect to its second argument, Problem \eqref{prob :: empirical risk dict vvrkhs supp} admits a unique minimizer $h_{\prodsamp}^\lambda$. Moreover there exists $\alpha \in \bb R^{\ndict \times \nsamp} $ such that $h_{\prodsamp}^\lambda = \sum_{\sampiterbis=1}^\nsamp \Kx_{x_\sampiterbis} \alpha_\sampiterbis$ .
\end{proposition}

\begin{proof}
Since the loss is assumed to be continuous and convex with respect to the second argument, the objective $h \longmapsto \emprisk (\projope \circ h, \prodsamp) + \lambda \|h \|_{\HKx}^2 $ is thus a continuous and strictly convex function on $\HKx$ (strictly because $\lambda > 0$). As a consequence, it admits a unique minimizer on $\HKx$ \citep{BauschkeCombettes17}, which we denote by $h_{\prodsamp}^\lambda $. 

Let $\cali U: = \left \{h|~ h = \sum_{\sampiterbis=1}^\nsamp \Kx_{x_\sampiterbis} \alpha_\sampiterbis,~ \alpha \in \bb R^{\ndict \times \nsamp} \right \}$. Since it is a closed
subspace of $\HKx$, $\HKx = \cali U \oplus \cali U^{\perp}$ and we can decompose $h_{\prodsamp}^\lambda$ as $h_{\prodsamp}^\lambda = h_{\prodsamp, \cali U}^\lambda + h_{\prodsamp, \cali U^\perp}^\lambda$ with 
$(h_{\prodsamp, \cali U}^\lambda, h_{\prodsamp, \cali U^\perp}^\lambda) \in \cali U \times \cali U^{\perp}$. We recall that $\dict \in \Lsqr^\ndict = (\dict_\dictiter)_{\dictiter=1}^\ndict$ is the dictionary associated to $\projope$ (see Definition \ref{def :: projection operator} of the main paper) and we take the convention that for $\theta \in \outfuncdom$, $\dict(\theta) = (\dict_\dictiter(\theta))_{\dictiter=1}^\ndict \in \bb R^\ndict$. Now, for all $\sampiter \in [\nsamp]$ and $\theta \in \outfuncdom $,  from Theorem \ref{thm :: def vvrkhs reproducing}, we have:
\begin{equation*}
(\projope h_\prodsamp^\lambda(x_\sampiter))(\theta) = \langle \dict(\theta),  h_{\prodsamp}^\lambda(x_\sampiter)
\rangle_{\bb R^\ndict}  = \langle \Kx_{x_\sampiter} \dict(\theta), h_{\prodsamp}^\lambda \rangle_{\HKx}.
\end{equation*}
Since $\Kx_{x_\sampiter} \phi(\theta)\in \cali U$, we get that
\begin{equation*}
(\projope h_\prodsamp^\lambda(x_\sampiter))(\theta) = \langle\Kx_{x_\sampiter} \dict(\theta), h_{\prodsamp, \cali U}^\lambda \rangle_{\HKx}
 = \langle \dict(\theta), h_{\prodsamp, \cali U}^\lambda (x_\sampiter) \rangle_{\bb R^\ndict} = (\projope h_{\prodsamp, \cali U}^\lambda (x_\sampiter))(\theta) \;.
\end{equation*}
Then, on the one hand the data-attach term in the criterion to minimize is unchanged when replacing $h_\prodsamp^\lambda$ by its projection $h_{\prodsamp, \cali U}^\lambda $ onto $\cali U$. On the other hand, the penalty $\| h_{\prodsamp}^\lambda \|_{\cali H_{\Kx}}^2$ decreases if we replace $h_{\prodsamp}^\lambda$ by $h_{\prodsamp, \cali U}^\lambda$, hence we must have $h_{\prodsamp}^\lambda = h_{\prodsamp, \cali U}^\lambda $.
\end{proof}

\subsection{Proof of Proposition \ref{prop :: closed form ridge} from the main paper} \label{subsec :: proof closed ridge supplementary}
First, we recall the proposition which corresponds to Proposition \ref{prop :: closed form ridge} of the main paper. We want to solve the following (Problem (\ref{prob :: problem in ridge form}) from the main paper):

\begin{align}
\min_{\alpha \in \bb R^{\ndict \times \nsamp}} \frac 1 \nsamp \left \| \outsamp - \projope_{(\nsamp)} \pmb{\Kx} \mattovec(\alpha) \right \|_{\Lsqr^\nsamp}^2 + \lambda \langle \mattovec(\alpha), \pmb {\Kx} \mattovec(\alpha)\rangle_{\bb R^{\ndict \nsamp}}.\label{prob :: problem in ridge form supp}
\end{align}

\begin{proposition} \label{prop :: closed form ridge supp}
\textbf{(Ridge solution)}
The minimum in Problem \eqref{prob :: problem in ridge form supp} is achieved by any $\alpha^* \in \bb R^{\ndict \times \nsamp} $ verifying
\begin{equation} \label{eq :: closed form ridge supp}
\left( \pmb {\Kx} (\projope^\adj \projope)_{(\nsamp)} \pmb {\Kx} + \nsamp \lambda \pmb {\Kx}
\right ) \mattovec (\alpha^*):=  \pmb {\Kx} \projope^\adj_{(\nsamp)} \outsamp.
\end{equation}
Such $\alpha^*$ exists. Moreover if $\pmb {\Kx}$ is full rank then $\left( (\projope^\adj \projope)_{(\nsamp)} \pmb {\Kx} + \nsamp \lambda \pmb {\matop I}
\right )$ is invertible and $\alpha^*$ is such that
\begin{equation}
\mattovec (\alpha^*) =  \left( (\projope^\adj \projope)_{(\nsamp)} \pmb {\Kx} + \nsamp \lambda \pmb {\matop I}
\right )^{-1} \projope^\adj_{(\nsamp)} \outsamp.
\end{equation}
We then define the ridge estimator as $ \regridge := \sum_{\sampiterbis=1}^\nsamp \Kx_{x_\sampiterbis} \alpha^*_\sampiterbis$.
\end{proposition}

\begin{proof}
For $\pmb \alpha \in \bb R^{\ndict \nsamp}$ we consider the objective function 
$$
\frac 1 \nsamp \left \| \projope_{(\nsamp)} \pmb{\Kx} \pmb \alpha\right\|_{\Lsqr^\nsamp}^2 - \frac 2 \nsamp \langle  \outsamp,\projope_{(\nsamp)}\pmb{\Kx} \pmb \alpha \rangle_{\Lsqr^\nsamp} +\lambda \langle \pmb \alpha,\pmb{\Kx} \pmb \alpha \rangle_{\bb R^{\ndict \nsamp}}\;.
$$

Up to an additional term not dependant on $\pmb \alpha$, this corresponds to the objective function in Problem \eqref{prob :: problem in ridge form supp} where we have set $\pmb \alpha = \mattovec (\alpha)$ to simplify the exposition. 

Using that
$(\projope_{(\nsamp)})^\adj \projope_{(\nsamp)}=\projope_{(\nsamp)}^\adj \projope_{(\nsamp)} = (\projope^\adj\projope)_{(\nsamp)}$, that $ \pmb{\Kx}^\adj= \pmb{\Kx}$ and multiplying by $\nsamp$, we can consider as objective function
\begin{align*}
V(\pmb \alpha) &:=\langle  \pmb \alpha, \pmb{\Kx}(\projope^\adj\projope)_{(\nsamp)} \pmb{\Kx} \pmb \alpha \rangle_{\bb R^{\ndict \nsamp}} - 2 \langle \projope_{(\nsamp)}^\adj \outsamp,\pmb{\Kx} \pmb \alpha \rangle_{\bb R^{\ndict \nsamp}} + \nsamp \lambda \langle \pmb \alpha, \pmb{\Kx} \pmb \alpha \rangle_{\bb R^{\ndict \nsamp}}\\
& = \langle  \pmb \alpha, \pmb{\Kx} \left((\projope^\adj\projope)_{(\nsamp)} \pmb{\Kx} + \nsamp \lambda\, \pmb {\matop I} \right ) \pmb \alpha \rangle_{\bb R^{\ndict \nsamp}} - 2 \langle \projope_{(\nsamp)}^\adj  \outsamp, \pmb{\Kx} \pmb \alpha\rangle_{\bb R^{\ndict \nsamp}}    \;.
\end{align*}
Let $\pmb \alpha^* \in \bb R^{\ndict \nsamp}$ be such that 
\begin{equation*} 
\left(\pmb {\Kx}(\projope^\adj \projope)_{(\nsamp)} \pmb {\Kx} + \nsamp \lambda \pmb {\Kx} \right )\pmb \alpha^*= \pmb {\Kx}\projope^\adj_{(\nsamp)} \outsamp \;.
\end{equation*}
We want to prove that $\pmb \alpha^*$ is then a solution to Problem \eqref{prob :: problem in ridge form supp}. Observe now that
\begin{align}
\langle  \pmb \alpha^*, \pmb{\Kx} \left((\projope^\adj \projope)_{(\nsamp)} \pmb{\Kx} + \nsamp \lambda\, \pmb {\matop I}\right ) \pmb \alpha \rangle_{\bb R^{\ndict \nsamp}}
&=\langle  \pmb \alpha, \pmb{\Kx} \left((\projope^\adj\projope)_{(\nsamp)}\pmb{\Kx}+ \nsamp \lambda \, \pmb {\matop I} \right ) \pmb \alpha^*\rangle_{\bb R^{\ndict \nsamp}} \nonumber \\
&=\langle \pmb \alpha, \pmb{\Kx} \projope_{(\nsamp)}^\adj  \outsamp\rangle_{\bb R^{\ndict \nsamp}} \nonumber \\
&=\langle \projope_{(\nsamp)}^\adj \outsamp, \pmb{\Kx} \pmb \alpha\rangle_{\bb R^{\ndict \nsamp}}\; \label{eqline :: relation alpha star alpha}.
\end{align}

Using Equation \eqref{eqline :: relation alpha star alpha}, we deduce that
\begin{align*}
V(\pmb \alpha) &= \langle \pmb \alpha, \pmb{\Kx} \left( (\projope^\adj \projope)_{(\nsamp)} \pmb{\Kx} + \nsamp \lambda\, \pmb {\matop I} \right ) \pmb \alpha \rangle_{\bb R^{\ndict \nsamp}} - 2 \langle \projope_{(\nsamp)}^\adj \outsamp, \pmb{\Kx} \pmb \alpha\rangle_{\bb R^{\ndict \nsamp}} \\
&=\langle  \pmb \alpha -\pmb \alpha^*, \pmb{\Kx} \left( (\projope^\adj \projope)_{(\nsamp)} \pmb{\Kx} + \nsamp \lambda\, \pmb {\matop I} \right ) \left(\pmb \alpha - \pmb \alpha^* \right)\rangle_{\bb R^{\ndict \nsamp}}\\
&\phantom{=}+ \langle \pmb \alpha^*, \pmb{\Kx} \left( (\projope^\adj \projope)_{(\nsamp)}\pmb{\Kx} + \nsamp \lambda\, \pmb {\matop I} \right ) \pmb \alpha^*\rangle_{\bb R^{\ndict \nsamp}}.
\end{align*}
Since $\pmb{\Kx} \left( (\projope^\adj \projope)_{(\nsamp)}\pmb{\Kx} + \nsamp \lambda\, \pmb {\matop I} \right )$ is a non-negative symmetric matrix, we conclude that $V(\pmb \alpha)$ is minimal at $\pmb \alpha = \pmb \alpha^*$.

We now show that Equation \eqref{eq :: closed form ridge supp} always has a solution $\pmb \alpha^*$ in $\bb R^{\ndict \nsamp}$ and conclude with the special case where $\pmb {\Kx}$ is full rank.  Note that
$\left(\pmb {\Kx}(\projope^\adj \projope)_{(\nsamp)} \pmb {\Kx} + \nsamp \lambda \pmb {\Kx} \right )$ is a positive symmetric matrix and its null space is exactly that of $\pmb {\Kx}$. Hence it is bijective on the image of $\pmb {\Kx}$, which shows that Equation \eqref{eq :: closed form ridge supp} always has a solution. If $\pmb {\Kx}$ is moreover full rank then
$$
\left( (\projope^\adj \projope)_{(\nsamp)} \pmb {\Kx} + \nsamp\lambda \pmb {\matop I} \right )  = \pmb {\Kx}^{-1} \left( \pmb {\Kx}(\projope^\adj \projope)_{(\nsamp)} \pmb {\Kx} + \nsamp \lambda \pmb {\Kx} \right )
$$
is also invertible and we can simplify by  $\pmb {\Kx}$ on both sides of Equation \eqref{eq :: closed form ridge supp} and obtain the claimed formula for $\pmb \alpha^*$. Taking $\alpha^* \in \bb R^{\ndict \times \nsamp}$ such that $\mattovec(\alpha^*) = \pmb \alpha^*$ yields the desired results.
\end{proof}

\section{LEARNING THEORY AND INTEGRAL OPERATORS} \label{sec :: integral operators}
This section is devoted to the study of Problem \eqref{prob :: empirical risk dict vvrkhs supp} for the functional square loss in the framework of integral operators \citep{CaponettoDevito05, CaponettoDevito07, SmaleZhou07}. In Section \ref{subsec :: excess risk reformulation} the expected risk and the excess risk are reformulated in terms of two operators of interest. In Section \ref{subsec :: empirical and closed-forms}, we introduce empirical approximations of those operators. From there we can reformulate the minimizer of the regularized empirical risk in terms of those empirical operators. 

\subsection{Excess risk reformulation} \label{subsec :: excess risk reformulation}
The first goal is to characterize the minimizer of the expected risk using two operators of interest as in \citep{CaponettoDevito07}. Using this characterization, a closed form for the excess risk of any regressor $\projope \circ h$ is derived. 

Considering the functional square loss, we recall the definition of the expected risk $\exprisk$ of a regressor $\regfunc \in \cali F(\inspace, \Lsqr)$
\begin{equation} \label{eq :: expected risk}
\exprisk (\regfunc):= \bb E_{(\inva, \outva) \sim \funcprob} \left[ \sqrlossev{\outva}{\regfunc(\inva)} \right],
\end{equation} \label{eq :: empirical risk}
as well as that of its empirical risk on a sample $\prodsamp$
\begin{equation}
\emprisk(\regfunc, \prodsamp) := \frac 1 \nsamp \sum_{\sampiter=1}^\nsamp \sqrlossev{y_\sampiter}{\regfunc(x_\sampiter)}.
\end{equation}

Let us introduce $\rmc L^2(\prodspace, \funcprob, \Lsqr)$ the space of square integrable functions from $\prodspace$ to $\Lsqr$ with respect to the measure $\funcprob$ endowed with the scalar product 
$$ \langle \psi_0, \psi_1 \rangle_{\funcprob} = \int_{\prodspace} \langle \psi_0(x, y), \psi_1(x, y) \rangle_{\Lsqr} ~ \mathrm d \funcprob(x, y),$$ 
and its associated norm $\|. \|_{\funcprob}$. Then, the expected risk in Equation \eqref{eq :: expected risk} of a regressor $\regfunc$ can then be equivalently formulated as
 
\begin{equation} \label{eq :: true risk functional :: supp}
\exprisk(\regfunc) = \|\regfunc \circ X - Y \|_{\funcprob}^2,
\end{equation}

where we have defined  $X : (x, y) \in \prodspace \longmapsto x \in \cali X$ and $Y \in \rmc L^2(\prodspace, \funcprob, \Lsqr)$ as $Y : (x, y) \in \prodspace \longmapsto y \in \Lsqr$.

We wish to study the excess risk of any regressor of the form $f = \projope \circ h$. To that end, we define the operator $\matop A_{\projope}: \HKx \longrightarrow \rmc L^2(\prodspace, \funcprob, \Lsqr)$ as 

\begin{equation} \label{eq :: def Aphi}
\matop A_{\projope}: h \longmapsto \matop A_{\projope}h ~~\text{with}~~ (\matop A_{\projope}h): (x, y) \in \prodspace \longmapsto \projope \Kx_x^{\adj} h .
\end{equation}

%\begin{remark} The second variable $y \in \Lsqr$ is a dummy variable, however defining $\matop A_{\projope}$ in this way is interesting because of the resulting adjoint operator $\matop A_{\projope}^\adj$.
%\end{remark}

We can reformulate the expected risk in terms of $\matop A_{\projope}$ for any $h \in \HKx$, 

\begin{equation} \label{eq :: projected risk with Aphi}
\| \matop A_{\projope} h - Y \|_{\funcprob}^2 = \int_{\prodspace} \| \projope \Kx_x^\adj h - y \|_{\Lsqr}^2 ~ \mathrm d \funcprob(x, y) = \int_{\prodspace} \| \projope 
h(x) - y \|_{\Lsqr}^2 ~ \mathrm d \funcprob(x, y) = \exprisk(\projope \circ h).
\end{equation}

We now define $\matop T_{\projope}$ as $\matop T_{\projope}:= \matop A_{\projope}^\adj \matop A_{\projope}$.

\begin{lemma}
Assume that there exists $h_{\HKx} \in \HKx$ such that 
$$h_{\HKx} := \inf_{h \in \HKx} \cali R(\projope \circ h).$$
Then, for all $h \in \HKx$, 
\begin{equation} \label{eq :: condition hHK ortho}
\langle h,  \matop T_{\projope} h_{\HKx} - \matop A_{\projope}^{\adj} Y \rangle_{\HKx} = 0 ;
\end{equation}
or equivalently:
\begin{equation}\label{eq :: condition hHK}
\matop T_{\projope} h_{\HKx} = \matop A_{\projope}^\adj Y,
\end{equation}
with $Y \in \rmc L^2(\prodspace, \funcprob, \Lsqr)$ denoting the function $Y: (x, y) \longmapsto y$. 
\end{lemma}

\begin{proof}
We use the formulation of the expected risk from Equation \eqref{eq :: projected risk with Aphi}.
The function $h \longmapsto \exprisk(\projope \circ h) = \| \matop A_{\projope} h - Y \|_{\funcprob}^2$ is convex as a convex function composed with an affine mapping. Its differential is given by

$$ \diffop \exprisk(\projope \circ h_{\HKx}) (h) = 2 \langle \matop A_{\projope} h, \matop A_{\projope} h_{\HKx} - Y \rangle_{\funcprob} = 2 \langle h,  \matop A_{\projope}^{\adj} \matop A_{\projope} h_{\HKx} - \matop A_{\projope}^{\adj} Y \rangle_{\HKx} = 2 \langle h,  \matop T_{\projope} h_{\HKx} - \matop A_{\projope}^{\adj} Y \rangle_{\HKx}.$$ 

We then must have for all $h \in \HKx$, 
\begin{equation*}
\langle h,  \matop T_{\projope} h_{\HKx} - \matop A_{\projope}^{\adj} Y \rangle_{\HKx} = 0 .
\end{equation*}
\end{proof}

Using the formulation of the expected risk from Equation \eqref{eq :: projected risk with Aphi} as well as the characterization of $h_{\HKx}$ in Equation \eqref{eq :: condition hHK ortho}, for any $h \in \HKx$, we can then reformulate the excess risk of $h$ as a distance in $\HKx$ between $h$ and $h_{\HKx}$ taken through the operator $\rmc T_\projope$.

\begin{lemma}\label{lemma :: key rewriting of the risk}
We have that for any $h \in \HKx$, 
\begin{equation} \label{eq :: key rewriting of the risk}
\exprisk(\projope \circ h ) - \exprisk(\projope \circ h_{\HKx}) = \|\sqrt{\matop T_{\projope}}(h - h_{\HKx}) \|_{\HKx}^2.
\end{equation}
\end{lemma}

\begin{proof}
\begin{align*}
\exprisk (\projope \circ h ) - \exprisk (\projope \circ h_{\HKx}) & = \|\matop A_{\projope} h - Y \|_{\funcprob}^2 - \|\matop A_{\projope} h_{\HKx} - Y \|_{\funcprob}^2  \\
& = \|\matop A_{\projope} (h - h_{\HKx}) \|_{\funcprob} ^2 + 2 \langle \matop A_{\projope} (h - h_{\HKx}), \matop A_{\projope} h_{\HKx} - Y \rangle_{\funcprob}  \\
& = \| \matop A_{\projope}(h - h_{\HKx}) \|_{\funcprob}^2,
\end{align*}

where we have used Equation \eqref{eq :: condition hHK ortho}. 
%\dimitri{Pas besoin de faire appel à la décomposition polaire}
%\dimitri{<(),A^#A()>=<(),T^2()>=||T()||^2}
Since we have the following polar decomposition $\matop A_{\projope} = \matop U \sqrt{\matop A_{\projope}^\adj \matop A_{\projope}} = \matop U \sqrt{\matop T_{\projope}}$ with $\matop U$ a partial isometry from the closure of $\rangespace(\sqrt{\matop T_\projope})$ onto the closure of $\rangespace(\matop A_\projope)$, %---see for instance Theorem 7.20 in \citet{Weidmann80}.

$$\|\matop A_{\projope}(h - h_{\HKx}) \|_{\funcprob} = \|\matop  U \sqrt{\matop T_{\projope}}(h - h_{\HKx}) \|_{\funcprob} = \|\sqrt{\matop T_{\projope}}(h - h_{\HKx}) \|_{\HKx}. $$
\end{proof}

Such reformulation enables us to decompose the excess risk in terms that we can easily control using concentration inequalities in Hilbert spaces. 

\subsection{Empirical approximations and closed form solutions} \label{subsec :: empirical and closed-forms}
We now define empirical approximations of the operators $\matop A_{\projope}$ and $\matop T_{\projope}$. Using those approximations, we can derive a closed-form for the minimizer of the regularized expected risk. We utilize that closed-form to bound the excess risk in the subsequent proof. 

To define those approximations, we need to precise the integral expressions of $\matop A_{\projope}^\adj$ and $\matop T_{\projope}$. This is the object of the following lemma, which is almost a restatement of Proposition 1 from \citet{CaponettoDevito05}, as a consequence, we do not re-write the proof here.

Let us define for all $x \in \inspace$ the operators $\Kx_{x, \projope}: = \Kx_x \projope^\adj$ and $\matop T_{x, \projope}:= \Kx_{x, \projope} \Kx_{x, \projope}^\adj$.

\begin{lemma} \label{lemma :: A Phi adjoint and T Phi}
For $\psi \in \rmc L^2(\prodspace, \funcprob, \Lsqr)$, the adjoint of $\matop A_{\projope}$ applied to $\psi$ is given by 
\begin{equation} \label{eq :: adj of Aphi}
\matop A_{\projope}^{\adj} \psi = \int_{\prodspace} \Kx_{x, \Phi} \psi(x, y) ~ \mathrm d \funcprob(x, y),
\end{equation}
with the integral converging in $\HKx$.
And $\matop A_{\projope}^{\adj} \matop A_{\projope}$ is the Hilbert Schmidt operator on $\HKx $ given by
\begin{equation} \label{eq :: Tphi from Aphi}
\matop A_{\projope}^{\adj} \matop A_{\projope} = \matop T_{\projope} = \int_{\cali X} \matop T_{x, \projope} ~ \mathrm d \funcprob_{\rmc X}(x),
\end{equation}
with the integral converging in $\hsops(\HKx)$.
\end{lemma}

Empirical approximations of the operators $\matop A_{\projope}$ and $\matop T_{\projope}$ can then straightforwardly be set as

\begin{align*}
\matop A_{\insamp, \projope}^{\adj} \mathbf w &= \frac 1 \nsamp \sum_{\sampiter=1}^{\nsamp} \Kx_{x_\sampiter, \projope} w_\sampiter,~~ \mathbf w = (w_\sampiter)_{\sampiter=1}^\nsamp \in \Lsqr^\nsamp. \\
(\matop A_{\insamp, \projope} h)_\sampiter & = \Kx_{x_\sampiter, \projope}^{\adj} h = \projope h(x_\sampiter),~~ h \in \HKx,~~ \forall \sampiter \in [\nsamp]. \\
\matop T_{\insamp, \projope} & = \matop A_{\insamp, \projope}^{\adj} \matop A_{\insamp, \projope} = \frac 1 \nsamp \sum_{\sampiter=1}^{\nsamp} \matop T_{x_\sampiter, \projope}.
\end{align*}

Defining the regularized empirical risk of $\projope \circ h$ for any $h \in \HKx$ as 
\begin{equation*}
\emprisk^\lambda (\projope \circ h, \prodsamp):= \emprisk(\projope \circ h, \prodsamp) + \lambda \|h \|_{\HKx}^2 = \frac 1 \nsamp \sum_{\sampiter=1}^{\nsamp} \|\Kx_{x_\sampiter, \projope}^\adj h - y_\sampiter \|^2_{\Lsqr} + \lambda \|h \|_{\HKx}^2, 
\end{equation*}
the following closed form for its minimizer can be derived.
\begin{lemma} \label{lemma :: h z lambda close form}
There exists a unique minimizer $\regridge$ of $h \in \HKx \longmapsto \emprisk^\lambda (\projope \circ h, \prodsamp)$ which is given by 
\begin{equation} \label{eq :: h z lambda close form}
\regridge := (\matop T_{\insamp, \projope} + \lambda \matop I)^{-1} \matop A_{\insamp, \projope}^\adj \outsamp~.
\end{equation}
\end{lemma}

\begin{proof}
Since $\lambda > 0$, $h \longmapsto \emprisk^\lambda (\projope \circ h, \prodsamp)$ is strictly convex. As it is continuous, there exist a unique minimizer which can be found by setting the differential to zero. 

\begin{align*}
\diffop \emprisk^\lambda (\projope \circ h_0, \prodsamp) (h_1) &= \frac 2 \nsamp \sum_{\sampiter=1}^\nsamp \langle \Kx_{x_\sampiter, \projope}^\adj h_0 - y_i, \Kx_{x_\sampiter, \projope}^\adj h_1 \rangle_{\Lsqr} + 2 \lambda \langle h_0, h_1 \rangle_{\HKx} \\
& = 2 \left \langle \left ( \frac 1 \nsamp \sum_{\sampiter=1}^\nsamp \matop T_{x_\sampiter, \projope} + \lambda \right ) h_0 - \frac 1 \nsamp \sum_{\sampiter=1}^\nsamp \Kx_{x_\sampiter, \projope} y_i, h_1 \right \rangle_{\HKx} \\
& = 2 \langle (\matop T_{\insamp, \projope} + \lambda \matop I) h_0 - \matop A_{\insamp, \projope}^\adj \outsamp, h_1 \rangle_{\HKx}.
\end{align*}

As a consequence, $\regridge $ is characterized by 
$$ (\matop T_{\insamp, \projope} + \lambda \matop I) \regridge - \matop A_{\insamp, \projope}^\adj \outsamp = 0.$$

Since $\matop T_{\insamp, \projope}$ is positive and $\lambda > 0$, $ (\matop T_{\insamp, \projope} + \lambda \matop I)$ is invertible and thus 
$$ \regridge = (\matop T_{\insamp, \projope} + \lambda \matop I)^{-1} \matop A_{\insamp, \projope}^\adj \outsamp.$$
\end{proof}

Importantly, $ \regridge $ is the same object as the ridge estimator from Proposition \ref{prop :: closed form ridge supp} which is why we have used the same notation. The representation in terms of operators introduced above is however needed to carry out an excess risk analysis.

\section{SUPPORTING RESULTS FOR SECTION \ref{sec :: proofs theory}} \label{sec :: supporting results}
This section is dedicated to technical results on which the proofs in Section \ref{sec :: proofs theory} rely.

\subsection{Riesz families and projection operator} \label{subsec :: riesz and projection}
The proofs in the next section strongly relies on general inequalities on Riesz families and on the associated
projection operator $\projope$, that we state and prove in this section.

Using the definition of a Riesz family we have

\begin{lemma}\label{lem:riesz-supp}
Let $\dict:= (\dict_1,..., \dict_\ndict)$ be a Riesz family, let $\projope$ be its associated projection operator (see Definition \ref{def :: projection operator} from the main paper). Then
\begin{align}
\| \projope \|_{\boundedops(\bb R^\ndict, \Lsqr)} \leq C_\dict \label{eq:riesz:supp}\\
\| \projope^\adj \|_{\boundedops(\Lsqr, \bb R^\ndict)} \leq C_\dict \label{eq:riesz adj:supp} \\
\| \projope^\adj \projope \|_{\boundedops (\bb R^\ndict)} \leq C_\dict ^2 ~ .\label{eq :: ope norm Phi adj Phi}
\end{align}
\end{lemma}

\begin{proof}
Equation \eqref{eq:riesz:supp} is a direct consequence of the definition of a Riesz family (Definition \ref{def :: riesz family} from the main paper). 
Since the operator $\projope$ is bounded, $\| \projope^\adj \|_{\boundedops(\Lsqr, \bb R^\ndict)} = \| \projope \|_{\boundedops(\bb R^\ndict, \Lsqr)}$ implying Equation \eqref{eq:riesz adj:supp}. Finally combining the two inequalities yields Equation \eqref{eq :: ope norm Phi adj Phi}.
\end{proof}

\subsection{Bound on Hilbert-Schmidt norm of $\matop T_{x, \projope}$}

In the subsequent proof, we need to derive concentration results on $\matop T_{\insamp, \projope}$. To that end, we need to bound the Hilbert-Schimdt norm of $\matop T_{x, \projope}$.

For all $x \in \inspace$, we recall the definition of the following operators
\begin{itemize} 
\item $\Kx_{x, \projope}: \Lsqr \longrightarrow \HKx$ is defined by $\Kx_{x, \projope} := \Kx_x \projope^{\adj}$ with $\Kx_x$ as defined in Equation \eqref{eq :: def Kx}.
\item $\matop T_{x, \projope}: \HKx \longrightarrow \HKx$ is defined as $\matop T_{x, \projope} := \Kx_{x, \projope} \Kx_{x, \projope}^\adj$.
\end{itemize}

Observe that $\matop T_{x, \projope}$ is of finite rank and positive. We can then deduce the following bound on its Hilbert-Schmidt norm. 

\begin{lemma}
Assume that there exists $\kappa \geq 0$ such that for all $x \in \inspace$,
\begin{equation} \label{eq :: anticipated bounded kernel}
\|\Kx (x, x) \|_{\boundedops(\bb R^\ndict)} \leq \kappa, 
\end{equation}
then for all $x \in \inspace$,
\begin{equation} \label{eq :: bound HS norm Tx Phi}
\| \matop T_{x, \projope} \|_{\hsops(\HKx)} \leq \sqrt{\ndict} \kappa C_{\dict}^2.
\end{equation}
\end{lemma}

\begin{proof}
For all $x \in \inspace$, $\rank(\matop T_{x, \projope}) \leq \ndict$. Let $(e_\dictiter)_{\dictiter=1}^{\rank(\matop T_{x, \projope})}$ be an orthonormal basis of $\rangespace(\matop T_{x, \projope})$. We complete it to $(e_\dictiter)_{\dictiter \in \bb N^*}$ to be an orthonormal basis of $\HKx$. Since $\rangespace(\matop T_{x, \projope})$ is a finite dimensional subspace of $\HKx$ and $\matop T_{x, \projope}$ is self adjoint, we have that $\rangespace(\matop T_{x, \projope}) = \nullspace(\matop T_{x, \projope})^{\perp}$. As a consequence, for all $\dictiter > \rank(\matop T_{x, \projope})$, $\matop T_{x, \projope} e_{\dictiter} = 0$, which implies
$$ 
\| \matop T_{x, \projope} \|_{\hsops(\HKx)}^2 = \sum_{\dictiter=1}^{ \rank(\matop T_{x, \projope})} \langle \matop T_{x, \projope} e_\dictiter, \matop T_{x, \projope} e_\dictiter \rangle_{\HKx} = \sum_{\dictiter=1}^{ \rank(\matop T_{x, \projope})} \langle \Kx_x^{\adj} e_\dictiter, \projope^\adj \projope \Kx(x, x) \projope^\adj \projope \Kx_x^{\adj} e_\dictiter \rangle_{\bb R^\ndict}.
$$

Using Cauchy-Schwartz in the previous expression along with Equation \eqref{eq :: ope norm Phi adj Phi}, Equation \eqref{eq :: anticipated bounded kernel} and Equation \eqref{eq :: op norm Kx} we have that 

$$
\| \matop T_{x, \projope} \|_{\hsops(\HKx)}^2 \leq C_{\dict}^4 \kappa\sum_{\dictiter=1}^{ \rank(\matop T_{x, \projope})} \| \Kx_x^{\adj} e_\dictiter \|_{\bb R^\ndict}^2 \leq C_{\dict}^4 \kappa^2 \rank(\matop T_{x, \projope}) \leq \ndict C_{\dict}^4 \kappa^2,
$$
which achieves the proof.
\end{proof}

\subsection{Concentration results} \label{subsubsec :: concentration results}
We now state two concentration inequalities that we use to control the different terms in our decomposition of the excess risk in Section \ref{sec :: proofs theory}. We also introduce Lemma \ref{lemma :: diff square root norm} which we use to deduce concentration properties of $\sqrt{\matop T_{\insamp, \projope}}$ from  concentration properties of $\matop T_{\insamp, \projope}$.

The following is a direct consequence of a Bernstein inequality for independent random variables in a separable Hilbert space---see Proposition 3.3.1 in \citep{Yurinsky95} or Theorem 3 in \citep{PinelisSakhanenko86}. It corresponds to Proposition 2 in \citep{CaponettoDevito07}.

\begin{lemma} \label{lemma :: vv bernstein}
Let $\xi$ be a random variable taking its values in a real separable Hilbert space $\cali K$ such that there exist $H \geq 0$ and $\sigma \geq 0$ such that
\begin{align*}
\| \xi \|_{\cali K} &\leq \frac H 2~\text{almost surely, and} \\
\bb E[\| \xi \|_{\cali K}^2] &\leq \sigma^2 .
\end{align*}
Let $\nsamp \in \bb N$ and $(\xi_1,...,\xi_\nsamp)$ be i.i.d. realizations of $\xi$. Let $0 < \eta < 1$, then
\begin{equation*}
\bb P \left [\left \| \frac 1 \nsamp \sum_{\sampiter=1}^\nsamp \xi_\sampiter - \bb E[\xi] \right \|_{\cali K} \leq 2 \left (\frac H \nsamp + \frac \sigma {\sqrt \nsamp} \right ) \log \frac 2 \eta \right ] \geq 1 - \eta.
\end{equation*}
\end{lemma}

We introduce a variant of the previous Lemma for independent variables that are not necessarily identically distributed. It stems from the same Bernstein inequality \citep{PinelisSakhanenko86, Yurinsky95}. We need it to treat the case where the output functions are partially observed in Section \ref{sec :: proofs theory}. The proof is almost similar to that of Lemma \ref{lemma :: vv bernstein} which can be found in \citet{CaponettoDevito07}, so we do not rewrite it here.

\begin{lemma} \label{lemma :: vv bernstein inde}
Let $(\rmc U_\sampiter)_{\sampiter=1}^\nsamp$ be independent random variables taking their values in a real separable Hilbert space $\cali K$ such that for all $\sampiter \in [\nsamp]$
\begin{align*}
\bb E[\rmc U_\sampiter] = 0,
\end{align*}
and there exist $H \geq 0$ and $\sigma \geq 0$ such that for all $\sampiter \in [\nsamp]$
\begin{align*}
\| \rmc U_\sampiter \|_{\cali K} & \leq \frac H 2~\text{almost surely, and }\\
\bb E[\| \rmc U_{\sampiter} \|_{\cali K}^2] &\leq \sigma^2.
\end{align*}
Let $0 < \eta < 1$, then
\begin{equation*}
\bb P \left [\left \| \frac 1 \nsamp \sum_{\sampiter=1}^\nsamp \rmc U_\sampiter \right \|_{\cali K} \leq 2 \left (\frac H \nsamp + \frac \sigma {\sqrt \nsamp} \right ) \log \frac 2 \eta \right ] \geq 1 - \eta.
\end{equation*}
\end{lemma}

Finally, we need the following result to state concentration results on the square root of Hilbert-Schmidt operators. It corresponds to Theorem X.1.1 in \citet{Bhatia97} where it is stated for positive symmetric matrices. Their proof remains however fully valid for positive bounded operators defined on real separable Hilbert spaces.

\begin{lemma} \label{lemma :: diff square root norm}
Let $\cali K$ be a real separable Hilbert space, let $\matop A, \matop B \in \boundedops(\cali K)$ be two positive operators. Then, we have
\begin{equation*}
\| \sqrt{\matop A} - \sqrt{\matop B} \|_{\boundedops(\cali K)} \leq \sqrt { \| \matop A - \matop B \|}_{\boundedops(\cali K)}.
\end{equation*}
\end{lemma}

\section{PROOFS FOR SECTION \ref{sec :: theory}} \label{sec :: proofs theory}
\subsection{Proof of Proposition~\ref{prop :: excess risk bound} from the main paper} \label{subsec :: proof excess risk bound}

We recall the assumptions, as well as the proposition itself which corresponds to Proposition \ref{prop :: excess risk bound} of the main paper. 

\begin{assumption} \label{ass :: conditions on kernel supp}
$\Kx$ is a vector-valued continuous kernel and there exists $\kappa > 0$ such that for $x \in \inspace$, $\| \Kx(x, x) \|_{\cali L(\bb R^\ndict)} \leq \kappa$.
\end{assumption}
\begin{remark}
We suppose that $\kappa$ is independant from $\ndict$. This is for instance the case if for $x \in \inspace$, $\Kx(x, x)$ is diagonal or block diagonal with bounded coefficients. More generally, we can rely on the fact that $\kappa$ is bounded by the maximal $\|\cdot\|_1$-norm of the columns of $\Kx(x, x)$, which can easily be imposed to be be independent of $\ndict$.
\end{remark}
\begin{assumption} \label{ass :: dictionary orthonormality :: supp}
The dictionary $\phi$ is a normed Riesz family in $\Lsqr$ with upper constant $C_{\phi}.$
\end{assumption}
\begin{remark}
We do not use the lower constant $c_\dict$.
\end{remark}
\begin{assumption} \label{ass :: existence of risk minimizer :: supp}
There exist $h_{\HKx} \in \HKx$ such that $h_{\HKx} = \inf_{h \in \HKx} \cali R(\projope \circ h).$
\end{assumption}
\begin{remark}
This is a standard assumption \citep{CaponettoDevito07, BaldassarreRosascoBarla12, LiAl19}, it implies the existence of a ball of radius $R>0$ in $\HKx$ containing $h_{\HKx}$, as a consequence
\begin{equation} \label{eq :: definition of R}
\|h_{\HKx} \|_{\HKx} \leq R.
\end{equation}
\end{remark}
%\begin{assumption} \label{ass :: as boundedness of outputs}
%There exist $L > 0$ such that almost surely, $\| \rmc Y \|_{\Lsqr}\leq L$. 
%\end{assumption}
\begin{assumption} \label{ass :: value bounded Y :: supp}
There exists $L \geq 0$ such that for all $\theta \in \outfuncdom$, almost surely $|\outva(\theta)| \leq L$.
\end{assumption}
\begin{remark}
This implies that almost surely $\| \outva \|_{\Lsqr} \leq L$.
\end{remark}

We now state Proposition \ref{prop :: excess risk bound} of the main paper.
\begin{proposition} \label{prop :: excess risk bound :: supp}
Let $0 < \eta < 1$, taking
\begin{equation*}
\lambda = \lambda_\nsamp^*( \nicefrac \eta 2):= 6 \kappa C_{\phi}^2 \frac {\log \left (\nicefrac 4 \eta \right )\sqrt \ndict}{\sqrt \nsamp},
\end{equation*}
with probability at least $1 - \eta$
\begin{equation*}
\cali R(\projope \circ \regridge) - \cali R(\projope \circ h_{\HKx}) \leq 27 \left ( \frac {B_0}{\sqrt \ndict} + B_1 \sqrt \ndict \right ) \frac {\log \left (\nicefrac 4 \eta \right )} {\sqrt \nsamp},
\end{equation*}
with $B_0: = (L + \sqrt \kappa C_{\phi} R)^2 $ and $B_1:= \kappa C_\dict^2 R^2$.
\end{proposition}

\subsubsection{Concentration results}

\begin{lemma} \label{lemma :: concentration result 1}
Let $ 0 < \eta < 1$, then with probability at least $1 - \eta$
\begin{align*}
\| \matop A_{\insamp, \projope}^{\adj} \outsamp - \matop T_{\insamp, \projope} h_{\HKx} \|_{\HKx} & \leq \delta_1(\nsamp, \eta),
\end{align*}
with $\delta_1$ defined as
\begin{equation} \label{eq :: def delta1}
\delta_1(\nsamp, \eta) := 6 (\sqrt{\kappa} C_{\phi} L + \kappa C_{\phi}^2 R) \frac {\log \left (\nicefrac 2 \eta \right )}{\sqrt{\nsamp}}.
\end{equation}
\end{lemma}

\begin{proof}
Let us define the function $\xi_1: \prodspace \longrightarrow \HKx$ as $\xi_1: (x, y) \longmapsto \Kx_{x, \projope}(y - \projope h_{\HKx}(x)) = \Kx_{x, \projope}(y - \Kx_{x, \projope}^{\adj} h_{\HKx}).$

Observe that 
$$
\frac 1 \nsamp \sum_{\sampiter=1}^{\nsamp} \xi_1(x_\sampiter, y_\sampiter) = \matop A_{\insamp, \projope}^{\adj} \outsamp - \matop T_{\insamp, \projope}h_{\HKx},
$$

and using Equation \eqref{eq :: condition hHK}, that 
$$ \bb E_{\inva, \outva \sim \funcprob}\left [\xi_1 (\inva, \outva) \right ] = \int_{\prodspace} \Kx_{x, \projope} y ~ \mathrm d \funcprob(x, y) - \left ( \int_{\prodspace} \Kx_{x, \projope} \Kx_{x, \projope}^{\adj} ~ \mathrm d \funcprob(x, y) \right ) h_{\cali H_{\HKx}} = \matop A_{\projope}^{\adj}Y - \matop T_{\projope}h_{\HKx} = 0. $$

The aim is now to apply the Bernstein inequality of Lemma \ref{lemma :: vv bernstein} to the random variable (RV) $\xi_1(\inva, \outva)$. First, we have almost surely

\begin{align}
\| \xi_1(\inva, \outva) \|_{\HKx} = \| \Kx_{\inva, \projope} (\outva - \projope h_{\HKx}(\inva)) \|_{\HKx} & \leq \| \Kx_{\inva, \projope} \|_{\cali L(\Lsqr, \HKx)} \| \outva - \projope h_{\HKx}(\inva)) \|_{\Lsqr} \nonumber \\
& \leq \sqrt{\kappa} C_{\phi} (\|\outva \|_{\Lsqr} + \| \Kx_{\inva, \projope}^\adj h\|_{\Lsqr}) \nonumber \\
& \leq \sqrt{\kappa} C_{\dict} (L + \sqrt{\kappa}  C_{\dict} R), \label{eqline :: upper bound xi1 2}
\end{align}

where we have used the inequality $\| \Kx_{x, \projope} \|_{\cali L(\Lsqr, \HKx)} = \| \Kx_{x, \projope}^\adj \|_{\cali L(\Lsqr, \HKx)} \leq  \sqrt \kappa C_\dict$ (immediate consequence of Equations \eqref{eq:riesz:supp} and \eqref{eq :: op norm Kx}), as well as Assumptions \ref{ass :: value bounded Y :: supp} and \ref{ass :: existence of risk minimizer :: supp}. 

Equation \eqref{eqline :: upper bound xi1 2} also implies

\begin{align*}
\bb E_{\inva, \outva \sim \funcprob}[\| \xi_1(\inva, \outva)\|_{\HKx}^2 ] \leq \kappa C_{\dict}  (L + \sqrt \kappa C_{\dict} R)^2.
\end{align*}

Hence we can apply Lemma \ref{lemma :: vv bernstein}, yielding that with probability at least $1 - \eta$,

\begin{align*}
\| \matop A_{\insamp, \projope}^{\adj} \outsamp - \matop T_{\insamp, \projope} h_{\HKx} \|_{\HKx} & \leq (\sqrt{\kappa} C_{\dict} L + \kappa C_{\dict}^2  R) \log \left (\nicefrac 2 \eta \right ) \left ( \frac 4 \nsamp + \frac 2 {\sqrt{\nsamp}} \right ) \\
& \leq 6 (\sqrt{\kappa} C_{\dict} L + \kappa C_{\dict}^2 R) \frac {\log \left (\nicefrac 2 \eta \right )} {\sqrt{\nsamp}}.
\end{align*}
\end{proof}

\begin{lemma} \label{lemma :: concentration result 2}
Let $ 0 < \eta < 1$, then with probability at least $1 - \eta$
\begin{align*}
\|\matop T_{\insamp, \projope} - \matop T_{\projope} \|_{\hsops(\HKx)} & \leq  \delta_2(\nsamp, \ndict, \eta),
\end{align*}
with $\delta_2$ defined as
\begin{equation} \label{eq :: def delta2}
\delta_2(\nsamp, \ndict, \eta) := 6 \kappa C_{\phi}^2 \frac {\log \left (\nicefrac 2 \eta \right ) \sqrt \ndict} {\sqrt{\nsamp}}.
\end{equation}
\end{lemma}

\begin{proof}
We introduce the $\xi_2: \prodspace \longrightarrow \hsops(\HKx)$ as $\xi_2: x, y \longmapsto \matop T_{x, \projope}.$

We have that
$$ \bb E_{\inva, \outva \sim \funcprob}[\xi_2(\inva, \outva)] = \int_{\cali X} \matop T_{x, \projope} ~ \mathrm d \funcprob_{\rmc X}(x) = \matop T_{\projope}.$$

And from Equation \eqref{eq :: bound HS norm Tx Phi}, we have almost surely

\begin{equation*}
\|\xi_2(\inva, \outva) \|_{\hsops(\HKx)} \leq \kappa C_{\dict}^2 \sqrt \ndict,
\end{equation*}

which implies as well

\begin{equation*}
\bb E_{\inva, \outva \sim \funcprob} [\| \xi_2(\inva, \outva) \|_{\hsops(\HKx)}^2] \leq \kappa^2 C_{\dict}^4  \ndict.
\end{equation*}

Since $\Kx$ is continuous and $\inspace$ is separable, $\HKx$ is separable. As a consequence the space $\hsops(\HKx)$ is also separable, we can thus apply Lemma \ref{lemma :: vv bernstein}, yielding that with probability at least $1 - \eta$, 

\begin{align*}
\|\matop T_{\insamp, \projope} -  \matop T_{\projope} \|_{\hsops(\HKx)} & \leq \kappa C_{\dict}^2 \sqrt \ndict \log \left (\nicefrac 4 \eta \right ) \left ( \frac 4 \nsamp + \frac 2 {\sqrt{\nsamp}} \right ) \\
& \leq 6 \kappa C_{\dict}^2 \sqrt \ndict \frac {\log \left (\nicefrac 2 \eta \right )} {\sqrt{\nsamp}}. 
\end{align*}
\end{proof}

\begin{lemma} \label{lemma :: concentration results}
Let $ 0 < \eta < 1$, then with probability at least $1 - \eta$ the two following inequalities hold:
\begin{align*}
\| \matop A_{\insamp, \projope}^{\adj} \outsamp - \matop T_{\insamp, \projope} h_{\HKx} \|_{\HKx} & \leq \delta_1(\nsamp, \nicefrac \eta 2)\\
\|\matop T_{\insamp, \projope} - \matop T_{\projope} \|_{\hsops(\HKx)} & \leq  \delta_2(\nsamp, \ndict, \nicefrac \eta 2),
\end{align*}
with $\delta_1$ and $\delta_2$ defined respectively in Equations \eqref{eq :: def delta1} and \eqref{eq :: def delta2}.
\end{lemma}

\begin{proof}
This is a union bound using Lemma \ref{lemma :: concentration result 1} and Lemma \ref{lemma :: concentration result 2}.
\end{proof}

\subsubsection{Proof} \label{subsubsection :: proof core}
We are now ready to prove Proposition \ref{prop :: excess risk bound :: supp}. We follow the same proof strategy as \citep{BaldassarreRosascoBarla12}. To that end, we first prove the following intermediate proposition of which Proposition \ref{prop :: excess risk bound :: supp} is a direct consequence. 

\begin{proposition} \label{prop :: intermediate bound with lambda}
Let $0 < \eta < 1$, provided $\lambda$ is taken such that 
\begin{equation} \label{eq :: condition on lambda :: supp}
\lambda \geq 6 \kappa C_{\phi}^2 \frac {\log \left (\nicefrac 4 \eta \right )\sqrt \ndict}{\sqrt \nsamp} = \delta_2(\nsamp, \ndict, \nicefrac \eta 2),
\end{equation}
we have with probability at least $1 - \eta$ that
\begin{equation} \label{eq :: intermediate bound with lambda}
\exprisk(\projope \circ \regridge) - \exprisk(\projope \circ h_{\HKx}) \leq \frac 9 2 \left (\frac {36(\sqrt \kappa C_{\phi} L + \kappa C_{\phi}^2 R)^2 \log \left(\nicefrac 4 \eta \right)^2}{\lambda \nsamp} + \lambda R^2 \right ).
\end{equation}
\end{proposition}

\begin{proof}

We introduce $h^{\lambda}$ as
\begin{equation} \label{eq :: definition h lambda}
h^{\lambda}:= (\matop T_{\insamp, \projope} + \lambda \matop I)^{-1} \matop T_{\insamp, \projope} h_{\HKx}.
\end{equation}

We consider the following decomposition of the risk using Equation \eqref{eq :: key rewriting of the risk}, 
\begin{align}
\exprisk(\projope \circ \regridge) - \exprisk(\projope \circ h_{\HKx}) & = \| \sqrt{\matop T_{\projope}}(\regridge - h_{\HKx}) \|_{\HKx}^2 \nonumber \\
& \leq 2 \| \sqrt{\matop T_{\projope}}(\regridge - h^{\lambda}) \|_{\HKx}^2 + 2 \| \sqrt{\matop T_{\projope}}(h^{\lambda} - h_{\HKx}) \|_{\HKx}^2 . \label{eq :: risk decomposition}
\end{align}

We first bound the term $ \| \sqrt{\matop T_{\projope}}(\regridge - h^{\lambda}) \|_{\HKx}$. Using the expression of $\regridge$ from Lemma \ref{lemma :: h z lambda close form}, we have that

\begin{align}
\sqrt{\matop T_{\projope}}(\regridge - h^{\lambda}) & = \sqrt{\matop T_{\insamp, \projope}}(\matop T_{\insamp, \projope} + \lambda \matop I)^{-1}(\matop A_{\insamp, \projope}^{\adj} \outsamp - \matop T_{\insamp, \projope} h_{\HKx}) \label{eq :: decomposition hzlambda - hlambda} \\ 
& + (\sqrt{\matop T_{\projope}} - \sqrt{\matop T_{\insamp, \projope}})(\matop T_{\insamp, \projope} + \lambda \matop I)^{-1}(\matop A_{\insamp, \projope}^{\adj} \outsamp - \matop T_{\insamp, \projope}h_{\HKx}).\nonumber
\end{align}

Since for all $a \geq 0$, $\frac {\sqrt a}{a + \lambda} \leq \frac 1 {2 \sqrt{\lambda}}$, since $\matop T_{\insamp, \projope}$ is positive, by spectral theorem we have that

\begin{equation} \label{eq :: spectral majoration strategy}
\| \sqrt{\matop T_{\insamp, \projope}}(\matop T_{\insamp, \projope} + \lambda \matop I)^{-1} \|_{\boundedops(\HKx)} \leq \max_{a \in \spectrum(\matop T_{\insamp, \projope})} \frac {\sqrt a}{a + \lambda} \leq  \max_{a \in \bb R_+} \frac {\sqrt a}{a + \lambda} \leq \frac 1 {2 \sqrt{\lambda}},
\end{equation}
where $\spectrum(\matop T_{\insamp, \projope})$ denotes the spectrum of $\matop T_{\insamp, \projope}$.

Similarily, since for all $a \geq 0$, $\frac 1 {a + \lambda} \leq \frac 1 \lambda$, we have as well  
\begin{equation*}
\|(\matop T_{\insamp, \projope} + \lambda \matop I)^{-1} \|_{\cali L(\HKx)} \leq \frac 1 \lambda.
\end{equation*}

Taking the norm in Equation \eqref{eq :: decomposition hzlambda - hlambda}, applying Minkowski's inequality and using Lemma \ref{lemma :: diff square root norm} as well as the last two displays yields

\begin{equation} \label{eq :: majoration operator norm 1st term}
\| \sqrt{\matop T_{\projope}}(\regridge - h^{\lambda}) \|_{\HKx} \leq  \| \matop A_{\insamp, \projope}^{\adj} \outsamp - \matop T_{\insamp, \projope} h_{\HKx} \|_{\HKx} \left ( \frac 1 {2 \sqrt \lambda} + \frac {\sqrt{\| \matop T_{\projope} - \matop T_{\insamp, \projope} \|}_{\cali L(\HKx)}}{\lambda} \right ).
\end{equation}

Now dealing with the term on the right-hand side in Equation \eqref{eq :: risk decomposition}, using the definition of $h^\lambda$ in Equation \eqref{eq :: definition h lambda}, we have that

\begin{align}
\sqrt{\matop T_{\projope}}(h_{\HKx} - h^{\lambda}) & = \sqrt{\matop T_{\projope}} (\matop I - (\matop T_{\insamp, \projope} + \lambda \matop I)^{-1}\matop T_{\insamp , \projope}) h_{\HKx} \nonumber \\
& = (\sqrt{\matop T_{\projope}} - \sqrt{\matop T_{\insamp, \projope}})(\matop I - (\matop T_{\insamp, \projope} + \lambda \matop I)^{-1} \matop T_{\insamp, \projope}) h_{\HKx} \label{eq :: rewriting second term} \\ 
& + \sqrt{\matop T_{\insamp, \projope}}(\matop I - (\matop T_{\insamp, \projope} + \lambda \matop I)^{-1} \matop T_{\insamp, \projope}) h_{\HKx} \nonumber .
\end{align}

Since for all $a \geq 0$, $ \sqrt a \left( 1 - \frac{a}{a + \lambda} \right ) = \frac {\sqrt a \lambda} {a + \lambda} \leq \frac 1 2 \sqrt \lambda$, using the same arguments as in Equation  \eqref{eq :: spectral majoration strategy} yields

\begin{equation*}
\| \sqrt{\matop T_{\insamp, \projope}}(\matop I - (\matop T_{\insamp, \projope} + \lambda \matop I)^{-1}\matop T_{\insamp, \projope}) \|_{\cali L(\HKx)} \leq \frac 1 2 \sqrt \lambda.
\end{equation*}

Moreover, since for all $a \geq 0$, $1 - \frac{a}{a + \lambda} = \frac \lambda {a + \lambda} \leq 1$, similarly we have that

\begin{align*}
\| \matop I - (\matop T_{\insamp, \projope} + \lambda \matop I)^{-1}\matop T_{\insamp, \projope} \|_{\cali L(\HKx)} \leq 1 .
\end{align*}

Thus, taking the norm in Equation \eqref{eq :: rewriting second term}, using Minkowski's inequality, Lemma \ref{lemma :: diff square root norm} and Equation \eqref{eq :: definition of R} yields 

\begin{equation} \label{eq :: majoration operator norm 2nd term}
\| \sqrt{\matop T_{\projope}}(h_{\HKx} - h^{\lambda}) \|_{\HKx} \leq R \sqrt{ \| \matop T_{\projope} - \matop T_{\insamp, \projope} \|}_{\cali L(\HKx)} + \frac R 2 \sqrt \lambda.
\end{equation}

Combining Equations \eqref{eq :: majoration operator norm 1st term} and \eqref{eq :: majoration operator norm 2nd term} with Lemma \ref{lemma :: concentration results}, for $0 < \eta < 1$, we have with probability at least $1 - \eta$

\begin{align*}
\| \sqrt{\matop T_{\projope}}(\regridge - h^{\lambda}) \|_{\HKx} & \leq \delta_1(\nsamp, \nicefrac \eta 2) \left ( \frac 1 {2 \sqrt \lambda} + \frac{\sqrt{\delta_2(\nsamp, \ndict, \nicefrac \eta 2)}}{\lambda} \right ) \\
\| \sqrt{\matop T_{\projope}}(h_{\HKx} - h^{\lambda}) \|_{\HKx} & \leq R \sqrt{\delta_2(\nsamp, \ndict, \nicefrac \eta 2)} + \frac R 2 \sqrt \lambda .
\end{align*}

Using the condition on $\lambda$ given by Equation \eqref{eq :: condition on lambda :: supp}, still with probability at least $1 - \eta$, we have

\begin{align}
\| \sqrt{\matop T_{\projope}}(\regridge - h^{\lambda}) \|_{\HKx} & \leq \frac{3}{2 \sqrt \lambda} \delta_1(\nsamp, \nicefrac \eta 2), \label{eq :: majoration operator norm 1st term final} \\
\| \sqrt{\matop T_{\projope}}(h_{\HKx} - h^{\lambda}) \|_{\HKx} & \leq \frac{3R}{2} \sqrt \lambda . \label{eq :: majoration operator norm 2nd term final}
\end{align}

Combining Equations \eqref{eq :: majoration operator norm 1st term final} and \eqref{eq :: majoration operator norm 2nd term final} into Equation \eqref{eq :: risk decomposition} yields that with probability at least $1 - \eta$, 

\begin{equation*}
\exprisk(\projope \circ \regridge) - \exprisk(\projope \circ h_{\HKx}) \leq \frac 9 2 \left ( \frac {\delta_1(\nsamp,\nicefrac \eta 2)^2}{\lambda} + R^2 \lambda \right ) .
\end{equation*}
\end{proof}

In Proposition \ref{prop :: intermediate bound with lambda}, we have a compromise in $\lambda$ in the two terms. Taking $\lambda = \cali O (\sqrt \nsamp)$ yields the best compromise. So as to satisfy the condition from Equation \eqref{eq :: condition on lambda :: supp}, we take $\lambda = 6 \kappa C_{\phi}^2 \frac {\log \left (\nicefrac 4 \eta \right )\sqrt \ndict}{\sqrt \nsamp}$, which after simplifications in the constants yields Proposition \ref{prop :: excess risk bound :: supp}. 

\subsection{Proof of Proposition \ref{prop :: excess risk bound partial} from the main paper} \label{subsec :: proof excess risk bound partial}

We recall the additional assumption made on the dictionary, as well as the proposition itself which corresponds to Proposition \ref{prop :: excess risk bound partial} from the main paper. 

\begin{assumption} \label{ass :: value bounded dict :: supp}
There exists $M(\ndict) \geq 0$ such that for all $\theta \in \outfuncdom$ and for all $\dictiter \in [\ndict]$, $|\dict_\dictiter(\theta) | \leq M(\ndict)$.
\end{assumption}
\begin{remark}
The dependence in $\ndict$ is specific to the family to which $\dict$ belongs. For instance for wavelets, we have $M(\ndict) = 2^{\nicefrac {r(\outfuncdom, \ndict)} 2} \max_{\theta \in \outfuncdom} |\psi(\theta)|$ with $\psi$ the mother wavelet and $r(\outfuncdom, \ndict) \in \bb N$ the number of dilatations that are included in $\dict$, whereas for a Fourier dictionary we have $M(\ndict)=1$. 
\end{remark}

\begin{proposition} \label{prop :: excess risk bound partial :: supp}
Let $0 < \eta < 1$, taking 
\begin{equation*}
\lambda = \lambda_\nsamp^*(\nicefrac \eta 3) := 6 \kappa C_{\phi}^2 \frac {\log \left (\nicefrac 6 \eta \right )\sqrt \ndict}{\sqrt \nsamp},
\end{equation*}
with probability at least $1 - \eta$, 
\begin{align*}
\cali R(\projope \circ \regridgepartial) - \cali R(\projope \circ h_{\HKx}) \leq \left (\frac{B_2(\ndict)\sqrt \nsamp}{\nobsf^2} + \frac {B_3(\ndict)} {\nobsf^{\nicefrac 32}} + \frac {9C(\ndict)^2} {2 \sqrt \nsamp \nobsf} + \frac {B_4(\ndict)} {\sqrt \nsamp} \right ) \log \left ( \nicefrac 6 \eta \right ),
\end{align*}
with  $ C(\ndict):=\frac {LM(\ndict)}{C_{\dict}}$, $B_2(\ndict):=18 \sqrt \ndict \left (C(\ndict) + \frac R {\sqrt \ndict} \right )^2$, $B_3(\ndict): = B_2(d) - 18 \frac{R^2}{\sqrt \ndict}$, $B_4(\ndict):= \frac {81} 2 \left (\frac {B_0} {\sqrt \ndict} + B_1 \sqrt \ndict \right )$ and $B_0$ and $B_1$ are defined as in Proposition \ref{prop :: excess risk bound :: supp}.
\end{proposition}

\subsubsection{Approximated solution for partially observed functions}

We recall the notion of partially observed functional output sample:
\begin{equation*}
 \obsfun{\prodsamp} :=(x_i, (\obslocsvec_\sampiter, \obsfun y_\sampiter) )_{\sampiter=1}^\nsamp,
 \end{equation*} 
 where for all $\sampiter \in [\nsamp]$, $\obslocsvec_\sampiter \in \outfuncdom^{\nobsf_\sampiter}$, $\obsfun{y}_\sampiter \in \bb R^{\nobsf_\sampiter}$ with $\nobsf_\sampiter \in \bb N^*$  the number of observations available for the $\sampiter$-th function, and for all $\obsfiter \in [\nobsf_\sampiter]$, $\obslocs_{\sampiter \obsfiter} \in \outfuncdom$ and $\obsfun {y}_{\sampiter \obsfiter} \in \bb R$. We remind the reader as well that to simplify, we have supposed in Section \ref{sec :: theory} from the main paper that for all $\sampiter \in [\nsamp]$, $\nobsf_\sampiter = \nobsf$. 

We introduce the notation $\obsfun{\mathbf y}:=(\obsfun{y}_\sampiter)_{\sampiter=1}^\nsamp$ and highlight that since there is no added noise, we have for all $\sampiter \in [\nsamp]$

$$ \obsfun{y}_\sampiter = (y_\sampiter(\theta_{\sampiter \obsfiter}))_{\obsfiter=1}^\nobsf.$$ 

We recall that $\locsprob$ is the uniform probability measure over $\outfuncdom$ which governs the draws of the locations of sampling.

For $\sampiter \in [\nsamp]$, we define $\obsfun{\projope}_\sampiter \in \bb R^{\nobsf \times \ndict}$ the approximation of $\projope$ using the locations $\obslocsvec_\sampiter $ as 

\begin{equation*}
\obsfun{\projope}_\sampiter : = (\dict_1(\obslocsvec_\sampiter),.., \dict_\ndict(\obslocsvec_\sampiter)),
\end{equation*}

where for $\sampiter \in [\nsamp]$ and for $\dictiter \in [\ndict]$, $\dict_\dictiter(\obslocsvec_\sampiter) = (\dict_\dictiter(\theta_{\sampiter \obsfiter}))_{\obsfiter=1}^\nobsf \in \bb R^\nobsf$.

Let us recall that the solution when the output functions are fully observed (Equation \eqref{eq :: h z lambda close form}) reads:

\begin{equation*}
\regridge = (\matop T_{\insamp, \projope} + \lambda \matop I)^{-1} \matop A_{\insamp, \projope}^{\adj} \outsamp,
\end{equation*}

with 
\begin{equation*}
\matop A_{\insamp, \projope}^{\adj} \mathbf w = \frac 1 \nsamp \sum_{\sampiter=1}^{\nsamp} \Kx_{x_i} \projope^{\adj} w_i ~~ \text{for} ~~ \mathbf w \in \Lsqr^\nsamp.
\end{equation*}

We now consider of partially observed output functions with observed locations $ (\obslocsvec_\sampiter)_{\sampiter=1}^\nsamp $ and define an estimator in this setting. We first define

$$ \matop A_{\insamp, \obsfun{\projope}}^{\adj} \obsfun{\mathbf w} =  \frac 1 \nsamp \sum_{\sampiter=1}^{\nsamp} \Kx_{x_i} \frac{\obsfun{\projope}_\sampiter^{\adj}} \nobsf \obsfun{w}_i ~~ \text{with} ~~ \obsfun{\mathbf w} \in \bb R^{\nsamp \times \nobsf}, $$ 
%
%where we have introduced $\obsfun{\projope}:=(\obsfun{\projope}_{\sampiter})_{\sampiter=1}^\nsamp$. 

The solution we consider when dealing with partially observed functions is then the following 
\begin{equation*}
\regridgepartial:= (\matop T_{\insamp, \projope} + \lambda \matop I)^{-1}  \matop A_{\insamp, \obsfun{\projope}}^{\adj} \obsfun{\mathbf y}.
\end{equation*}
It is another equivalent expression for the plug-in ridge estimator from Definition \ref{def :: plug-in ridge estimator} from the main paper.

\subsubsection{Concentration results}

\begin{lemma} \label{lemma :: concentration result 3}
Let $ 0 < \eta < 1$, then with probability at least $1 - \eta$
\begin{equation*}
\| \matop A_{\insamp, \obsfun{\projope}}^{\adj} \obsfun{\mathbf y} - \matop A_{\insamp, \projope}^{\adj} \mathbf y \|_{\HKx}  \leq \delta_3(\nsamp, \nobsf, \ndict, \eta),
\end{equation*}
with $\delta_3$ defined as 
\begin{equation} \label{eq :: def delta3}
 \delta_3(\nsamp, \nobsf, \ndict, \eta) := \left( \frac{4(L \sqrt \kappa \sqrt \ndict M(\ndict) + \sqrt \kappa C_\dict R)} \nobsf + \frac {2 L \sqrt \kappa \sqrt \ndict M(\ndict)} {\sqrt \nsamp \sqrt \nobsf} \right ) \log\left( \nicefrac 2 \eta \right ).
\end{equation}
\end{lemma}

\begin{proof}
Let us define the function $\xi_3: \inspace \times \Lsqr \times \outfuncdom \longrightarrow \HKx$ as $\xi_3 : (x, y, \theta) \longmapsto y (\theta) \Kx_x \dict(\theta) - \Kx_x \projope ^\adj y$

The proof relies on the fact that
\begin{align*}
\frac 1 \nsamp \sum_{\sampiter=1}^\nsamp \frac 1 \nobsf \sum_{\obsfiter=1}^\nobsf \xi_3(x_\sampiter, y_\sampiter, \theta_{\sampiter \obsfiter}) & = \frac 1 \nsamp \sum_{\sampiter=1}^\nsamp \Kx_{x_\sampiter} \frac{\obsfun{\projope}^\adj_\sampiter} \nobsf \obsfun{y}_\sampiter - \Kx_{x_\sampiter} \projope^\adj y_\sampiter \\
& = \matop A_{\insamp, \obsfun{\projope}}^{\adj} \obsfun{\mathbf y} - \matop A_{\insamp, \projope}^{\adj} \mathbf y.
\end{align*}

Let $(\inva_\sampiter, \outva_\sampiter)_{\sampiter=1}^\nsamp$ be $\nsamp$ i.i.d. RVs distributed according to the distribution $\funcprob$. Let $(\vartheta_{\sampiter \obsfiter})_{\sampiter=1, \obsfiter=1}^{\nsamp, \nobsf}$ be $\nsamp \nobsf$ i.i.d. RVs distributed according to the distribution $\locsprob$. For all $\sampiter \in [\nsamp]$ and for all $\obsfiter \in [\nobsf]$ we then define the RVs $\rmc W_{\sampiter \obsfiter}$ as 

\begin{align}
\rmc W_{\sampiter \obsfiter}: & = \xi_3(\inva_\sampiter, \outva_\sampiter, \vartheta_{\sampiter \obsfiter}) \nonumber \\
& = \outva_\sampiter(\vartheta_{\sampiter \obsfiter}) \Kx_{\inva_\sampiter} \dict(\vartheta_{\sampiter \obsfiter}) - \Kx_{\inva_\sampiter} \projope^\adj \outva_\sampiter \nonumber \\
& =  \outva_\sampiter(\vartheta_{\sampiter \obsfiter}) \Kx_{\inva_\sampiter} \dict(\vartheta_{\sampiter \obsfiter}) - \bb E[ \outva_\sampiter(\vartheta) \Kx_{\inva_\sampiter} \dict(\vartheta)| \inva_\sampiter, \outva_\sampiter], \label{eqline :: conditionally centered}
\end{align}

where the last line holds because $\locsprob$ is the uniform distribution and because we have assumed that $|\outfuncdom| = \int_{\outfuncdom} 1 \mathrm{d} \theta = 1$ (see the notation and context paragraph at the end of Section \ref{sec :: introduction} from the main paper).

We denote by $\bb P[. |\prodsamp]$ the probability conditional on the realization of the sample $\prodsamp$, thus $$\bb P[. |\prodsamp] = \bb P[. |\inva_\sampiter=x_\sampiter, \outva_\sampiter = y_\sampiter,~~ \sampiter \in [\nsamp]]$$

Then, Equation \eqref{eqline :: conditionally centered} implies that $\bb E[\rmc W_{\sampiter \obsfiter} | \prodsamp] = 0$. 

We define as well for all $\obsfiter \in [\nobsf]$, $\overline{\rmc W}_\obsfiter := \frac 1 \nsamp \sum_{\sampiter=1}^\nsamp \rmc W_{\sampiter \obsfiter}$.

We have almost surely that

\begin{align*}
\| \overline{\rmc W}_\obsfiter \|_{\HKx} \leq \frac 1 n \sum_{\sampiter=1}^\nsamp \|\rmc W_{\sampiter \obsfiter}\|_{\HKx} & \leq \frac 1 \nsamp \sum_{\sampiter=1}^\nsamp \left (| \outva_\sampiter(\vartheta_{\sampiter \obsfiter}) | \| \Kx_{\inva_\sampiter} \dict(\vartheta_{\sampiter \obsfiter}) \|_{\HKx} + \| \Kx_{\inva_\sampiter} \projope^\adj \outva_\sampiter \|_{\HKx} \right ) \\
& \leq L \sqrt{\kappa} \sqrt {\ndict} M(\ndict) + \sqrt {\kappa} C_{\dict} R.
\end{align*}

We have used Assumptions \ref{ass :: value bounded Y :: supp} and \ref{ass :: value bounded dict :: supp} as well as Equation \eqref{eq:riesz adj:supp}.

Since for all $\obsfiter \in [\nobsf]$, the RVs $(\rmc W_{\sampiter \obsfiter})_{\sampiter=1}^\nsamp$ are independent conditionally on $\prodsamp$, we have that 

\begin{equation}
\bb E[\| \overline{\rmc W}_\obsfiter \|_{\HKx}^2 | \prodsamp] = \frac 1 {\nsamp^2} \sum_{\sampiter=1}^\nsamp \bb E[ \| \rmc W_{\sampiter \obsfiter} \|_{\HKx}^2 | \prodsamp]. \label{eqline:sum norm inde:supp}
\end{equation}

Using the fact that $ \bb E[\outva_\sampiter(\vartheta_{\sampiter \obsfiter}) \Kx_{\inva_\sampiter} \dict(\vartheta_{\sampiter \obsfiter})|\prodsamp ] = \Kx_{x_\sampiter} \projope^\adj y_\sampiter $, the identity $\bb E[\| \rmc U - \bb E[\rmc U] \|_{\HKx}^2] = \bb E[\| \rmc U \|_{\HKx}^2]$ gives us

\begin{equation} \label{eqline:expe norm center:supp}
\bb E[ \| \rmc W_{\sampiter \obsfiter} \|_{\HKx}^2 | \prodsamp] = \bb E[\| \outva_\sampiter(\vartheta_{\sampiter \obsfiter}) \Kx_{\inva_\sampiter} \dict(\vartheta_{\sampiter \obsfiter}) \|_{\HKx}^2 | \prodsamp ].
\end{equation}

Then using Equation \eqref{eqline:expe norm center:supp} into Equation \eqref{eqline:sum norm inde:supp} along with Assumptions \ref{ass :: value bounded Y :: supp} and \ref{ass :: value bounded dict :: supp} yields
\begin{align*}
\bb E[\| \overline{\rmc W}_\obsfiter \|_{\HKx}^2 | \prodsamp] \leq \frac 1 \nsamp L^2 \kappa \ndict M(\ndict)^2.
\end{align*}
%
%Conditionally on $\prodsamp$, the RVs $ (\overline{\rmc W}_\obsfiter)_{

We can then apply Lemma \ref{lemma :: vv bernstein inde} to obtain that 

\begin{equation*}
\bb P \left [ \left \| \frac 1 \nobsf \sum_{\obsfiter=1}^\nobsf \overline{\rmc W}_\obsfiter \right \|_{\HKx} \leq \left( \frac{4(L \sqrt \kappa \sqrt \ndict M(\ndict) + \sqrt \kappa C_\dict R)} \nobsf + \frac {2 L \sqrt \kappa \sqrt \ndict M(\ndict)} {\sqrt \nsamp \sqrt \nobsf} \right ) \log\left( \nicefrac 2 \eta \right )\middle | \prodsamp \right ] \geq 1 - \eta.
\end{equation*}

Multiplying the above inequality by $\bb P[\prodsamp]$ and integrating over $\prodsamp \in \prodspace^\nsamp$, yields that 

\begin{equation*}
\bb P \left [ \left \| \matop A_{\insamp, \obsfun{\projope}}^{\adj} \obsfun{\mathbf y} - \matop A_{\insamp, \projope}^{\adj} \mathbf y \right \|_{\HKx} \leq \left( \frac{4(L \sqrt \kappa \sqrt \ndict M(\ndict) + \sqrt \kappa C_\dict R)} \nobsf + \frac {2 L \sqrt \kappa \sqrt \ndict M(\ndict)} {\sqrt \nsamp \sqrt \nobsf} \right ) \log\left( \nicefrac 2 \eta \right ) \right ] \geq 1 - \eta.
\end{equation*}
\end{proof}

\begin{lemma} \label{lemma :: concentration results partially observed}
Let $ 0 < \eta < 1$, then with probability at least $1 - \eta$ the three following inequalities hold:

\begin{align}
\| \matop A_{\insamp, \projope}^{\adj} \outsamp - \matop T_{\insamp, \projope} h_{\HKx} \|_{\HKx} & \leq \delta_1(\nsamp, \nicefrac \eta 3)\\
\|\matop T_{\insamp, \projope} - \matop T_{\projope} \|_{\hsops(\HKx)} & \leq  \delta_2(\nsamp, \ndict, \nicefrac \eta 3) \\
\| \matop A_{\insamp, \obsfun{\projope}}^{\adj} \obsfun{\mathbf y} - \matop A_{\insamp, \projope}^{\adj} \mathbf y \|_{\HKx} & \leq \delta_3(\nsamp, \nobsf, \ndict, \nicefrac \eta 3),
\end{align}

with $\delta_1$, $\delta_2$ and $\delta_3$ respectively defined as in Equations \eqref{eq :: def delta1}, \eqref{eq :: def delta2} and \eqref{eq :: def delta3}. 
\end{lemma}

\begin{proof}
This Lemma is an union bound using Lemma \ref{lemma :: concentration result 1}, Lemma \ref{lemma :: concentration result 2} and Lemma \ref{lemma :: concentration result 3}.
\end{proof}

\subsubsection{Proof}

We are now ready to prove Proposition \ref{prop :: excess risk bound partial :: supp}. To do so we prove the following intermediate result of which Proposition \ref{prop :: excess risk bound partial :: supp} is a direct consequence.

\begin{proposition} \label{prop :: intermediate bound partial}
Let $0 < \eta < 1$, provided $\lambda$ is taken such that 
\begin{equation} \label{eq :: condition on lambda 2}
\lambda \geq 6 \kappa C_{\phi}^2 \frac {\log \left (\nicefrac 6 \eta \right )\sqrt \ndict}{\sqrt \nsamp} = \delta_2(\nsamp, \ndict, \nicefrac \eta 3),
\end{equation}
we have with probability at least $1 - \eta$ that
\begin{equation} \label{eq:intermediate bound sparse:supp}
\cali R(\projope \circ \regridgepartial) - \cali R(\projope \circ h_{\HKx}) \leq \frac{27} 4 \left ( \left ( \frac {A_0(\ndict)^2}{\lambda \nobsf^2} + \frac {2 A_0(\ndict) A_1(\ndict)}{\lambda \sqrt{\nsamp} \nobsf^{\nicefrac 32}} + \frac{A_1(\ndict)^2}{\lambda \nsamp \nobsf} + \frac{A_2^2}{\lambda \nsamp} \right ) \log \left (\nicefrac 6 \eta \right )^2 + \lambda R^2 \right ),
\end{equation}
with
\begin{align*}
 A_0(\ndict)& : = 4(L \sqrt \kappa \sqrt \ndict M(\ndict) + \sqrt \kappa C_\dict R) \\
 A_1(\ndict)& :=2 L \sqrt \kappa \sqrt \ndict M(\ndict) \\
 A_2 & := 6 (\sqrt \kappa C_\dict L + \kappa C_\dict^2 R).
 \end{align*}
\end{proposition}

\begin{proof}
Taking $h^{\lambda}$ as in Equation \eqref{eq :: definition h lambda}, we consider the following decomposition of the risk using Equation \eqref{eq :: key rewriting of the risk}

\begin{align}
\cali R(\projope \circ \regridgepartial) - \cali R(\projope \circ h_{\HKx}) & = \| \sqrt{T_{\projope}}(\regridgepartial - h_{\HKx}) \|_{\HKx}^2 \nonumber \\
& \leq 3 \| \sqrt{T_{\projope}} (\regridgepartial - \regridge )\|_{\HKx}^2 + 3 \| \sqrt{T_{\projope}}(\regridge - h^{\lambda}) \|_{\HKx}^2 + 3 \| \sqrt{T_{\projope}}(h^{\lambda} - h_{\HKx}) \|_{\HKx}^2. \label{eq:risk decomposition sparse:supp}
\end{align}

We focus on the term on the left as we have already controlled the two others in the proof of Lemma \ref{prop :: intermediate bound with lambda} . Using the same strategy as for proving Equation \eqref{eq :: majoration operator norm 1st term}, we get that 

\begin{equation} \label{eq:maj sparse term:supp}
\| \sqrt{T_{\projope}}(\regridgepartial - \regridge ) \|_{\HKx} \leq  \| \matop A_{\insamp, \obsfun{\projope}}^{\adj} \obsfun{\mathbf y} - \matop A_{\insamp, \projope}^{\adj} \mathbf y \|_{\HKx} \left ( \frac 1 {2 \sqrt \lambda} + \frac {\sqrt{\| \matop T_{\projope} - \matop T_{\insamp, \projope} \|}_{\cali L(\HKx)}}{\lambda} \right ).
\end{equation}

Combining Equations \eqref{eq :: majoration operator norm 1st term} , \eqref{eq :: majoration operator norm 2nd term} and \eqref{eq:maj sparse term:supp} with Lemma \ref{lemma :: concentration results partially observed}, for $0 < \eta < 1$, the three following inequalities are verified with probability at least $1 - \eta$

\begin{align*}
\| \sqrt{T_{\projope}}(\regridgepartial - \regridge ) \|_{\HKx} & \leq \delta_3(\nsamp, \nobsf, \ndict, \nicefrac \eta 3) \left ( \frac 1 {2 \sqrt \lambda} + \frac{\sqrt{\delta_2(\nsamp, \ndict, \nicefrac \eta 3)}}{\lambda} \right ) \\
\| \sqrt{\matop T_{\projope}}(\regridge - h^{\lambda}) \|_{\HKx} & \leq \delta_1(\nsamp, \nicefrac \eta 3) \left ( \frac 1 {2 \sqrt \lambda} + \frac{\sqrt{\delta_2(\nsamp, \ndict, \nicefrac \eta 3)}}{\lambda} \right ) \\
\| \sqrt{\matop T_{\projope}}(h_{\HKx} - h^{\lambda}) \|_{\HKx} & \leq R \sqrt{\delta_2(\nsamp, \ndict, \nicefrac \eta 3)} + \frac R 2 \sqrt \lambda .
\end{align*}

Using the condition on $\lambda$ given by Equation \eqref{eq :: condition on lambda 2}, still with probability at least $1 - \eta$, we have

\begin{align}
\| \sqrt{T_{\projope}}( \regridgepartial - \regridge) \|_{\HKx} & \leq  \frac{3}{2 \sqrt \lambda} \delta_3(\nsamp, \nobsf, \ndict, \nicefrac \eta 3) \label{eqline:majoration sparse:supp}\\
\| \sqrt{\matop T_{\projope}}(\regridge - h^{\lambda}) \|_{\HKx} & \leq \frac{3}{2 \sqrt \lambda} \delta_1(\nsamp, \nicefrac \eta 3) \label{eqline:majoration esti:supp} \\
\| \sqrt{\matop T_{\projope}}(h_{\HKx} - h^{\lambda}) \|_{\HKx} & \leq \frac{3R}{2} \sqrt \lambda. \label{eqline:majoration approx:supp}
\end{align}

Combining Equation \eqref{eqline:majoration sparse:supp}, \eqref{eqline:majoration esti:supp} and \eqref{eqline:majoration approx:supp} into Equation \eqref{eq:risk decomposition sparse:supp} yields that with probability at least $1 - \eta$, 

\begin{equation*}
\cali R(\projope \circ \regridgepartial) - \cali R(\projope \circ h_{\HKx}) \leq \frac {27} 4 \left ( \frac{\delta_3(\nsamp, \nobsf, \ndict, \nicefrac \eta 3)^2}{\lambda} + \frac {\delta_1(\nsamp, \nicefrac \eta 3)^2}{\lambda} + R^2 \lambda \right ) .
\end{equation*}

In Proposition \ref{prop :: intermediate bound partial}, we have a compromise in $\lambda$. Taking $\lambda = \cali O (\sqrt \nsamp)$ yields the best one. So as to satisfy the condition on $\lambda$ (Equation \eqref{eq :: condition on lambda 2}), we take $\lambda = 6 \kappa C_{\phi}^2 \frac {\log \left (\nicefrac 6 \eta \right )\sqrt \ndict}{\sqrt \nsamp}$. After simplifications in the constants we get Proposition \ref{prop :: excess risk bound partial :: supp}.
\end{proof}

\section{ADDITIONAL PL AND KPL RESULTS} \label{sec :: additional}
\subsection{Gradient-based optimization for partially observed functions in the general case} \label{subsec :: gradient projection general}

An interesting property of PL (not only when considering vv-RKHSs as hypothesis class as in Section \ref{sec :: KPL} of the main paper) is that the gradient of the data-fitting term can be estimated straightforwardly from partially observed functions. Let us consider the general PL problem (Problem (\ref{prob :: empirical risk approx}) from the main paper):

\begin{equation} %\label{prob :: empirical risk approx :: supp}
\min_{h \in \cali H} \emprisk(\projope \circ h, \prodsamp)  + \lambda \generalregu_{\cali H}(h),
\end{equation}

We recall the definition of a partially observed functional output sample (Equation (\ref{eq :: partial sample}) from the main paper):

\begin{equation*}
 \obsfun{\prodsamp} :=(x_i, (\obslocsvec_\sampiter, \obsfun y_\sampiter) )_{\sampiter=1}^\nsamp,
 \end{equation*} 

Let us now compute the gradient for the data-fitting term considering a parametric hypothesis class of the form $\{ h_{\mathbf{w}}, \mathbf{w} \in \bb R^p \}$; such that for $x \in \inspace $, $\mathbf w \longmapsto h_{\mathbf{w}}$ is differentiable. The gradient is given by  
$$\sum_{\sampiter=1}^\nsamp (\nabla h_{\mathbf w}(x_\sampiter))^\trans \projope^\adj \nabla \geneloss_{y_\sampiter}( \projope h_{\mathbf w}(x_\sampiter)),$$

with $\nabla h_{\mathbf w}(x_\sampiter) \in \bb R^{\ndict \times p}$ the Jacobian of $h_{\mathbf w}(x)$ and $\nabla \geneloss(y_\sampiter, \projope h_{\mathbf w}(x_\sampiter)) \in \Lsqr$ the gradient of the loss $\geneloss$ with respect to its second argument. For integral losses (Equation (\ref{eq :: int loss}) from the main paper), this gradient is $\nabla \geneloss(y_\sampiter, .): v \longmapsto (\theta \longmapsto \groundloss(y_\sampiter(\theta), v(\theta)))$. We can estimate the vectors $\projope^\adj \nabla \geneloss(y_\sampiter, \projope h_{\mathbf w}(x_\sampiter))$ from the partially observed functions $((\obslocsvec_\sampiter, \obsfun{y}_\sampiter))_{\sampiter=1}^\nsamp$:

\begin{equation*}
\frac 1 {\nobsf_\sampiter} \sum_{\obsfiter=1}^{\nobsf_\sampiter} l \left (y_\sampiter(\obslocs_{\sampiter \obsfiter}), \dict(\obslocs_{\sampiter \obsfiter})^\trans h_{\mathbf{w}}(x_\sampiter) \right ) \dict (\obslocs_{\sampiter \obsfiter}), 
\end{equation*}

Then replacing $h_{\mathbf w}$ by the regressor corresponding to the vv-RKHS hypothesis class with separable kernel: $x \longmapsto \matop B k(x)$, we obtain Equation (\ref{eq :: estimated gradient KPL}) from the main paper.

Using those estimated gradient is unsurprisingly equivalent to minimizing the problem based on a formulation of an empirical risk using the partially observed functional output sample $\obsfun{\prodsamp}$:

\begin{equation} \label{prob :: empirical risk partial}
\min_{\mathbf w \in \bb R^p} \frac 1 \nsamp \sum_{\sampiter=1}^\nsamp \frac 1 {\nobsf_\sampiter} \sum_{\obsfiter=1}^{\nobsf_\sampiter} \groundloss \left (y_\sampiter(\obslocs_{\sampiter \obsfiter}), \dict(\obslocs_{\sampiter \obsfiter})^\trans h_{\mathbf{w}}(x_\sampiter) \right ).
\end{equation}

\subsection{Plug-in ridge estimator and iterative optimization solution for the square loss.} \label{subsec :: comparison iterative plug-in}

For $\sampiter \in [\nsamp]$, we recall the definition of $\obsfun{\projope}_\sampiter \in \bb R^{\nobsf_\sampiter \times \ndict}$ the discrete approximation of $\projope$ using the locations $\obslocsvec_\sampiter $:

\begin{equation*}
\obsfun{\projope}_\sampiter : = (\dict_1(\obslocsvec_\sampiter),.., \dict_\ndict(\obslocsvec_\sampiter)),
\end{equation*}

Then in the case of the square loss, Problem (\ref{prob :: empirical risk dict vvrkhs}) from the main paper can be rewritten as 

\begin{equation}
\min_{h \in \HKx} \frac 1 \nsamp \sum_{\sampiter=1}^\nsamp \bigg \| \frac{\obsfun{y}_\sampiter}{\sqrt{\nobsf_\sampiter}} - \frac{\obsfun{\projope}_\sampiter}{\sqrt{\nobsf_\sampiter}} h(x_\sampiter) \bigg \|_{\bb R^{\nobsf_\sampiter}}^2 + \lambda \|h\|_{\HKx}^2
\end{equation} 

Let us define $\obsfun{\projope} \in \cali L \left (\bb R^{\ndict \nsamp}, \bb R^{\overline \nobsf} \right ) $ as 
$ \obsfun{\projope} : (u_\sampiter)_{\sampiter=1}^\nsamp \longmapsto \mattovec \left( \left (\frac{\obsfun{\projope}_\sampiter}{\sqrt{\nobsf_\sampiter}} u_\sampiter \right )_{\sampiter=1}^\nsamp \right )$ where we have set $\overline \nobsf := \sum_{\sampiter=1}^\nsamp \nobsf_\sampiter$. 

Then using Proposition \ref{prop :: representer primal} from the main paper, we can rewrite Problem \eqref{prob :: empirical risk partial} as 

\begin{equation*}
\min_{\alpha \in \bb R^{\ndict \times \nsamp}} \frac 1 \nsamp \| \mattovec (\obsfun{\outsamp}) - \obsfun{\projope} \matrb \Kx \mattovec(\alpha) \|_{\bb R^{\overline \nobsf}}^2 + \lambda \langle \mattovec(\alpha), \matrb \Kx \mattovec(\alpha) \rangle_{\bb R^{\ndict \nsamp}}.
\end{equation*}

Carrying the same steps as in the proof of Proposition \ref{prop :: closed form ridge supp} yields that $\alpha^*$ is such that 

\begin{equation}
\mattovec(\alpha^*) = ( (\obsfun{\projope}^\adj \obsfun{\projope}) \matrb \Kx + \nsamp \lambda \matrb I )^{-1}\obsfun{\projope}^\adj \mattovec(\obsfun{\outsamp}).
\end{equation}

We remark that $\obsfun{\projope}^\adj \mattovec(\obsfun{\outsamp}) \in \bb R^{\ndict \nsamp}$ corresponds to the estimations of the scalar products that we use in the plug-in ridge estimator. Using the same notations as in  Definition \ref{def :: plug-in ridge estimator} from the main paper, we have $\obsfun{\projope}^\adj \mattovec(\obsfun{\outsamp}) = \mattovec(\obsfun{\nu})$. Then the only difference with the plug-in ridge estimator is that the matrix $(\projope^\adj \projope)_{(\nsamp)}$ is replaced by the matrix $(\obsfun{\projope}^\adj \obsfun{\projope})$ which is block-diagonal with the matrices $\left( \frac 1 {\nobsf_{\sampiter}}\obsfun{\projope}_\sampiter^\adj \obsfun{\projope}_\sampiter \right )_{\sampiter=1}^\nsamp$ as diagonal blocks. In other words, instead of using the true Gram matrix of the dictionary $\projope ^\adj \projope$ for all the observations, we use for the $\sampiter$-th observation an estimated Gram matrix using the locations of observation of the output function $y_\sampiter$.

\section{RELATED WORKS} \label{sec :: related}
We give more details on the methods presented briefly in Section \ref{subsec :: related main} from the main paper. Two of them  \citep{BarathAl17, PoczosAl15} are specific to functional input data. While we propose a straightforward extension of the latter for non-functional input data, such extension is not possible for the former. 

\subsection{Functional kernel ridge regression (FKRR)} \label{subsec :: related fkrr}

\citet{KadriAl10, KadriAl16} solve a functional KRR problem in the framework of function-valued-RKHSs (fv-RKHSs). To that end, they pose the following empirical risk minimization problem:

$$\min_{\regfunc \in \cali H_{\Kxkadri}} \frac 1 \nsamp \sum_{\sampiter=1}^\nsamp \| y_\sampiter - \regfunc(x_\sampiter) \|^2_{\cali Y} + \lambda \|\regfunc\|^2_{\cali H_{\Kxkadri}}, $$

with $\cali H_{\Kxkadri}$ the fv-RKHS associated to some OVK $\Kxkadri: \inspace \times \inspace \longrightarrow \boundedops(\cali Y)$, and $\cali Y$ a Hilbert space. 

Through a representer theorem, the problem can be reformulated using $\nsamp$ variables in $\cali Y$. The optimal representer coefficients can be found by solving the infinite dimensional system:
$$(\Kxmatkadri + \lambda \matrb I)\alpha^{\rmc{fun}} = \outsamp, $$
with $ \alpha^{\rmc{fun}} \in \cali Y^\nsamp$, $(\Kxmatkadri + \lambda \matrb I)^{-1} \in \boundedops(\cali Y)^{\nsamp \times \nsamp}$ and $\outsamp \in \cali Y^\nsamp$. 

We now focus on the case of the separable kernel $\Kxkadri(x, x')=k^{\text{in}}(x, x') \matop L$. $k^{\text{in}}$ is a scalar-valued kernel and $\matop L \in \boundedops(\cali Y)$ is an integral operator characterized by a scalar-valued kernel $k^{\text{out}}$ on $\outfuncdom^2$ and a measure on $\outfuncdom$.

As an example of such kernel, in the experiments we take $k^{\text{in}}$ a scalar Gaussian kernel, $k^{\text{out}}$ a Laplace kernel and use the Lebesgue measure on $\outfuncdom = [0, 1]$ to define the operator $\matop L$: 

\begin{equation}\label{eq :: kernel out fkrr}
\rmc L y: \theta' \longmapsto \int_{\theta \in \Theta} \exp \left (-\frac {|\theta' - \theta |}{\sigma_{k^{\text{out}}}} \right ) \mathrm d \theta.
\end{equation}

For such separable kernel, the Kronecker product structure $( \Kxmatkadri + \lambda \matrb I) = (\Kxinmatkadri \otimes \matop L + \lambda \matrb I)$ can greatly improve the computational complexity; two approaches are possible.
\begin{enumerate}
\item An eigendecomposition can be performed. If such decomposition of $\matop L$ is known in closed-form, the Kronecker product can be exploited to solve the system in $\cali O(\nsamp^3 + \nsamp^2 J \nobsf)$ time, with $J$ the number of eigenfunctions considered and $\nobsf$ the size of the discrete grid used to approximate functions in $\cali Y$. Unfortunately, such closed-forms are rarely known \citep[Section~4.3]{RasmussenWilliams06}. We know that one exists if $k^{\text{out}}(\theta_0, \theta_1) = \exp(-|\theta_0 - \theta_1 |)$, $\outfuncdom=[0, 1]$ and $\mu$ is the Lebesgue measure \citep{Hawkins89}, or if $k^{\text{out}}$ is a Gaussian kernel, $\outfuncdom = \bb R^q$ and $\mu$ is a Gaussian measure \citep{ZhuAl97}. Otherwise, an approximate eigendecomposition can be performed which adds a $\cali O(\nobsf^3)$ term to the above time complexity. 
\item The problem can be discretized on a regular grid \citep{KadriAl10} and solved in $\cali O(\nsamp^3 + \nobsf^3 + \nsamp^2 \nobsf + \nsamp \nobsf^2)$ time using a Sylvester solver or in $\cali O(\nsamp^3 + t^3)$ time using an eigen decomposition (with higher constants).  To compare the above time complexities to that of KPL, we highlight that typically $\nobsf \gg \ndict$ and $t$ is at least of the same order as $\nsamp$. 
\end{enumerate}

We compare both approaches numerically in Section \ref{subsubsec :: comparison solvers fkrr}.

\subsection{Triple basis estimator (3BE)}

\citet{PoczosAl15} firstly represent separately the input and output functions on truncated orthonormal bases obtaining a set of input and output decomposition coefficients: $(\beta^{\rmc{in}}, \beta^{\rmc{out}})$ with $\beta^{\rmc{in}} \in \bb R^{\nsamp \times c}$ and $\beta^{\rmc{out}} \in \bb R^{\nsamp \times \ndict}$; $c \in \bb N^*$ being the cardinality of the input basis and $\ndict \in \bb N^*$ that of the output basis. Then, each set of output coefficient ($\beta^{\rmc{out}}_\dictiter$ for $\dictiter \in [\ndict]$) is regressed on the input coefficients $\beta^{\rmc{in}}$ using KRRs approximated with RFFs \citep{RahimiRecht08}. Denoting by $\matop R(\beta^{\rmc{in}}) \in \bb R^{\nsamp \times J}$ the matrix of RFFs evaluated on the input coefficients $\beta^{\rmc{in}}$, for all $\dictiter \in [\ndict]$, the following (scalar-valued) sub-problem is solved:
$$ \min_{c_{\dictiter} \in \bb R^J} \|\beta^{\rmc{out}}_\dictiter - \matop R(\beta^{\rmc{in}}) c \|_{\bb R^\nsamp}^2 + \lambda \| c_\dictiter \|_{\bb R^J}^2.$$

All those sub-problems require the inversion of the same matrix $(\matop R(\beta^{\rmc{in}})^\trans \matop R(\beta^{\rmc{in}}) + \lambda \matop I)$, which can thus be carried out only once. Putting aside the computations of the decomposition coefficients, solving 3BE then has time complexity $\cali O(J^3 + J^2 \ndict)$.

Nevertheless, 3BE as proposed in \citep{PoczosAl15} is specific to function-to-function regression. As a consequence, when the input data are not functional (as in Section \ref{subsec :: speech} from the main paper), we propose to directly deal with them through a kernel; we call this extension \textbf{one basis estimator (1BE)}. We highlight that 1BE is in fact a particular case of the KPL plug-in ridge estimator with $\dict$ orthonormal and $\Kx = k \matop I$. In that case, the time complexity is $\cali O(\nsamp^3 + \nsamp^2 \ndict)$ (we solve solve $\ndict$ scalar-valued KRRs problems sharing the same kernel matrix and the same regularization parameter).

\subsection{Kernel additive model (KAM)}

In this section only, we consider that the input data consist of functions and that $[0, 1]$ is the domain of both input and output functions. In the function-to-function additive linear model \citep{RamsaySilverman05}, the following empirical risk is minimized: 
\begin{equation} \label{eq :: objective functional additive}
\sum_{\sampiter=1}^\nsamp \int_0^1 \left (y_\sampiter(\theta) - a(\theta) - \int_0^1 b(\zeta, \theta) x_\sampiter(\zeta) ~ \mathrm d \zeta \right)^2 ~ \mathrm d \theta.
\end{equation}

The functions $a: [0, 1] \longrightarrow \bb R$ and $b: [0, 1] \times [0, 1] \longrightarrow \bb R$ are the functions we want to learn. To define an hypothesis class for them, two truncated bases of $\rmc L^2([0, 1])$ are chosen, one for the input space $(e_{\dictiter}^{\rmc{in}})_{\dictiter=1}^{c}$ and one for the output space $(e_{\dictiter}^{\rmc{out}})_{\dictiter=1}^{\ndict}$. With the convention that for $\zeta \in [0, 1]$ and $\theta \in [0, 1]$, $e^{\rmc{in}}(\zeta) = (e^{\rmc{in}}_{\dictiter}(\zeta))_{\dictiter=1}^{c}$ and $e^{\rmc{out}}(\theta) = (e_{\dictiter}^{\rmc{out}}(\theta))_{\dictiter=1}^{\ndict}$, the functions $a$ and $b$ are specified as
\begin{align*}
a(\theta) & = \matop A e^{\rmc{out}}(\theta) \\ b(\zeta, \theta) & = (e^{\rmc{in}}(\zeta))^\trans \matop B e^{\rmc{out}}(\theta).
\end{align*}
Then, we use those expressions for $a$ and $b$ and minimize the objective from Equation \eqref{eq :: objective functional additive} in the variables $\matop A \in \bb R^{1 \times \ndict}$ and $\matop B \in \bb R^{c \times \ndict}$. Importantly, there is not explicit regularization penalty in the problem, however some regularization is achieved implicitly through the choice of the size of the bases $c$ and $\ndict$. 

\citet{BarathAl17} build on this model using RKHSs. The following empirical risk minimization problem is considered 
$$ \min_{h \in \cali H_{\Kxbarath}}\sum_{\sampiter=1}^\nsamp \int_0^1 \left (y_\sampiter(\theta) - \int_0^1 h(\zeta, \theta, x_\sampiter(\zeta)) \mathrm d \zeta \right )^2 \mathrm d \theta + \lambda \|h \|^2_{\cali H_{\Kxbarath}}, $$ 

where $\cali H_{\Kxbarath}$ is the RKHS of a scalar-valued kernel $\Kxbarath: ([0, 1] \times [0, 1] \times \bb R)^2 \longrightarrow \bb R$ and $\lambda > 0$. A representer theorem leads to a closed-from solution. To alleviate the computations, a truncated basis of $J < \nsamp $ of empirical functional principal components of $(y_\sampiter)_{\sampiter=1}^\nsamp$ is used. A matrix of size $\nsamp J \times \nsamp J$ must then be inverted yielding a time complexity of $\cali O(\nsamp^3 J^3)$. However, if $\Kxbarath$ is chosen as a product of three kernels, the separability property can be exploited to solve the problem in $\cali O(\nsamp^3 + J^3 + \nsamp^2 J + \nsamp J^2)$ time using a Sylvester Solver. Note that this possibility to exploit the Kronecker structure of the matrix $A$---page 6 of \citep{BarathAl17}---is not highlighted nor exploited by the authors. However the main bottleneck of the method is the computation of this matrix $A$ in itself; even when exploiting the product of kernels, $\nsamp^2 + J^2$ double integrals must be computed yielding a time complexity of $\cali O(\nsamp^2 t^2 + J^2 \nobsf^2)$ with $t$ the size of the input discretization grid and $\nobsf$ that of the output one. Even for medium $\nsamp$,  $t$ and $\nobsf$ this becomes a challenge, especially as this matrix must be computed many times so as to tune the multiple kernel parameters. 

As an example of a product of kernels used for KAM, in the experiments on the toy dataset and on the DTI dataset, we use a product of three Gaussian kernels: 
\begin{equation} \label{eq :: kernel kam}
\Kxbarath: ((\zeta, \theta, s), (\zeta', \theta', s)) \longmapsto \exp \left ( \frac {-(\zeta - \zeta')^2}{\sigma_1^2} \right ) \exp \left ( \frac {-(\theta - \theta')^2}{\sigma_2^2} \right ) \exp \left ( \frac {-(s - s')^2}{\sigma_3^2} \right ).
\end{equation}

\citet{BarathAl17} present the model for one functional covariate. However, it is straightforward to extend it to the case where there are several ones. Equivalently, consider the input functions are vector-valued with values in $\bb R^o$. Then we can consider a kernel defined on the adapted domain $\Kxbarath: ([0, 1] \times [0, 1] \times \bb R^o)^2 \longrightarrow \bb R$ and no further adaptations are required. 

\subsection{Kernel Estimator (KE)} 

Finally, the functional Nadaraya-Watson kernel estimator has been studied in \citet{Ferraty2011} in the general setting of Banach spaces. Considering a kernel function $K: \bb R \longmapsto \bb R$ combined with a given semi-metric $S$ on $\inspace$, for all $x \in \inspace$, they use the following estimator: 
$$\frac{\sum_{\sampiter=1}^\nsamp K \circ S(x, x_\sampiter) y_\sampiter} {\sum_{\sampiter=1}^\nsamp K \circ S(x, x_\sampiter)}.$$
This method is very fast as fitting it boils down to memorizing the training data, however it can lack precision.

\section{EXPERIMENTAL DETAILS AND SUPPLEMENTS} \label{sec :: experiments :: supp}
In this Section we give more insights into the numerical experiments. We introduce a toy function-to-function data to test several robustness properties of our method while two real worlds datasets have been gathered from different publications about functional regression. This collection of dataset could be used in the future for benchmarking. 

To avoid mentioning it repeatedly, we highlight that when performing cross-validation, we use $5$ folds in all the experiments; and when several values are given for a same parameters, all configurations generated by combining the described parameters/dictionaries are included in the cross-validation.

\subsection{Parametrized logcosh loss}

We consider the following logcosh loss in 1d:
$$ a \in \bb R \longmapsto \frac 1 \gamma \log(\text{cosh}(\gamma a)).$$
It corresponds to the loss $\groundloss_{\rmc{lch}}^{(\gamma)}$ defined in Section \ref{subsec :: preliminary elements} from the main paper. We illustrate the effect of the parameter $\gamma$ in Figure \ref{fig :: logcosh 1d}.

As we cannot plot the integral version of this loss, we consider the loss defined on $\bb R^2$ as follows:
$$ (a_0, a_1) \longmapsto \frac 1 \gamma \left ( \log(\text{cosh}(\gamma a_0)) + \log(\text{cosh}(\gamma a_1)) \right ).$$
We plot this loss for $\gamma=5$ in Figure \ref{fig :: logcosh 2d}. 
\begin{figure}
    \begin{center}
    \begin{minipage}{0.25\textwidth}
        \centering
		\includegraphics[width=0.8\linewidth]{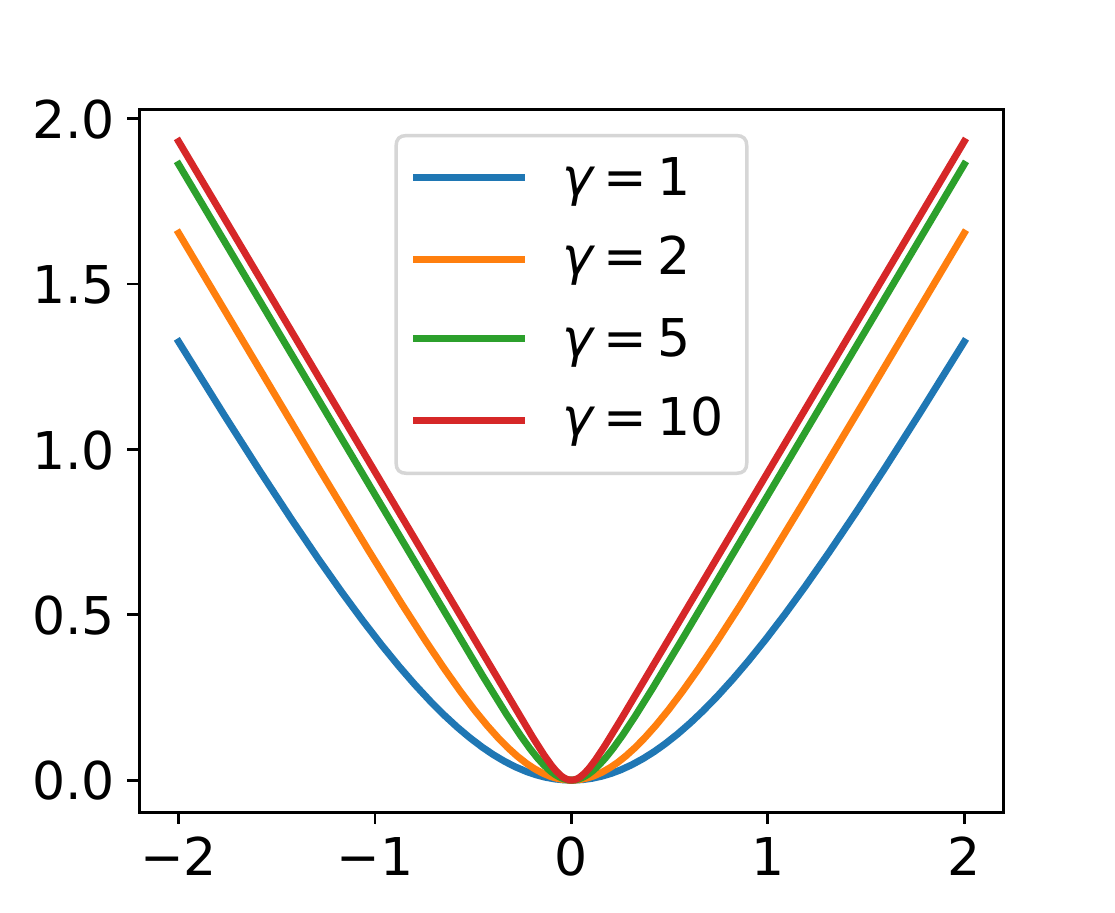}
		\caption{Logcosh loss on $\bb R$.}
		\label{fig :: logcosh 1d}
    \end{minipage}\hfill
    \begin{minipage}{0.35\textwidth}
        \centering
        \includegraphics[width=\linewidth]{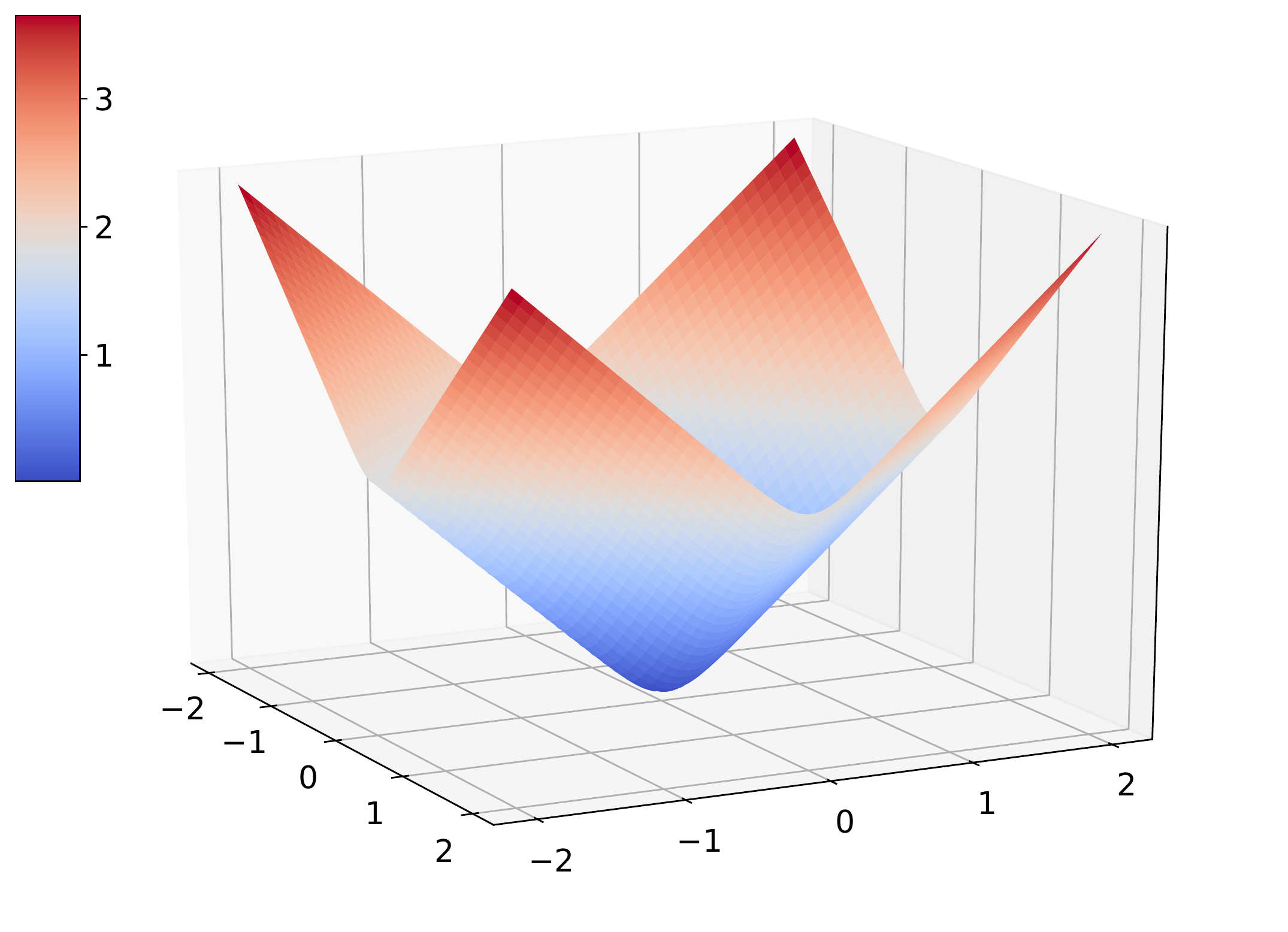}
	\caption{Logcosh loss on $\bb R^2$ ($\gamma=5$).}
	\label{fig :: logcosh 2d}
    \end{minipage}\hfill
    \begin{minipage}{0.15\textwidth}
    \includegraphics[width=\linewidth]{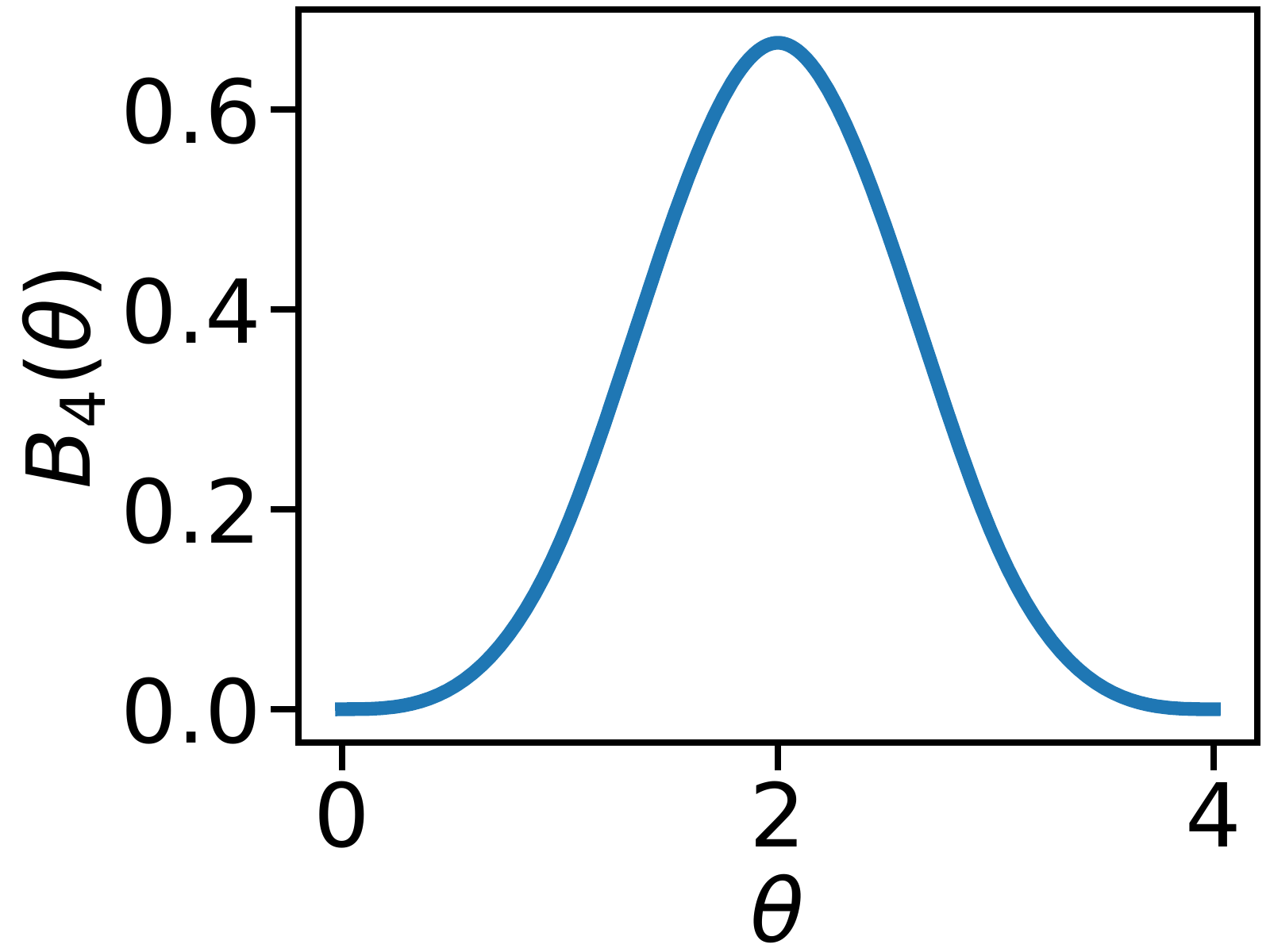}
\caption{Cubic B-spline.}
\label{fig :: cubic spline}
    \end{minipage}\hfill
    \begin{minipage}{0.2\textwidth}
    \includegraphics[width=\linewidth]{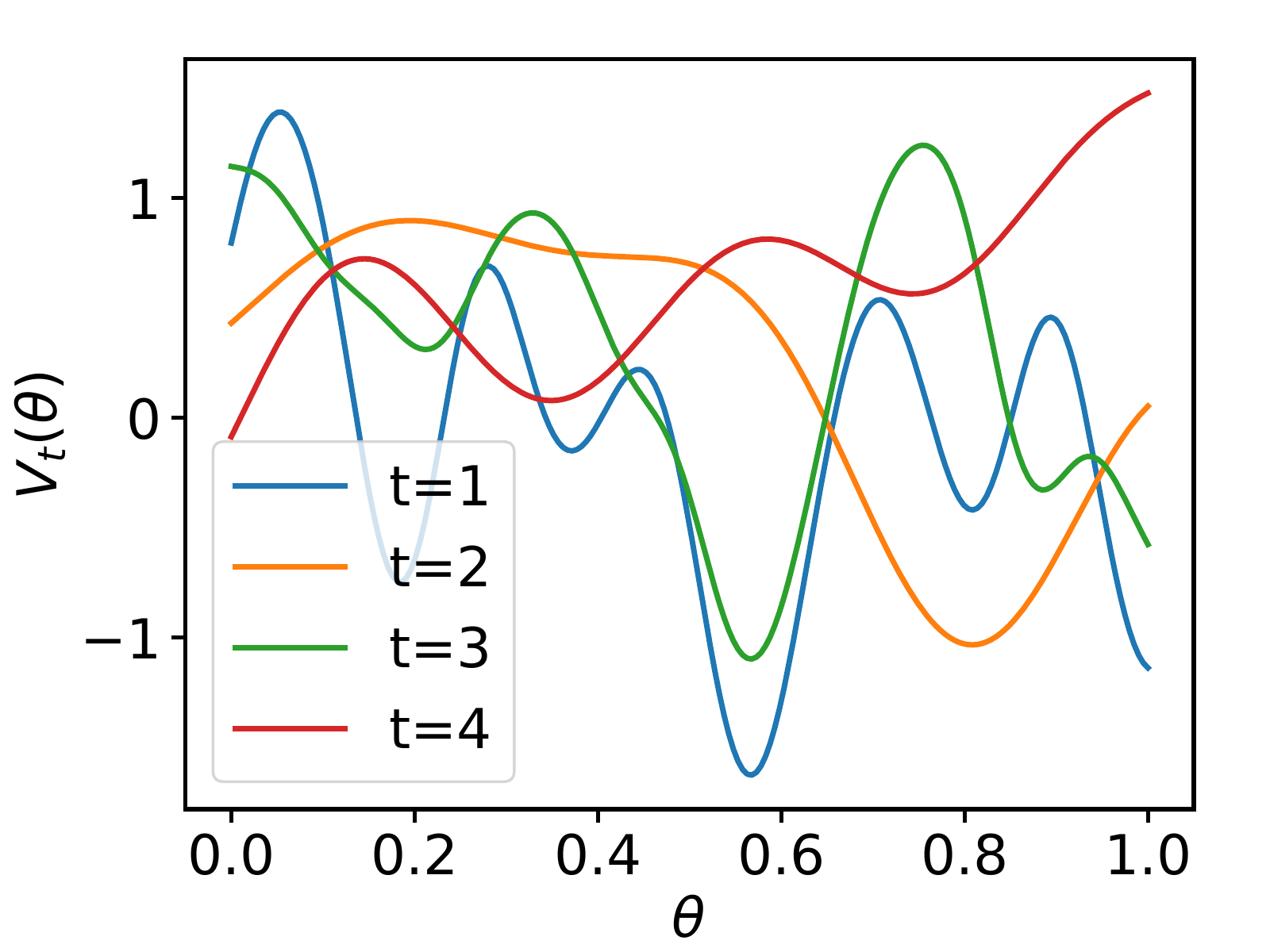}
\caption{GP draws.}
\label{fig :: gp draws}
    \end{minipage}
    \end{center}
\end{figure}

\subsection{Toy dataset}
\begin{figure}
    \begin{center}
    \begin{minipage}{0.5\textwidth}
        \centering
		\includegraphics[width=0.98\linewidth]{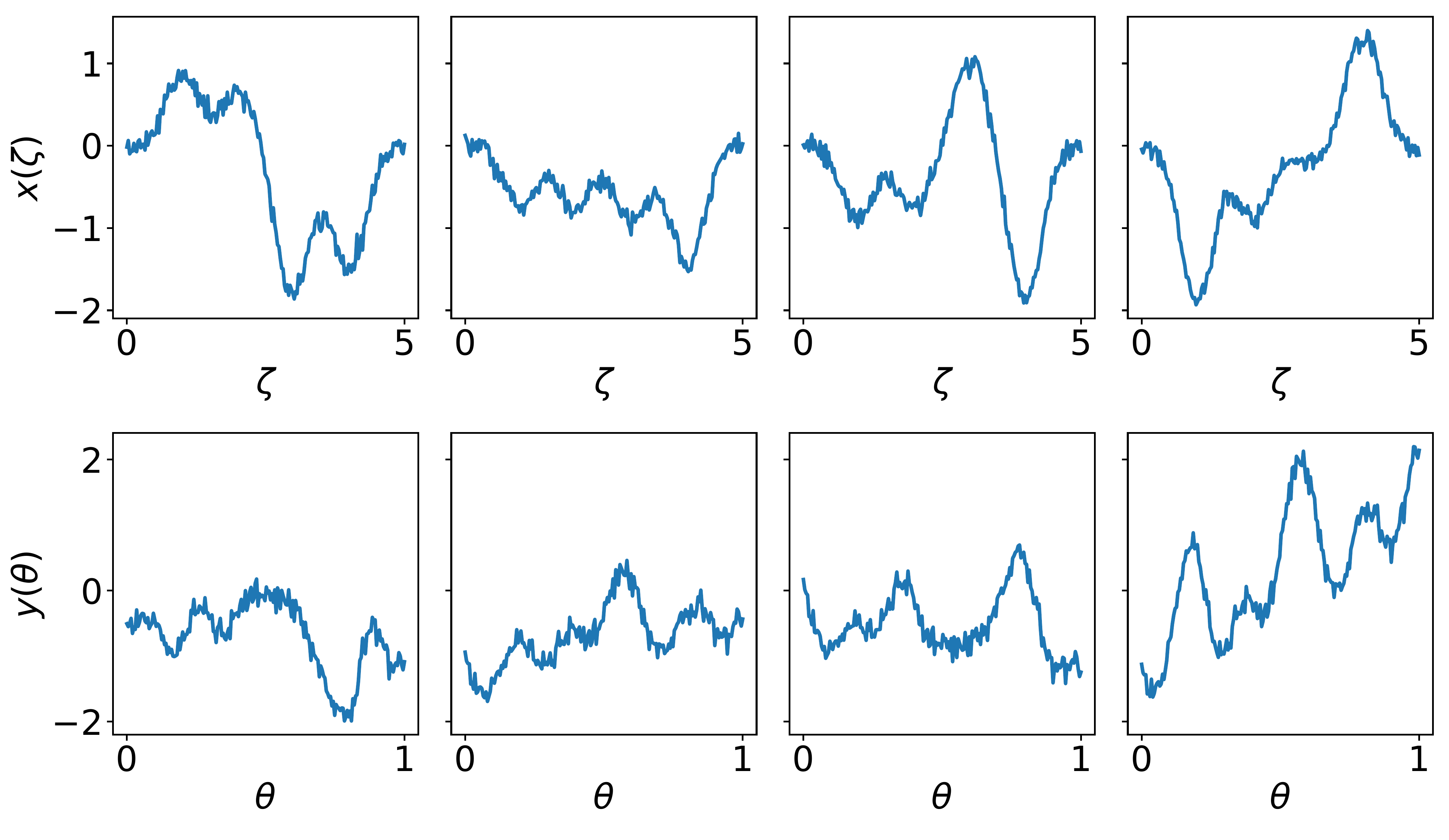}
		\caption{Examples of generated toy data.}
		\label{fig :: example toy data}
    \end{minipage}\hfill
    \begin{minipage}{0.5\textwidth}
        \centering
        \includegraphics[width=0.98\linewidth]{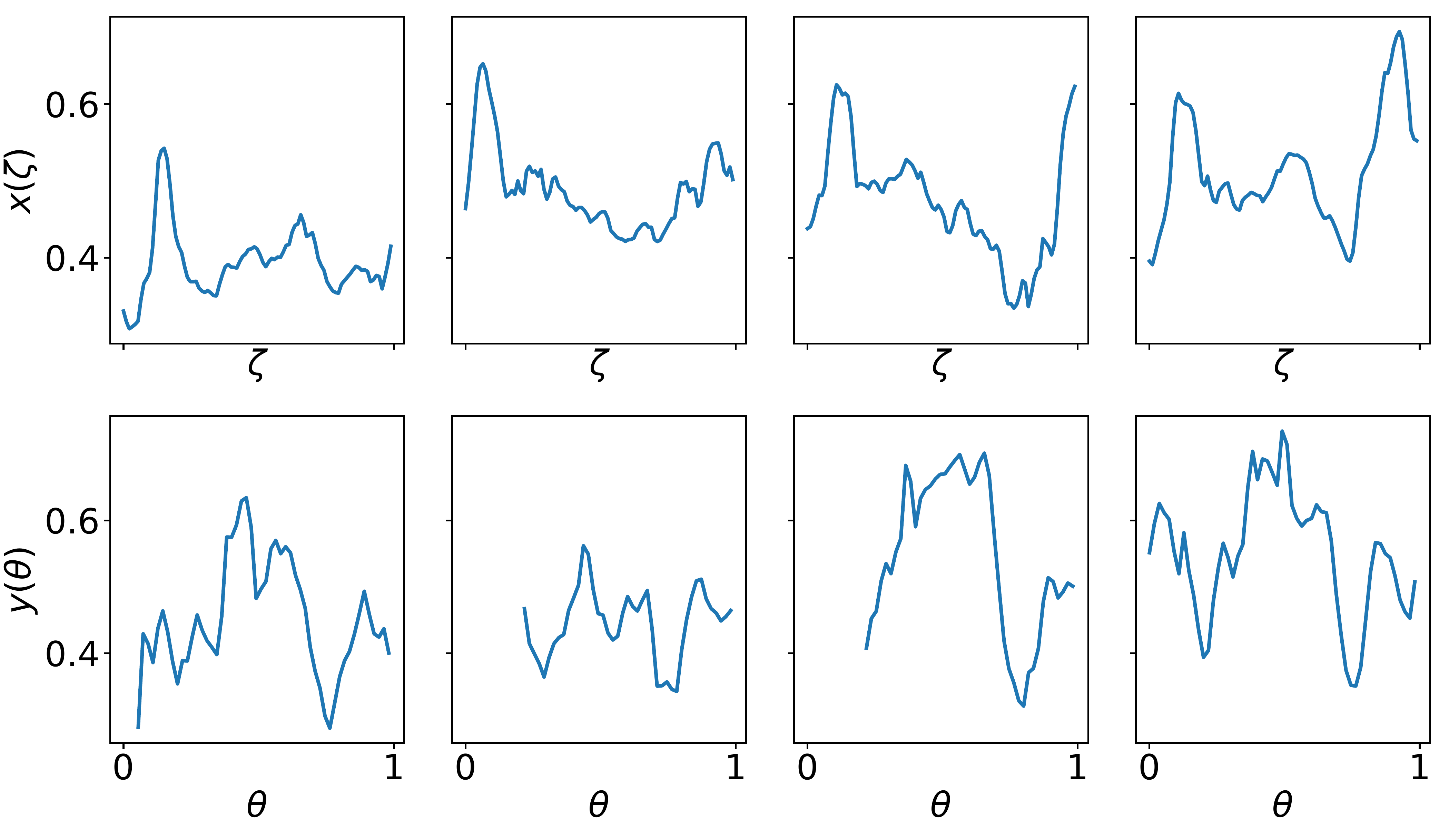}
	\caption{Examples from the DTI dataset.}
	\label{fig :: example dti data}
    \end{minipage}
    \end{center}
\end{figure}

\subsubsection{Generating process}

We consider a functional toy dataset. To generate it, we draw $r \in \bb N$ indepent zero mean Gaussian processes (GP) with Gaussian covariance functions. More precisly, for $t \in [r]$ the Gaussian process $V_t$ has covariance $(\theta_1, \theta_2) \longmapsto \exp \left ( -\frac{(\theta_2 - \theta_1)^2}{b_t^2} \right )$. We then keep those Gaussian processes fixed. In practice in those experiments we take $r=4$ and $b_1=0.1$, $b_2=0.25$, $b_3=0.1$ and $b_4=0.25$. An example of a draw of such GPs is displayed in Figure \ref{fig :: gp draws}. To generate an input/output pair, we draw $r$ coefficients $a \in \bb R^r $ i.i.d according to a uniform distribution $\cali U \left ([-1, 1] \right )$ 
Let $B_4$ denote the cardinal cubic spline \citep{DeBoor01}; it is symmetric around $\zeta=2$ and of width $4$ (see Figure \ref{fig :: cubic spline}). Let then $\bar B_4: \zeta \longmapsto B_4(4\zeta + 2)$ (a centered version of $B_4$ rescaled to have width $1$). We consider the input function $x(\zeta):=\sum_{t=1}^r a_t \bar B_4(\zeta - t) $ with $\zeta \in [0, 5]$. To it we associate the output function $y(\theta) = \sum_{t=1}^r a_t V_t(\theta)$ with $\theta \in [0, 1]$. In practice, we observe $x$ and $y$ on regular grids of size $200$. For the experiments with missing data, we remove sampling points from those grids. Finally we add Gaussian noise on the input observations with standard deviation $\sigma_x=0.07$ in all experiments. Examples of data generated that way with a Gaussian noise with standard deviation $\sigma_y=0.1$ added on the output observations are shown in Figure \ref{fig :: example toy data}. 

\subsubsection{Experimental details}
We compute the means over 10 runs with different train/test split for all experiments. For all the methods, $\lambda$ is taken in a geometric grid of size $20$ ranging from $10^{-9}$ to $10^{-4}$. Moreover, we consider the following specific parameters.
\begin{itemize}
\item \textbf{KPL}. We take a truncated Fourier dictionary including $15$ frequencies and use the separable kernel $\rmc K(x, x'):=k(x, x') \matop I$ with $k$ a scalar-valued Gaussian kernel with standard deviation $\sigma_k = 20$ and $\matop I \in \bb R^{\ndict \times \ndict}$ the identity matrix. When using the logcosh loss, the parameter $\gamma$ is set to $\gamma = 25$ for the in two experiments related to outliers (so as to approach the absolute loss) and to $\gamma = 10$ for the two other experiments.
\item \textbf{3BE}. We use $k$ a Gaussian kernel with standard deviation $\sigma_k = 3$. We use truncated Fourier bases as dictionaries, we include $10$ and $15$ frequencies respectively for the input dictionary and the output one.
\item \textbf{KAM}. We use the kernel defined in Equation \eqref{eq :: kernel kam} taking $\sigma_1=0.2$, $\sigma_2=0.1$ and $\sigma_3=2.5$ and use $J=20$ functional principal components.
\item \textbf{FKRR}. We take a Gaussian kernel as input kernel with standard deviation parameter set as $\sigma_{k^{\text{in}}} = 20$. We use the output kernel defined in Equation \eqref{eq :: kernel out fkrr} setting its parameter to $\sigma_{k^{\text{out}}} = 0.5$.
\end{itemize}

\subsection{DTI dataset}

\subsubsection{Extensive description of the dataset}

The diffusion tensor imaging (DTI) dataset \footnote[1]{This dataset was collected at Johns Hopkins University and the Kennedy-Krieger Institute and is freely available as a part of the \textit{Refund} R package} consists of 382 Fractional anisotropy (FA) profiles inferred from DTI scans along two tracts---corpus callosum (CCA) and right corticospinal (RCS). The scans were performed on 142 subjects; 100 multiple sclerosis (MS) patients and 42 healthy controls. MS is an auto-immune disease which causes the immune system to gradually destroy myelin (the substance which isolates and protects the axons of nerve cells), resulting in brain lesions and severe disability. FA profiles are frequently used as an indicator for demyelification which causes a degradation of the diffusivity of the nerve tissues. The latter process is however not well understood and does not occur uniformly in all regions of the brain. We thus propose here to use our method to try to predict FA profiles along the RCS tract from FA profiles along the CCA tract. So as to remain in an i.i.d. framework, we consider only the first scans of MS patients resulting in $\nsamp=100$ pairs of functions. The functions are observed on regular grids of sizes $93$ and $54$ respectively for the CCA and RCS tracts. However, significant parts of the FA profiles along the RCS tract are missing, we are thus dealing with sparsely sampled functions. Examples of instances from this dataset are shown in Figure \ref{fig :: example dti data}.

\subsubsection{Tuning details for Table \ref{tab :: results DTI} of the main paper}

The reported means and standard deviations are computer over 20 runs with different train/test split. For all methods (except KE) we center the output functions using the training examples and add back the corresponding mean to the predictions; and we consider values of $\lambda$ in a geometric grid of size $25$ ranging from $10^{-6}$ to $10^{-2}$.

\begin{itemize}
\item \textbf{KE}. We use a Gaussian kernel with standard deviation in a regular grid ranging from $0.05$ to $2$ with $200$ points.  

\item \textbf{KPL}. For the dictionary, we consider several families of Daubechies wavelets \citep{DaubechiesHeil92} with $2$ or $3$ vanishing moments and $4$ or $5$ dilatation levels. We use a separable kernel of the form $\Kx(x, x') = k(x, x') \matop D$ with $k$ a Gaussian kernel with fixed standard deviation parameter $\sigma_k = 0.9$. The matrix $\rmc D$ is a diagonal matrix of weights decreasing geometrically with the scale of the wavelet at the rate $\frac 1 b$ (meaning for instance that at the $j$-th scale, the corresponding coefficients in the matrix are set to $\frac 1 {b^j}$). $b$ is chosen in a grid ranging from $1$ to $2$ with granularity $0.1$. When using the logcosh loss, we consider values of the parameter $\gamma$ in $\{0.25, 0.5, 0.75, 1, 1.5, 2, 3, 4, 5, 10 \}$.

\item \textbf{3BE}. We test the same dictionaries of wavelets as for KPL for both the input and the output functions. We use $200$ RFFs for the approximated KRRs; and  consider standard deviation for the corresponding approximated Gaussian kernel in the grid $\{ 7.5, 10, 12.5, 15, 17.5, 20 \}$.

\item \textbf{KAM}. We use the product of Gaussian kernels defined in Equation \eqref{eq :: kernel kam} fixing $\sigma_1 = \sigma_2 = \sigma_3 = 0.1$. We consider including $J=20$ and $J=30$ principal components for the approximation.

\item \textbf{FKRR}. We take a Gaussian kernel as input kernel with standard deviation parameter set as $\sigma_{k^{\text{in}}} = 0.9$. We use the output kernel defined in Equation \eqref{eq :: kernel out fkrr} choosing its parameter in $\sigma_{k^{\text{out}}} \in \{ 0.5, 0.75, 1, 1.25, 1.5, 1.75, 2, 3, 4, 5, 7.5, 10 \}$.
\end{itemize}

\subsection{Speech dataset}

\subsubsection{More on the experimental setting} \label{subsubsec :: kernel details speech}

To match words of varying lengths, we extend symmetrically both the input sounds and the VT functions so as to match the longest word. We represent the sounds using 13 mel-frequency cepstral coefficients (MFCC) acquired each 5ms with a window duration of 10ms. We split the data as $\nsamp_{\text{train}} = 300$ and $\nsamp_{\text{test}}=113$. Finally, we normalize the domain of the output functions to $[0, 1]$, and normalize as well their range of values to $[-1, 1]$ so that the scores are of the same magnitude for the different vocal tracts.

The input data consist of matrices in $\bb R^{\nobsf \times 13}$ (here the number of discretization points is the same for the input and for the output functions, so we have $t=\nobsf$ discretization points for the MFCCs). Those correspond to discrete observations from $\bb R^{13}$-valued functions. For ridge-DL-KPL, 1BE/ridge-Four-KPL and FKRR, we wish to use the following integral kernel based on a Gaussian kernel:
$$ (x_0, x_1) \longmapsto \int_{[0, 1]} \exp\left ( \frac {- \|x_1(\zeta) - x_0(\zeta) \|^{2}_2} {\sigma^2} \right ) d \zeta. $$

In practice, we approximate it using the discretized datapoints as:

\begin{equation} \label{eq :: kernel speech}
(\obsfun{x}_0, \obsfun{x}_1) \longmapsto \frac 1 \nobsf \sum_{\obsfiter=1}^{\nobsf} \exp \left (\frac {- \|\obsfun{x}_{1 \obsfiter} - \obsfun{x}_{0 \obsfiter} \|^{2}_2} {\sigma^2} \right ).
\end{equation}

For KAM, we use the kernel defined on $([0, 1] \times [0, 1] \times \bb R^{13})^2$ by:

\begin{equation} \label{eq :: kernel kam speech}
((\zeta, \theta, w), (\zeta', \theta', w')) \longmapsto \exp \left ( \frac {-(|\zeta - \zeta'|}{\sigma_1} \right ) \exp \left ( \frac {-|\theta - \theta'|}{\sigma_2} \right ) \exp \left ( \frac {-\|w - w' \|_2^2}{\sigma_3^2} \right ).
\end{equation}

In practice there are magnitude differences between the MFCCs. So as to avoid biasing the norms to be over-representative of the larger ones, before applying the above describe kernels, we standardize the MFCCs using the training data. For the $r$-th MFCC, we set $\text{avg}^{(r)}: = \frac 1 {\nsamp_{\text{train}} \nobsf} \sum_{\sampiter = 1}^{\nsamp_{\text{train}}} \sum_{\obsfiter=1}^\nobsf \obsfun{x}_{\sampiter \obsfiter}^{(r)}$ and $\text{std}^{(r)}:= \sqrt{\frac 1 {\nsamp_{\text{train}} \nobsf - 1}\sum_{\sampiter=1}^{\nsamp_{\text{train}}} \sum_{\obsfiter=1}^\nobsf (\obsfun{x}_{\sampiter \obsfiter}^{(r)} - \text{avg}^{(r)})^2 }$, and use as input data $\left ( \left ( \frac {x^{(r)}_\sampiter} {\text{std}^{(r)}} \right )_{r=1}^{13} \right)_{\sampiter=1}^{\nsamp_{\text{train}}}$.

\subsubsection{Details for the MSEs part of Figure \ref{fig :: speech} from the main paper}

The reported means and standard deviations are computed over 10 runs with different train/test split. For all methods, we consider values of $\lambda$ in a geometric grid ranging from $10^{-12}$ to $10^{-5}$ of size $30$ and try both centering and not centering the output functions. For ridge-DL-KPL, 1BE/ridge-Four-KPL and FKRR, we use the kernel from Equation \eqref{eq :: kernel speech} as input kernel taking $\sigma \in \{3, 4, 5, 7.5, 10 \}$.

\begin{itemize}
\item \textbf{ridge-DL-KPL}. The dictionary $\dict$ is learnt by solving Problem (\ref{prob :: dictionary learning}) from the main paper with $\cali C$ and $\Omega_{\bb R^\ndict}$ as introduced in Section \ref{subsec :: dictionary learning} from the main paper. The number of atoms is fixed at $30$. 

\item \textbf{1BE/ridge-Four-KPL}. We use a truncated Fourier basis as dictionary with included number of frequencies in the grid $\{ 20, 30, 40, 50 \}$. 

\item \textbf{FKRR}. We use the kernel from Equation \eqref{eq :: kernel out fkrr} as output kernel. We consider the following values for its parameter: $\sigma_{k^{\text{out}}} \in  \{ 0.005, 0.01, 0.05, 0.1, 0.125, 0.15 \}$.

\item \textbf{KAM}. We use the kernel defined above in Equation \eqref{eq :: kernel kam speech} for which we consider the following parameters values $\sigma_1 \in \{0.01, 0.05, 0.1, 0.5 \}$, $\sigma_2 \in \{ 0.0005, 0.001, 0.005, 0.01 \}$ and $\sigma_3 \in \{ 0.05, 0.1, 0.5, 1, 5\}$. We consider also $J \in \{30, 40, 50 \}$ functional PCAs.  
\end{itemize}

\subsubsection{Details for the fitting times part of Figure \ref{fig :: speech} from the main paper}

\noindent{\bf Infrastructure and measurements details.}
So as to get better control over execution, we perform those experiments on a laptop rather than on the computing cluster used for the other experiments. This laptop is equipped with a 8th Generation Intel Core i7-8665U processor and 16 Gb of RAM. In Python, using the \textit{multiprocessing} package, we execute the  tasks in parallel, each on exactly one core of the CPU. We measure the corresponding CPU time using the \textit{process\_time()} function from the \textit{time} package.

\noindent{\bf Parameters.}
Computation times necessarily depend on the choice of parameters. This dependence can be explicit for parameters determining the complexity of the problems (for instance the size of a dictionary or the size of an approximation grid). For such parameters, we use fixed values for each method which correspond either to the fixed values used or to those elected by cross-validation in the MSEs experiments; we detail those values below. Other parameters can influence the computational times through the conditioning of the problem. To account for this, we consider several values which we give below as well. The means and standard deviations of the obtained fitting times are reported in the right panel of Figure \ref{fig :: speech} from the main paper.

The computation times are averaged over 10 runs of the experiments with different shuffling of the dataset and over the VTs. For all methods, we consider values of $\lambda$ in a geometric grid ranging from $10^{-12}$ to $10^{-5}$ of size $30$ and center the output functions. For ridge-DL-KPL, 1BE/ridge-Four-KPL and FKRR, we use the kernel from Equation \eqref{eq :: kernel speech} as input kernel taking $\sigma=3$.
\begin{itemize}

\item \textbf{ridge-DL-KPL}. The dictionary $\dict$ is learnt by solving Problem (\ref{prob :: dictionary learning}) from the main paper with $\cali C$ and $\Omega_{\bb R^\ndict}$ as introduced in Section \ref{subsec :: dictionary learning} from the main paper. The number of atoms is fixed at $30$. 

\item \textbf{1BE/ridge-Four-KPL}. We use a truncated Fourier basis as dictionary with $50$ included frequencies, thus the size of the dictionary is $\ndict = 99$ (cosinuses and sinuses are included plus a constant function).

\item \textbf{FKRR}. We use the kernel from Equation \eqref{eq :: kernel out fkrr} as output kernel. We consider the following values for its parameter: $\sigma_{k^{\text{out}}} \in  \{ 0.05, 0.1 \}$.

\item \textbf{KAM}. We use the kernel defined above in Equation \eqref{eq :: kernel kam speech} for which we use the following parameters values: $\sigma_1=0.1$, $\sigma_2=0.05$ and $\sigma_3=1$. We take $J=40$ functional PCAs.

\end{itemize}

\subsubsection{Comparison of solvers for FKRR} \label{subsubsec :: comparison solvers fkrr}
As highlighted in Section \ref{sec :: related}, there are two possible ways of solving FKRR with a separable kernel. 
We compare the two approaches on the speech dataset in Figure \ref{fig :: comparison solvers}. FKRR Eigapprox corresponds to the eigendecomposition solver and FKRR Syl to the Sylvester solver. Let $J$ be the number of eigenfunctions considered for the output operator $\matop L$. The difference in computational cost is mostly imputable to the need in FKRR Eigapprox to instantiate and compute $\nsamp J$ functions which correspond to Kronecker products between eigenvectors of the kernel matrix and eigenfunctions of the output operator. However, since those vectors, are functions, so as to be manipulated, they need to be discretized. Considering a discretization grid of size $\nobsf$, those vectors are of size $\nsamp \times \nobsf$ (see Algorithm 1 in \citet{KadriAl16} for more details) which can be heavy (there are $\nsamp J$ of them).

To obtain Figure \ref{fig :: comparison solvers}, we consider the following parameters for the two solvers.
\begin{itemize}
\item \textbf{FKRR Eigapprox.} We use $J=20$ eigenfunctions to approximate the output operator, a grid of size $t=300$ to approximate functions. We take the output kernel parameters in $\sigma_{k^{\text{out}}} \in \{ 0.02, 0.05, 0.1, 0.15 \} $ and  $\lambda$ in a geometric grid of size $30$ ranging from $10^{-12}$ to $10^{-5}$. 
\item \textbf{FKRR Syl.} The plots correspond to the experiments already performed and described previously.
\end{itemize}

\begin{figure}
\begin{center}
\includegraphics[width=0.4\textwidth]{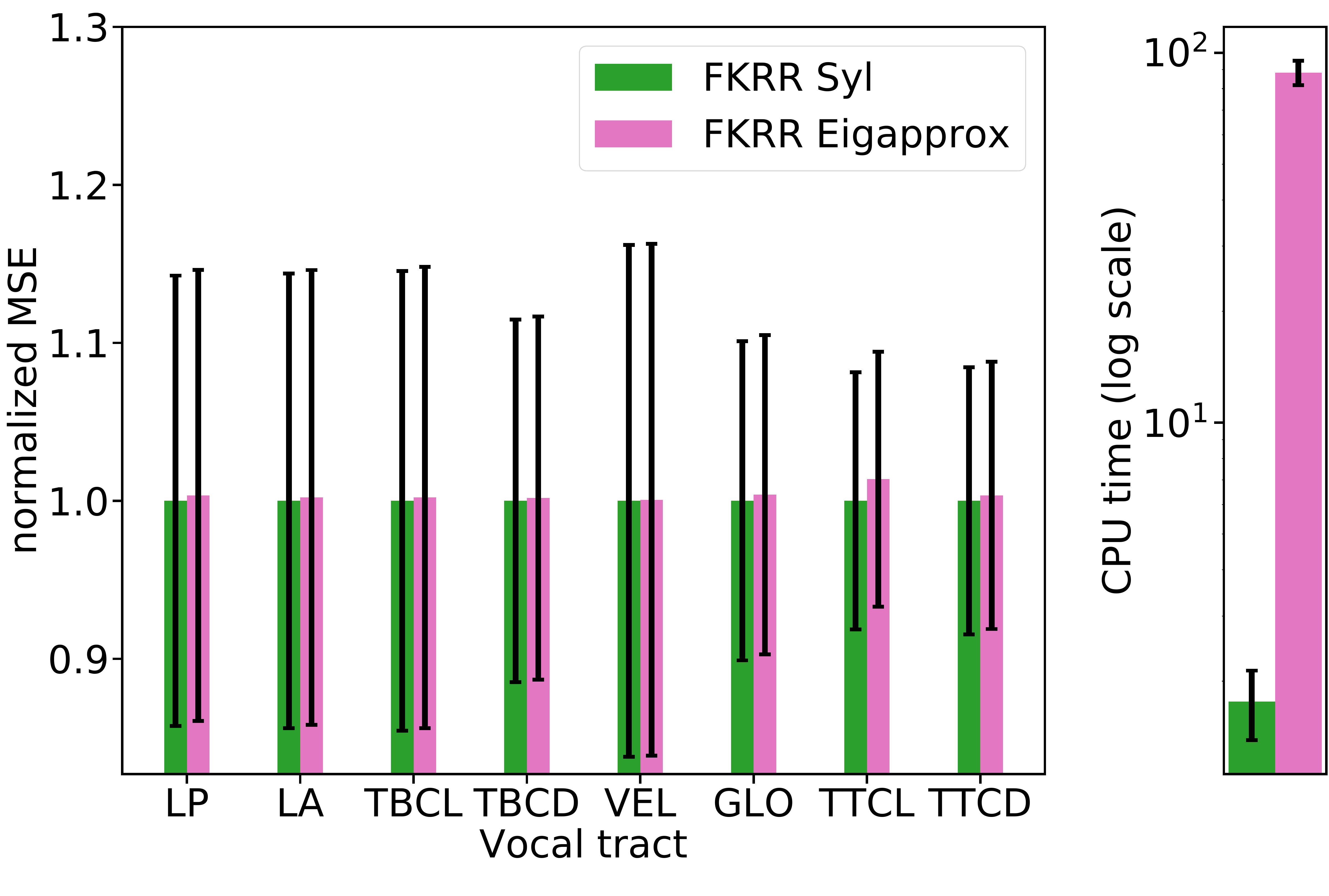}
\caption{Comparison of two solvers for FKRR on speech dataset}
\label{fig :: comparison solvers}
\end{center}
\end{figure}

\end{document}